\documentclass[twoside,11pt]{article}

\usepackage{jmlr2e}

\usepackage{lastpage}
\jmlrheading{23}{2022}{1-\pageref{LastPage}}{12/21; Revised 10/22}{12/22}{21-1524}{ Nikola Kovachki, Zongyi Li, Burigede Liu, Kamyar Azizzadenesheli, Kaushik Bhattacharya, Andrew Stuart, Anima Anandkumar}
\ShortHeadings{Neural Operator: Learning Maps Between Function Spaces With Applications to PDEs}{ Kovachki,  Li,  Liu,  Azizzadenesheli,  Bhattacharya,  Stuart,  Anandkumar}

\firstpageno{1}

\usepackage{breakcites}
\usepackage{natbib}
\usepackage{enumitem,kantlipsum}

\usepackage{booktabs}       
\usepackage{amsmath} 
\usepackage{amsfonts,pifont}       
\usepackage{mathtools}
\usepackage{nicefrac}       
\usepackage{xcolor}
\usepackage{soul}
\usepackage{dsfont}
\usepackage{hyperref}
\usepackage{tikz-cd}
\usepackage{subcaption}

\usepackage{graphicx}
\usepackage{times}
\usepackage{dsfont}

\usepackage{diagbox}

\newcommand{\one}{\mathds{1}}
\newcommand{\cmark}{\ding{51}}%
\newcommand{\xmark}{\ding{55}}%
\definecolor{darkred}{rgb}{.7,0,0}
\definecolor{darkgreen}{rgb}{0,.7,0}

\newcommand{\kamyar}[1]{\textcolor{red}{Kamyar:~#1}}
\newcommand{\nk}[1]{{\color{black}{#1}}}

\newcommand{\as}[1]{{\color{black}{#1}}}

\newcommand{\done}[1]{{\iffalse \color{darkgreen}{#1} \fi}}
\newcommand{\postponed}[1]{{\iffalse \color{darkgreen}{#1} \fi}}
\newcommand{\hold}[1]{#1}

\newcommand{\A}{\mathcal{A}}
\newcommand{\N}{\mathbb{N}}
\newcommand{\Z}{\mathbb{Z}}
\newcommand{\bbZ}{\mathbb{Z}}
\newcommand{\R}{\mathbb{R}}
\newcommand{\C}{\mathbb{C}}
\newcommand{\E}{\mathbb{E}}
\newcommand{\U}{\mathcal{U}}

\newcommand{\LL}{\mathsf{L}}

\newcommand{\cK}{\mathcal{K}}
\newcommand{\cP}{\mathcal{P}}
\newcommand{\cQ}{\mathcal{Q}}
\newcommand{\cW}{\mathcal{W}}

\newcommand{\Ftrue}{\mathcal{G}^\dagger}
\newcommand{\G}{\mathcal{G}}
\newcommand{\F}{\mathcal{G}}

\newcommand{\cF}{\mathcal{G}}
\newcommand{\cG}{\mathcal{F}}

\newcommand{\X}{\mathcal{X}}
\newcommand{\Y}{\mathcal{Y}}

\newcommand{\NN}{\mathsf{N}}

\newcommand{\dx}{\:\mathsf{d}x}
\newcommand{\dy}{\:\mathsf{d}y}

\newtheorem{assumption}[theorem]{Assumption}

\newcommand{\vdown}{\check{v}}
\newcommand{\vup}{\hat{v}}
\newcommand{\eps}{\epsilon}

\begin{document}

\title{
Neural Operator: Learning Maps Between Function Spaces \\ With Applications to PDEs
}

\renewcommand{\thefootnote}{\fnsymbol {footnote}}

\author{\name Nikola Kovachki\thanks{Equal contribution.} \,\thanks{Majority of the work was completed while the author was at Caltech.} \email nkovachki@nvidia.com \addr Nvidia
\AND \name Zongyi Li$^*$ \email zongyili@caltech.edu \addr Caltech
\AND \name Burigede Liu \email bl377@cam.ac.uk \addr Cambridge University
\AND \name Kamyar Azizzadenesheli \email kamyara@nvidia.com \addr  Nvidia
\AND \name Kaushik Bhattacharya \email bhatta@caltech.edu \addr Caltech
\AND \name Andrew Stuart \email astuart@caltech.edu \addr Caltech
\AND \name Anima Anandkumar \email anima@caltech.edu \addr Caltech
}


%

\editor{Lorenzo Rosasco}

\maketitle

\begin{abstract}

The classical development of neural networks has primarily focused on learning mappings between finite dimensional Euclidean spaces or finite sets.
We propose a generalization of neural networks to learn operators, termed \emph{neural operators}, that map between infinite dimensional function spaces.
We formulate the neural operator as a composition of  linear integral operators and nonlinear activation functions. We prove a universal approximation theorem for our proposed neural operator, showing that it can approximate any given nonlinear continuous operator. The proposed neural operators are also discretization-invariant, i.e., they share the same model parameters among different discretization of the underlying function spaces. 
Furthermore, we introduce four classes of efficient parameterization, viz., graph neural operators,  multi-pole graph neural operators, low-rank neural operators, and Fourier neural operators.
An important application for neural operators is learning surrogate maps for the solution operators of
partial differential equations (PDEs). We consider standard PDEs such as the Burgers, Darcy subsurface flow, and the Navier-Stokes equations, and show that the proposed neural operators have superior performance compared to existing machine learning based methodologies,
while being several orders of magnitude faster than conventional PDE solvers. 


 
\end{abstract}


%
%

\section{Introduction}
\label{sec:I}


Learning mappings between  function spaces has widespread applications  in science and engineering. For instance, for  solving differential equations,  the input is a coefficient function and the output is a solution function. A straightforward solution to this problem  is to simply  discretize the infinite-dimensional input and output function spaces into finite-dimensional grids, and apply standard learning models such as 
neural networks. However, this limits  applicability since the learned  neural network model may not  generalize well to different discretizations, beyond the discretization grid of the training data.

To overcome these limitations of standard neural networks, we  formulate a new deep-learning framework for learning operators, called {\em neural operators}, which directly map between  function spaces on bounded domains.  Since our neural operator is  designed on function spaces, they
can be discretized by a variety of different methods, and at different levels of resolution, without the need for re-training. In contrast, standard neural network architectures depend heavily on the discretization of training data: new architectures with new parameters may be needed to achieve the same error for   data with varying discretization.  We also propose the notion of {\em discretization-invariant} models and prove that our neural operators satisfy this property, while standard neural networks do not.

\subsection{Our Approach}
\label{ssec:OC}

\paragraph{Discretization-Invariant Models.}

We formulate a precise mathematical notion of discretization invariance. We require any discretization-invariant model with a fixed number of parameters to satisfy the following:
\begin{enumerate}[leftmargin=*]
    \item  acts on any discretization of the input function, i.e. accepts any set of points in the input domain,
    \item can be evaluated at any point of the output domain,
    \item  converges to a continuum operator as the discretization is refined. 
\end{enumerate}

\noindent The first two requirements of accepting any input and output points in the domain is a natural requirement for discretization invariance, while the last one ensures consistency in the limit as the discretization is refined. For example, families of graph neural networks~\citep{scarselli2008graph} and transformer models~\citep{vaswani2017attention} are resolution invariant, i.e., they can receive inputs at any resolution, but they fail to converge to a continuum operator as discretization is refined. Moreover, we require the models to have a fixed number of parameters; otherwise, the number of parameters becomes unbounded in the limit as the discretization is refined, as shown in Figure \ref{fig:discretization-invariance}.  Thus the notion of discretization invariance allows us to define neural operator models that are consistent in function spaces and can be applied to data given at any resolution and on any mesh. We also establish that standard neural network models are not discretization invariant.



\begin{figure}[h]
    \centering
    \includegraphics[width=0.8\textwidth]{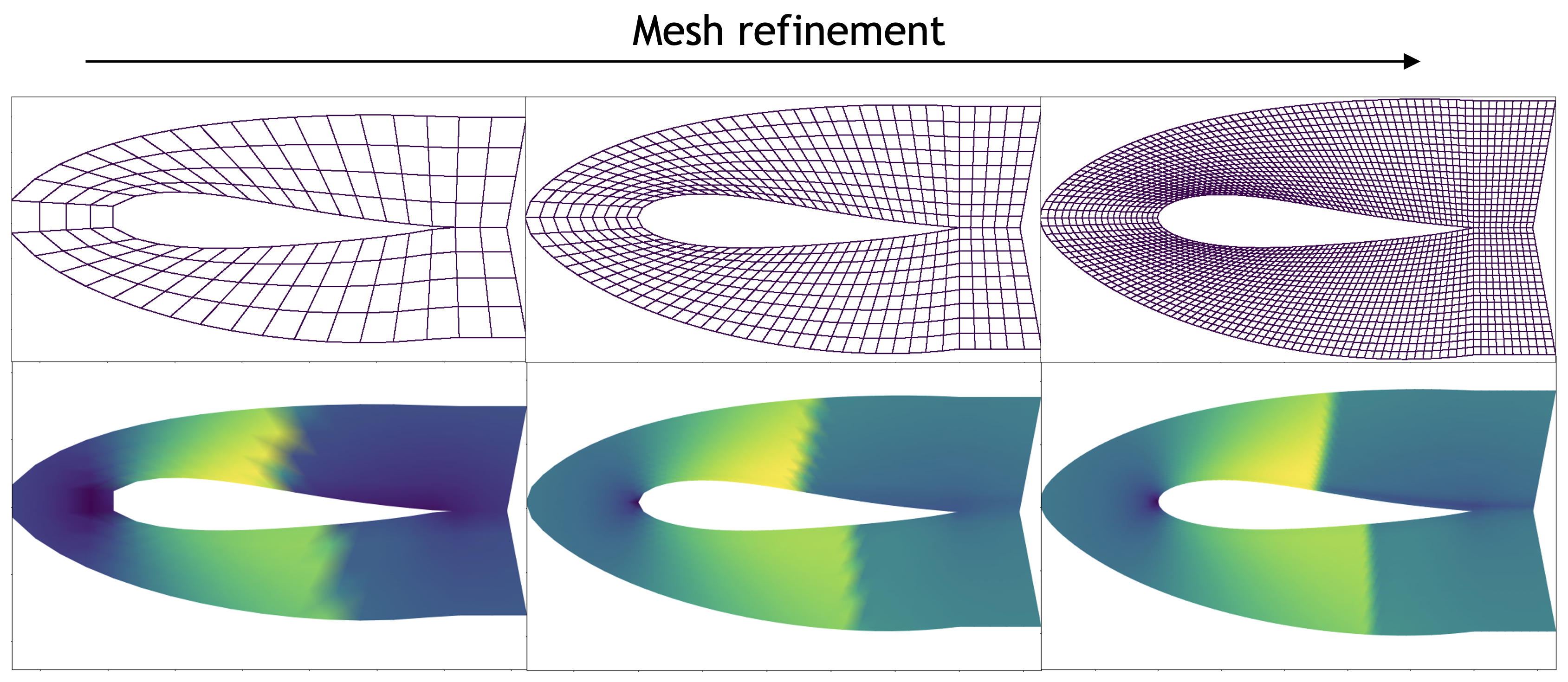}
    \caption{Discretization Invariance}
    \label{fig:discretization-invariance}
    An discretization-invariant operator has convergent predictions on a mesh refinement.
\end{figure}

\paragraph{Neural Operators.}

We introduce the concept of neural operators   for learning operators that are  mappings between infinite-dimensional function spaces. We propose neural operator architectures to be multi-layers where layers  are themselves operators composed with non-linear activations. This ensures that that the overall end-to-end composition  is an operator, and thus satisfies the discretization invariance property. The key design choice for neural operator is the operator layers. To keep it simple, we limit ourselves to layers that are linear operators. Since these layers are composed with non-linear activations, we obtain neural operator models that are expressive and able to capture any continuous operator. The latter property is known as universal approximation. 

The above line of reasoning for neural operator design follows closely the design of standard neural networks, where linear layers (e.g. matrix multiplication, convolution) are composed with non-linear activations, and we have universal approximation of continuous functions defined on compact domains~\citep{hornik1989multilayer}. Neural operators replace finite-dimensional linear layers in neural networks with linear operators in function spaces.

We  formally establish that neural operator models with a fixed number of parameters satisfy discretization invariance. We further show that neural operators models are universal approximators of continuous operators acting between Banach spaces, and  can uniformly approximate any continuous operator defined on a compact set of a Banach space.   {\bf Neural operators are the only known class of models that guarantee both discretization-invariance and universal approximation.} See Table~\ref{table:deeplearning_comparison} for a comparison among the deep learning models. Previous deep learning models are mostly defined on a fixed grid, and  removing, adding, or moving grid points generally makes these models no longer applicable. Thus, they are  not discretization invariant.

We propose several design choices for the linear operator layers in neural operator such as a parameterized integral operator or through multiplication in the spectral domain as shown in Figure~\ref{fig:NO_architecture}.  
Specifically, we propose four practical methods for implementing the neural operator framework: graph-based operators, low-rank operators, multipole graph-based operators, and Fourier  operators. Specifically, for graph-based operators, we develop a Nystr\"om extension to connect the integral operator formulation of the neural operator to families of graph neural networks (GNNs) on arbitrary grids. For Fourier operators,  we consider the spectral domain formulation of the neural operator which leads to efficient algorithms in settings where fast transform methods are applicable. 

We include an exhaustive numerical study of the four formulations of neural operators. Numerically, we show that the proposed methodology consistently outperforms all existing deep learning methods even on the resolutions for which the standard neural networks were designed. For the two-dimensional Navier-Stokes equation, when learning the entire flow map,  the method achieves $<1\%$ error for a Reynolds number of 20 and $8\%$ error for a Reynolds number of 200.

The proposed Fourier neural operator (FNO) has an inference time that is three orders of magnitude faster than the pseudo-spectral method used to generate the data for the Navier-Stokes problem \citep{chandler2013invariant} -- $0.005$s compared to the $2.2s$ on a $256 \times 256$ uniform spatial grid. 
Despite its tremendous speed advantage, the method does not suffer from accuracy degradation when used in downstream applications such as solving Bayesian inverse problems. Furthermore, we demonstrate
that FNO is robust to noise on the testing problems we consider here.

\begin{figure}[t]
    \centering
    \includegraphics[width=0.8\textwidth]{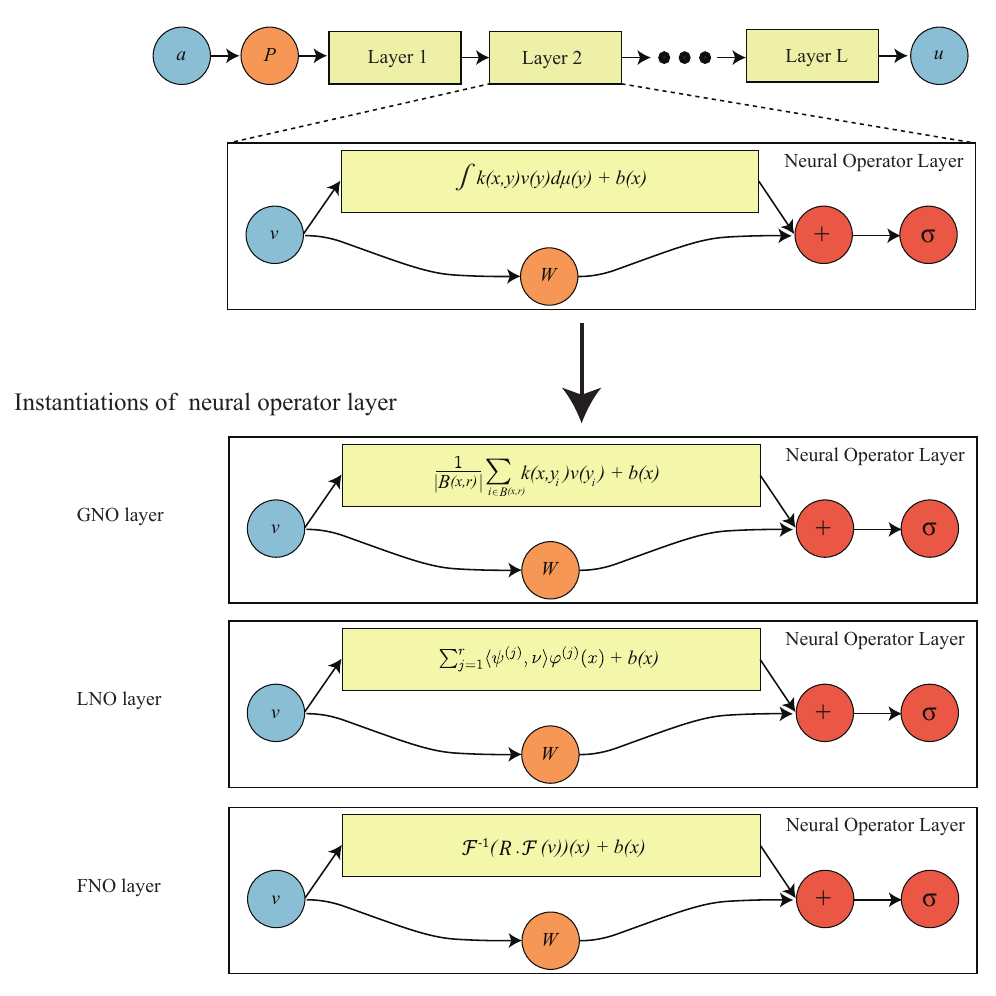}
    \caption{Neural operator architecture schematic}
    \label{fig:NO_architecture} 
    \small{The input function $a$ is passed to a pointwise lifting operator $P$ that is followed by $T$ layers of integral operators and pointwise non-linearity operations $\sigma$. In the end, the pointwise projection operator $Q$ outputs the function $u$. Three instantiation of neural operator layers, GNO, LNO, and FNO are provided.}
\end{figure}



\begin{table}[h!]
\centering
\begin{tabular}{l|c|c|c|c}
\diagbox{Property\hspace{.5cm}}{\hspace{2em}\raisebox{-.3cm}{Model}}&
NNs& DeepONets & Interpolation & Neural Operators\\
\hline
Discretization Invariance & \xmark & \xmark & \cmark & \cmark \\
\hline
Is the output a function? & \xmark & \cmark & \cmark & \cmark \\
\hline
Can query the output at any point? & \xmark & \cmark & \cmark & \cmark \\
\hline
Can take the input at any point? & \nk{\xmark} & \xmark & \cmark & \cmark \\
\hline
Universal Approximation & \xmark & \cmark & \xmark & \cmark \\
\hline
\end{tabular}
\caption{Comparison of deep learning models. The first row indicates whether the model is discretization invariant. The second and third rows indicate whether the output and input are a functions. The fourth row indicates whether the model class is a universal approximator of operators. Neural Operators are discretization invariant deep learning methods that output functions and can approximate any operator.}
\label{table:deeplearning_comparison}
\end{table}

\subsection{Background and Context}
\label{ssec:LR}

\paragraph{Data-driven approaches for solving PDEs.}
Over the past decades, significant progress has been made in formulating \citep{gurtin1982introduction} and solving \citep{johnson2012numerical} the governing PDEs in many scientific fields from micro-scale problems (e.g., quantum and molecular dynamics) to macro-scale applications (e.g., civil and marine engineering). Despite the success in the application of PDEs to solve real-world problems,  two significant challenges remain: (1) identifying the governing model for complex systems; (2) efficiently solving large-scale nonlinear systems of equations.

Identifying and formulating the underlying PDEs appropriate for modeling a specific problem usually requires extensive prior knowledge in the corresponding field which is then combined with universal conservation laws to design a predictive model. For example, modeling the deformation and failure of solid structures requires detailed knowledge of the relationship between stress and strain in the constituent material. For complicated systems such as living cells, acquiring such knowledge is often elusive and formulating the governing PDE for these systems remains prohibitive, or the models
proposed are too simplistic to be informative. The possibility of acquiring such knowledge from data can revolutionize these fields. 
Second, solving complicated nonlinear PDE systems (such as those arising in turbulence and plasticity) is computationally demanding and can often make realistic simulations intractable. Again the possibility of using instances of data to design fast approximate solvers holds great potential for accelerating numerous problems.

\paragraph{Learning PDE Solution Operators.}
 In PDE applications, the governing differential equations are by definition local, whilst the solution operator exhibits non-local properties. Such non-local effects can be described by integral operators explicitly in the spatial domain, or by means of spectral domain multiplication; convolution is an archetypal example. For integral equations, the graph approximations of Nystr\"om type \citep{belongie2002spectral} provide a consistent way of connecting different grid or data structures arising in computational methods and understanding their continuum limits \citep{von2008consistency,trillos2018variational,trillos2020error}. For spectral domain calculations, there are well-developed tools that exist for approximating the continuum \citep{boyd2001chebyshev,trefethen2000spectral}. However, these approaches for approximating integral operators are not data-driven.  Neural networks present a natural approach for learning-based integral operator approximations since they can incorporate non-locality. However, standard neural networks are limited to the discretization of training data and hence, offer a poor approximation to the integral operator. We tackle this issue here by proposing the framework of neural operators. 
 
 

%

\paragraph{Properties of existing deep-learning models.} 

Previous deep learning models are mostly defined on a fixed grid, and  removing, adding, or moving grid points generally makes these models no longer applicable, as seen in Table~\ref{table:deeplearning_comparison}. Thus, they are  not discretization invariant. 
In general, standard neural networks (NN) (such as Multilayer perceptron (MLP), convolution neural networks (CNN), 
Resnet, and Vision Transformers (ViT)) that take the input grid and output grid as finite-dimensional vectors are not discretization-invariant since their input and output have to be at the fixed grid with fixed location. 
On the other hand, the pointwise neural networks used in PINNs \citep{raissi2019physics} that take each coordinate as input are discretization-invariant since it can be applied at each location in parallel. However PINNs only represent the solution function of one instance and it does not learn the map from the input functions to the output solution functions.
A special class of neural networks is convolution neural networks (CNNs). 
CNNs also do not converge with grid refinement since their respective fields change with different input grids. On the other hand,  if normalized by the grid size, CNNs can be applied to uniform grids with different resolutions, which converge to differential operators, in a  similar fashion to the finite difference method.  Interpolation is a baseline approach to achieve discretization-invariance.
While NNs+Interpolation (or in general any finite-dimensional neural networks+Interpolation) are resolution invariant and their outputs can be queried at any point, they are not universal approximators of operators since the dimension of input and output of the internal CNN model is defined to a bounded number. 
DeepONets \citep{lu2019deeponet} are a class of operators that have the universal approximation property. DeepONets consist of a branch net and a trunk net. The trunk net allows queries at any point, but the branch net constrains the input to fixed locations; however it is possible to modify the
branch net to make the methodology discretization invariant, for example by using
the PCA-based approach as used in \citep{de2022cost}.




Furthermore, we show transformers~\citep{vaswani2017attention} are special cases of neural operators with structured kernels that can be used with varying grids to represent the input function. However, the commonly used vision-based extensions of transformers, e.g., ViT~\citep{dosovitskiy2020image}, use convolutions on patches  to generate tokens, and therefore, they are not discretization-invariant models.

We also show that when our proposed neural operators are  applied only on fixed grids, the resulting architectures coincide with neural networks and other operator learning frameworks. In such reductions, point evaluations of the input functions are available on the grid points. In particular, we show that the recent work of DeepONets~\citep{lu2019deeponet}, which are maps from finite-dimensional spaces to infinite dimensional spaces are special cases of neural operators architecture when neural operators are limited only to fixed input grids. Moreover, by introducing an adjustment to the DeepONet architecture, we propose the DeepONet-Operator model that fits into the full operator learning framework of maps between function spaces.

\done{I don't quite get this sentence's meaning: While previous continuous methods have not yielded efficient numerical algorithms that can parallel the success of convolutional or recurrent neural networks in the finite-dimensional setting due to the cost of evaluating integral operators, our work alleviates this issue through the use kernel approximation methods and fast transform algorithms.}




%
%

\section{Learning Operators}
\label{sec:setting}
In subsection~\ref{sec:genericPDE}, we describe the generic setting of PDEs to make the discussions in the following setting concrete.
In subsection \ref{sec:PS}, we outline the general problem of operator learning as well as our approach to solving it. In subsection \ref{sec:discretization}, we discuss the functional data that is available and how we work with it numerically.

\subsection{Generic Parametric PDEs}
\label{sec:genericPDE}
We consider the generic family of PDEs of the following form,
 \begin{align}
\label{eq:generalpde0}
\begin{split}
(\LL_a u)(x) &= f(x), \qquad x \in D, \\
u(x) &= 0, \qquad \quad \:\: x \in \partial D,
\end{split}
\end{align}
for some \(a \in \A\), \(f \in \U^*\) and \(D \subset \R^d\) a bounded domain. We assume that the solution \(u: D \to \R\) lives in the Banach space \(\U\) and  \(\LL_{a}: \A \to \mathcal{L}(\U; \U^*)\) is a mapping from the parameter Banach space \(\A\) to the space of (possibly unbounded) linear operators mapping \(\U\) to its dual \(\U^*\). A natural operator which arises from this PDE is \(\G^\dagger:= \LL_{a}^{-1}f : \A \to \U\) defined to map the parameter to the solution \(a \mapsto u\). A simple example that we study further in Section \ref{ssec:darcy} is when \(\LL_a\) is the weak form of the second-order elliptic operator  \(-\nabla \cdot (a \nabla)\) subject to homogeneous Dirichlet
boundary conditions. In this setting, \(\A = L^\infty(D;\R_+)\), \(\U = H^1_0(D;\R)\), and \(\U^* = H^{-1}(D;\R)\). When needed, we will assume that the domain $D$ is discretized into $K \in \N$ points 
and that we observe $N \in \N$ pairs of coefficient functions and (approximate) solution functions \(\{a^{(i)}, u^{(i)}\}_{i=1}^N\) that are used to train the model (see Section \ref{sec:PS}). We assume that \(a^{(i)}\) are i.i.d. samples from a probability measure \(\mu\) supported on \(\A\) and \(u^{(i)}\) are the pushforwards under \(\G^\dagger\).

\subsection{Problem Setting}
\label{sec:PS}


Our goal is to learn a mapping between two infinite dimensional spaces by using a finite
collection of observations of input-output pairs from this mapping. We make this problem concrete in the following setting. Let \(\A\) and \(\U\) 
be Banach spaces of functions defined on bounded domains $D \subset \R^d$, $D' \subset \R^{d'}$ respectively and \(\Ftrue : \A \to \U\) be
a (typically) non-linear map. Suppose we have observations \(\{a^{(i)}, u^{(i)}\}_{i=1}^N\) where 
\(a^{(i)} \sim \mu\) are i.i.d. samples drawn from some probability measure \(\mu\) supported on 
\(\A\) and \(u^{(i)} = \Ftrue(a^{(i)})\) is possibly corrupted with noise. We aim to build an approximation of \(\Ftrue\) by 
constructing a parametric map 
\begin{equation}
\label{eq:approxmap2}
\F_{\theta} : \A \to \U, \quad \theta \in \R^p
\end{equation}
with parameters from the finite-dimensional space \(\R^p\) and then choosing
\(\theta^\dagger \in \R^p\) so that \(\F_{\theta^\dagger} \approx \Ftrue\).

We will be interested in controlling the error of the approximation on average with respect to \(\mu\). In particular, assuming \(\G^\dagger\) is \(\mu\)-measurable, we will aim to control
the \(L^2_\mu(\A;\U)\) Bochner norm of the approximation
\begin{equation}
\label{eq:bochner_error}
\|\G^\dagger - \G_\theta\|^2_{L^2_\mu(\A;\U)} = \E_{a \sim \mu} \|\G^\dagger(a) - \G_\theta(a)\|_{\U}^2 = \int_\A \|\G^\dagger(a) - \G_\theta(a)\|_{\U}^2 \: d \mu (a).
\end{equation}
This is a natural framework for learning in infinite-dimensions as one could seek to solve the associated  empirical-risk minimization problem
\begin{equation}
\label{eq:empirical_risk}
\min_{\theta \in \R^p} \E_{a \sim \mu} \|\G^\dagger(a) - \G_\theta(a)\|_{\U}^2 \approx \min_{\theta \in \R^p} \frac{1}{N} \sum_{i=1}^N \|u^{(i)} - \G_\theta(a^{(i)})\|_{\U}^2
\end{equation}
which directly parallels the classical finite-dimensional 
setting \citep{Vapnik1998}. 
\nk{As well as using error measured in the Bochner norm, we 
will also consider the setting where error is measured uniformly over compact sets of \(\A\). In particular, given any \(K \subset \A\) compact, we consider
\begin{equation}
    \label{eq:unform_risk}
    \sup_{a \in K} \|\G^\dagger (a) - \G_\theta (a) \|_\U
\end{equation}
which is a more standard error metric in the approximation theory literature. Indeed, the classic approximation theory of neural networks in formulated analogously to equation \eqref{eq:unform_risk} \citep{hornik1989multilayer}. }

In Section~\ref{sec:approximation} we show that, for the
architecture we propose and given any desired error tolerance, there exists \(p \in \N\) and an associated parameter \(\theta^\dagger \in \R^p\), so that the loss \eqref{eq:bochner_error} or \eqref{eq:unform_risk} is less
than the specified tolerance. However, we do not address the challenging open problems of characterizing the error with respect to either (a) a fixed parameter dimension \(p\) or (b) a fixed number of training samples \(N\). Instead, we approach this in the empirical test-train setting where we minimize \eqref{eq:empirical_risk} based on a fixed training set and approximate \eqref{eq:bochner_error} from new samples that were not seen during training.   
Because we conceptualize our methodology in the
infinite-dimensional setting, all finite-dimensional approximations
can share a common set of network parameters which are defined in the (approximation-free) infinite-dimensional setting. In particular, our architecture does not depend on the way the functions \(a^{(i)},u^{(i)}\) are discretized.
. The notation used through out this paper,
along with a useful summary table, may be found in
Appendix \ref{sec:appendix_notation}.

\subsection{Discretization}
\label{sec:discretization}
Since our data \(a^{(i)}\) and \(u^{(i)}\) are, in general, functions, to work with them numerically, we assume access only to their point-wise evaluations. To illustrate this, we will continue with the example of the preceding paragraph. For simplicity, assume \(D = D'\) and suppose that the input and output functions are both real-valued. Let \(D^{(i)} = \{x^{(i)}_{\ell}\}_{\ell=1}^L \subset D\) be a \(L\)-point discretization of the domain \(D\) and assume we have observations \(a^{(i)}|_{D^{(i)}}, u^{(i)}|_{D^{(i)}} \in \R^{L}\), for a finite collection  of input-output pairs indexed by $j$.
In the next section, we propose a kernel inspired graph neural network architecture which, while trained on the discretized data, can produce the solution \(u(x)\) for any \(x \in D\) given an input \(a \sim \mu\). \nk{In particular, our discretized architecture maps into the space \(\U\) and not into a discretization thereof. Furthermore our parametric operator class is consistent, in that, given a fixed set of parameters, refinement of the input discretization converges to the true functions space operator. We make this notion precise in what follows and refer to architectures that possess it as function space architectures, 
mesh-invariant architectures, or \textit{discretization-invariant} architectures.
\footnote{\nk{Note that the meaning of the indexing of sets $D_{\bullet}$ in the following
definition differs that used earlier in this paragraph.}}
\begin{definition}
We call a \textit{discrete refinement} of the domain \(D \subset \R^d\) any sequence of nested sets \(D_1 \subset D_2 \subset \dots \subset D\) with \(|D_L| = L\) for any \(L \in \N\) such that, for any \(\epsilon > 0\), there exists a number \(L = L(\epsilon) \in \N\) such that
\[D \subseteq \bigcup_{x \in D_L} \{y \in \R^d : \|y - x\|_2 < \epsilon\}.\]
\end{definition}
\begin{definition}
Given a discrete refinement \((D_L)_{L=1}^\infty\) of the domain \(D \subset \R^d\), any member \(D_L\) is called a \textit{discretization} of \(D\).
\end{definition}
Since $a:D \subset \R^d \to \R^m$, pointwise evaluation of the
function (discretization) at a set of $L$ points gives rise to
the data set $\{(x_{\ell}, a(x_{\ell}))\}_{\ell=1}^L.$ Note that this
may be viewed as a vector in $\R^{Ld} \times \R^{Lm}.$ An example of the mesh refinement is given in Figure \ref{fig:discretization-invariance}.
\begin{definition}
Suppose \(\A\) is a Banach space of \(\R^m\)-valued functions on the domain \(D \subset \R^d\). Let \(\G : \A \to \U\) be an operator, \(D_L\) be an \(L\)-point discretization of \(D\), and \(\hat{\G} : \R^{Ld} \times \R^{Lm} \to \U\) some map. For any \(K \subset \A\) compact, we define the \textit{discretized uniform risk} as
\[R_K (\G,\hat{\G},D_L) = \sup_{a \in K} \|\hat{\G}(D_L,a|_{D_L}) - \G(a) \|_{\U}.\]
\end{definition}

\begin{definition}\label{def:discretization_invariance}
Let \(\Theta \subseteq \R^p\) be a finite dimensional parameter space and \(\G : \A \times \Theta \to \U\) a map representing a parametric class of operators with parameters \(\theta \in \Theta\). Given a discrete refinement \((D_n)_{n=1}^\infty\) of the domain \(D \subset \R^d\), we say \(\G\) is \textit{discretization-invariant} if there exists a sequence of maps \(\hat{\G}_1, \hat{\G}_2, \dots\) where \(\hat{\G}_L : \R^{Ld} \times \R^{Lm} \times \Theta \to \U\) such that, for any \(\theta \in \Theta\) and any compact set \(K \subset \A\),
\[\lim_{L \to \infty} R_K(\G(\cdot,\theta), \hat{\G}_L(\cdot,\cdot,\theta),D_L) = 0.\]
\end{definition}
We prove that the architectures proposed in Section~\ref{sec:neuraloperators} are discretization-invariant. 
We further verify this claim numerically by showing that the approximation error is 
approximately constant as we refine the discretization.}
Such a property is highly desirable as it allows a transfer of solutions between different grid geometries and discretization sizes with a single architecture that has a fixed number of parameters. 


We note that, while the application of our methodology is based on having point-wise evaluations of the function, it is not limited by it. One may, for example, represent a function numerically as a finite set of truncated basis coefficients. Invariance of the representation would then be with respect to the size of this set. Our methodology can, in principle, be modified to accommodate this scenario through a suitably chosen architecture. We do not pursue this direction in the current work. From the construction of neural operators, when the input and output functions are evaluated on fixed grids, the architecture of neural operators on these fixed grids coincide with the class of neural networks.

%
%

\section{Neural Operators}
\label{sec:neuraloperators}


In this section, we outline the neural operator framework.
We assume that the input functions \(a \in \A\) are \(\R^{d_a}\)-valued and defined on the bounded domain \(D \subset \R^d\) while the output functions \(u \in \U\) are \(\R^{d_u}\)-valued and defined on the bounded domain \(D' \subset \R^{d'}\).
The proposed architecture $\F_{\theta} : \A \to \U$ has the following overall structure:
\begin{enumerate}
    
    \item \textbf{Lifting}: Using a pointwise function \(\R^{d_a} \to \R^{d_{v_0}}\), map the input $\{a: D \to \R^{d_a}\} \mapsto \{v_0: D \to \R^{d_{v_0}}\}$ to its first hidden representation. Usually, we choose \(d_{v_0} > d_a\) and hence this is a lifting operation performed by a fully local operator.
    
    \item \textbf{Iterative Kernel Integration}: For \(t=0,\dots,T-1\), map each hidden representation to the next $\{v_t: D_t \to \R^{d_{v_t}}\} \mapsto \{v_{t+1}: D_{t+1} \to \R^{d_{v_{t+1}}}\}$ via the action of the sum of a local linear operator, a non-local integral kernel operator, and a bias function, composing the
    sum with a fixed, pointwise nonlinearity. Here we set \(D_0 = D\) and \(D_T = D'\) and impose that \(D_t \subset \R^{d_t}\) is a bounded domain.\footnote{\nk{The indexing of sets $D_{\bullet}$ here differs from the two previous indexings used in Subsection \ref{sec:discretization}. The index $t$ is not the physical time, but the iteration (layer) in the model architecture.}}
    
    \item \textbf{Projection}: Using a pointwise function \(\R^{d_{v_T}} \to \R^{d_u}\), map the last hidden representation $\{v_T: D' \to \R^{d_{v_T}}\} \mapsto \{u: D' \to \R^{d_u}\}$ to the output function. Analogously to the first step, we usually pick \(d_{v_T} > d_u\) and hence this is a projection step performed by a fully local operator.
\end{enumerate}

The outlined structure mimics that of a finite dimensional neural network where hidden representations are successively mapped to produce the final output. In particular, we have
\begin{equation}
\label{eq:F}
    \cF_{\theta} \coloneqq \cQ \circ \sigma_T(W_{T-1} + \cK_{T-1} + b_{T-1}) \circ \cdots \circ \sigma_1(W_0 + \cK_0 + b_0) \circ \cP
\end{equation}
where \(\cP: \R^{d_a} \to \R^{d_{v_0}}\), \(\cQ: \R^{d_{v_T}} \to \R^{d_u}\) are the local lifting and projection mappings respectively, \(W_t \in \R^{d_{v_{t+1}} \times d_{v_t}}\) are local linear operators (matrices), \(\cK_t: \{v_t: D_t \to \R^{d_{v_t}}\} \to \{v_{t+1}: D_{t+1} \to \R^{d_{v_{t+1}}}\}\) are integral kernel operators, \(b_t : D_{t+1} \to \R^{d_{v_{t+1}}}\) are bias functions, and \(\sigma_t\) are fixed activation functions acting locally as maps \(\R^{v_{t+1}} \to \R^{v_{t+1}}\) in each layer. The output dimensions \(d_{v_0},\dots,d_{v_T}\) as well as the input dimensions \(d_1,\dots,d_{T-1}\) and domains of definition \(D_1,\dots,D_{T-1}\) are hyperparameters of the architecture. By local maps, we mean that the action is pointwise, in particular, for the lifting and projection maps, we have \((\cP(a))(x) = \cP(a(x))\) for any \(x \in D\) and  \((\cQ(v_T))(x) = \cQ(v_T(x))\) for any \(x \in D'\) and similarly, for the activation,
\((\sigma(v_{t+1}))(x) = \sigma(v_{t+1}(x))\) for any \(x \in D_{t+1}\).
The maps, \(\cP\), \(\cQ\), and \(\sigma_t\) can thus be thought of as defining Nemitskiy operators \citep[Chapters 6,7]{dudley2011concrete} when each of their components are assumed to be Borel measurable. This interpretation allows us to define the general neural operator architecture when pointwise evaluation is not well-defined in the spaces \(\A\) or \(\U\) e.g. when they are Lebesgue, Sobolev, or Besov spaces. 

The crucial difference between the proposed architecture \eqref{eq:F} and a standard feed-forward neural network is that all operations are directly defined in function space (\nk{noting that the activation funtions, $\cP$ and $\cQ$ are all interpreted} through their extension to Nemitskiy operators) and therefore do not depend on any discretization of the data. Intuitively, the lifting step locally maps the data to a space where the non-local part of \(\G^\dagger\) is easier to capture. \nk{We confirm this intuition numerically in Section~\ref{sec:numerics}; however, we note that for the theory presented in Section~\ref{sec:approximation} it suffices that $\cP$ is the identity map.} \nk{The non-local part of \(\G^\dagger\)} is then learned by successively approximating using integral kernel operators composed with a local nonlinearity. Each 
integral kernel operator is the function space analog of the weight matrix in a standard feed-forward network since they are infinite-dimensional linear operators mapping one function space to another. We turn the biases, which are normally vectors, to functions and, using intuition from the ResNet architecture \citep{he2016deep}, we further add a local linear operator acting on the output of the previous layer before applying the nonlinearity. The final projection step simply gets us back to the space of our output function. We concatenate in \(\theta \in \R^p\) the parameters of \(\cP\), \(\cQ\), \(\{b_t\}\) which are usually themselves shallow neural networks, the parameters of the kernels representing \(\{\cK_t\}\) which are again usually shallow neural networks, and the matrices \(\{W_t\}\). We note, however, that our framework is general and other parameterizations such as polynomials may also be employed. 

\paragraph{Integral Kernel Operators} We define three version of the integral kernel operator \(\cK_t\) used in \eqref{eq:F}. For the first, let \(\kappa^{(t)} \in C(D_{t+1} \times D_t; \R^{d_{v_{t+1}} \times d_{v_t}})\) and let \(\nu_t\) be a Borel measure on \(D_t\). Then we define \(\cK_t\) by
\begin{equation}
    \label{eq:kernelop1}
    (\cK_t(v_t))(x) = \int_{D_t} \kappa^{(t)} (x,y) v_t(y) \: \text{d}\nu_t(y)
    \qquad \forall x \in D_{t+1}.
\end{equation}
Normally, we take \(\nu_t\) to simply be the Lebesgue measure on \(\R^{d_t}\) but, as discussed in Section~\ref{sec:four_schemes}, other choices can be used to speed up computation or aid the learning process by building in \textit{a priori} information. The choice of integral kernel operator in \eqref{eq:kernelop1} defines the basic form of the neural operator and is the one we analyze in Section~\ref{sec:approximation} and study most in the numerical experiments of Section~\ref{sec:numerics}.

For the second, let \(\kappa^{(t)} \in C(D_{t+1} \times D_t \times \R^{d_a} \times \R^{d_a}; \R^{d_{v_{t+1}} \times d_{v_t}})\). Then we define \(\cK_t\) by
\begin{equation}
    \label{eq:kernelop2}
    (\cK_t(v_t))(x) = \int_{D_t} \kappa^{(t)} (x,y,a(\Pi_{t+1}^D(x)),a(\Pi_t^D(y))) v_t(y) \: \text{d}\nu_t(y)
    \qquad \forall x \in D_{t+1}.
\end{equation}
where \(\Pi_t^{D} : D_t \to D\) are fixed mappings. We have found numerically that, for certain PDE problems, the form \eqref{eq:kernelop2} outperforms \eqref{eq:kernelop1} due to the strong dependence of the solution \(u\) on the parameters \(a\), \nk{for example, the Darcy flow problem considered in subsection \ref{ssec:DF}}. Indeed, if we think of \eqref{eq:F} as a discrete time dynamical system, then the input \(a \in \A\) only enters through the initial condition hence its influence diminishes with more layers. By directly building in \(a\)-dependence into the kernel, we ensure that it influences the entire architecture.

Lastly, let \(\kappa^{(t)} \in C(D_{t+1} \times D_t \times \R^{d_{v_t}} \times \R^{d_{v_t}}; \R^{d_{v_{t+1}} \times d_{v_t}})\). Then we define \(\cK_t\) by
\begin{equation}
    \label{eq:kernelop3}
    (\cK_t(v_t))(x) = \int_{D_t} \kappa^{(t)} (x,y,v_t(\Pi_t(x)),v_t(y)) v_t(y) \: \text{d}\nu_t(y)
    \qquad \forall x \in D_{t+1}.
\end{equation}
where \(\Pi_t : D_{t+1} \to D_t\) are fixed mappings. Note that, in contrast to \eqref{eq:kernelop1} and \eqref{eq:kernelop2}, the integral operator \eqref{eq:kernelop3} is nonlinear since the kernel can depend on the input function \(v_t\). With this definition and a particular choice of kernel \(\kappa_t\) and measure \(\nu_t\), we show in Section \ref{sec:transformers} that neural operators are a continuous input/output space generalization of the popular transformer architecture \citep{vaswani2017attention}. 

\paragraph{Single Hidden Layer Construction} Having defined possible choices for the integral kernel operator, we are now in a position to explicitly write down a full layer of the architecture defined by \eqref{eq:F}. For simplicity, we choose the integral kernel operator given by \eqref{eq:kernelop1}, but note that the other definitions 
\eqref{eq:kernelop2}, \eqref{eq:kernelop3} work analogously. We then have that a single hidden layer update is given by
\begin{equation}
    \label{eq:onelayer}
    v_{t+1}(x) = \sigma_{t+1} \left ( W_t v_t( \Pi_{t} (x)) + \int_{D_t} \kappa^{(t)} (x,y) v_t(y) \: \text{d}\nu_t(y) + b_t(x) \right ) \qquad \forall x \in D_{t+1}
\end{equation}
where \(\Pi_t : D_{t+1} \to D_t\) are fixed mappings. We remark that, since we often consider functions on the same domain, we usually take \(\Pi_t\) to be the identity. 

We will now give an example of a full single hidden layer architecture i.e. when \(T=2\). We choose \(D_1 = D\), take \(\sigma_2\) as the identity,
and denote \(\sigma_1\) by \(\sigma\), assuming it is any activation function. Furthermore, for simplicity, we set \(W_1 = 0\), \(b_1 = 0\), and assume that \(\nu_0 = \nu_1\) is the Lebesgue measure on \(\R^d\). Then  \eqref{eq:F} becomes
\begin{equation}
    \label{eq:singlehiddenlayer}
    (\G_\theta(a))(x) = \cQ \left (\int_D \kappa^{(1)}(x,y) \sigma \left ( W_0 \cP(a(y)) + \int_D \kappa^{(0)}(y,z) \cP (a(z)) \: \text{d}z + b_0(y) \right ) \: \text{d}y \right)
\end{equation}
for any \(x \in D'\). In this example, \(\cP \in C(\R^{d_a}; \R^{d_{v_0}})\),  \(\kappa^{(0)} \in C(D \times D; \R^{d_{v_1} \times d_{v_0}})\), \(b_0 \in C(D; \R^{d_{v_1}})\), \(W_0 \in \R^{d_{v_1} \times d_{v_0}}\), \(\kappa^{(1)} \in C(D' \times D; \R^{d_{v_2} \times d_{v_1}})\), and \(\cQ \in C(\R^{d_{v_2}}; \R^{d_u})\). One can then parametrize the continuous functions \(\cP,\cQ,\kappa^{(0)},\kappa^{(1)},b_0\) by standard feed-forward neural networks (or by any other means) and the matrix \(W_0\) simply by its entries. The parameter vector \(\theta \in \R^p\) then becomes the concatenation of the parameters of \(\cP,\cQ,\kappa^{(0)},\kappa^{(1)},b_0\) along with the entries of \(W_0\). One can then optimize these parameters by minimizing with respect to \(\theta\) using standard gradient based minimization techniques. To implement this minimization, the functions
entering the loss need to be discretized; but the learned parameters may then be used with other discretizations. In Section~\ref{sec:four_schemes}, we discuss various choices for parametrizing the kernels, picking the integration measure,  
and how those choices affect the computational complexity of the architecture.

\paragraph{Preprocessing} It is often beneficial to manually include features into the input functions \(a\) to help facilitate the learning process. For example, instead of considering the \(\R^{d_a}\)-valued vector field \(a\) as input, we use the \(\R^{d+ d_a}\)-valued vector field \((x,a(x))\). By including the identity element, information about the geometry of the spatial domain \(D\) is directly incorporated into the architecture. This allows the neural networks direct access to information that is already known in the problem and therefore eases learning. We use this idea in all of our numerical experiments in Section~\ref{sec:numerics}. Similarly, when learning a smoothing operator, it may be beneficial to include a smoothed version of the inputs \(a_\epsilon\) using, for example, Gaussian convolution. Derivative information may also be of interest and thus, as input, one may consider, for example, the \(\R^{d + 2d_a + dd_a}\)-valued vector field \((x,a(x),a_\epsilon(x),\nabla_x a_\epsilon (x))\). Many other possibilities may be considered on a problem-specific basis.

\nk{
\paragraph{Discretization Invariance and Approximation}

In light of discretization invariance Theorem~\ref{thm:discretizational_invariance} and universal approximation Theorems~\ref{thm:main_compact}~\ref{thm:cm_compact},~\ref{thm:measurable_approx},~\ref{thm:cm_measurable_approx} whose formal statements are given in Section~\ref{sec:approximation}, we may obtain a decomposition of the total error made by a neural operator as a sum of the discretization error and the approximation error. In particular, given a finite dimensional instantiation of a neural operator \(\hat{\G}_\theta : \R^{Ld} \times \R^{L d_a} \to \U\), for some \(L\)-point discretization of the input, we have 
\[\|\hat{\G}_\theta (D_L, a|_{D_L}) - \G^\dagger (a)\|_\U \leq \underbrace{\|\hat{\G}_\theta (D_L, a|_{D_L}) - \G_\theta (a) \|_\U}_{\text{discretization error}} + \underbrace{\|\G_\theta (a) - \G^\dagger (a) \|_\U}_{\text{approximation error}}.\]
Our approximation theoretic Theorems imply that we can find parameters \(\theta\) so that the approximation error is arbitrarily small while the discretization invariance Theorem states that we can find a fine enough discretization (large enough \(L\)) so that the discretization error is arbitrarily small. Therefore, with a fixed set of parameters independent of the input discretization, a neural operator that is able to be implemented on a computer can approximate operators to arbitrary accuracy.

}

\section{Parameterization and Computation}
\label{sec:four_schemes}
In this section, we discuss various ways of parameterizing the infinite dimensional architecture \eqref{eq:F}, Figure~\ref{fig:NO_architecture}. The goal is to find
an intrinsic infinite dimensional parameterization that achieves
small error (say $\eps$), and then rely on numerical approximation to ensure that this parameterization  delivers an error of the same
magnitude (say $2\eps$), for all
data discretizations fine enough. In this way the number of parameters used to achieve ${\mathcal O}(\eps)$ error is independent of the data discretization. In many applications we have in mind the data
discretization is something we can control, for example when
generating input/output pairs from solution of partial
differential equations via numerical simulation. The proposed
approach allows us to train a neural operator approximation
using data from different discretizations, and to predict with
discretizations different from those used in the data, all by
relating everything to the underlying infinite dimensional
problem.

We also discuss the computational complexity of the proposed parameterizations and suggest algorithms which yield efficient numerical methods for approximation. Subsections~\ref{sec:graphneuraloperator}-\ref{sec:fourier} delineate each of the proposed methods.

To simplify notation, we will only consider a single layer of \eqref{eq:F} i.e. \eqref{eq:onelayer} and choose the input and output domains to be the same. Furthermore, we will drop the layer index \(t\) and write the single layer update as
\begin{equation}
    \label{eq:onelayercompute}
    u(x) = \sigma \left ( W v(x) + \int_D \kappa (x,y) v(y) \: \text{d}\nu(y) + b(x) \right ) \qquad \forall x \in D
\end{equation}
where \(D \subset \R^d\) is a bounded domain, \(v: D \to \R^n\) is the input function and \(u: D \to \R^m\) is the output function. 
\done{Should we comment that little changes if domains of $v$ and $u$
are different?} 
When the domain domains $D$ of $v$ and $u$ are different, we will usually extend them to be on a larger domain.  
We will consider \(\sigma\) to be fixed, and, for the time being, take \(\text{d} \nu(y) = \text{d} y\) to be the Lebesgue measure on \(\R^d\). Equation \eqref{eq:onelayercompute} then leaves three objects which can be parameterized: \(W\), \(\kappa\), and \(b\). Since \(W\) is linear and acts only locally on \(v\), we will always parametrize it by the values of its associated \(m \times n\) matrix; hence \(W \in \R^{m \times n}\) yielding \(mn\) parameters. 
We have found empirically that letting \(b: D \to \R^m\) be a constant function over any domain \(D\) works at least as well as allowing it to be an arbitrary neural network. Perusal of the proof of Theorem~\ref{thm:main_compact} shows that we do not lose any approximation power by doing this, and we reduce the total number of parameters in the architecutre. Therefore we will always parametrize \(b\) by the entries of a fixed \(m\)-dimensional vector; in particular, \(b \in \R^m\) yielding \(m\) parameters. Notice that both parameterizations are independent of any discretization of \(v\). 


The rest of this section will be dedicated to choosing the kernel function \(\kappa : D \times D \to \R^{m \times n}\) and the computation of the associated integral kernel operator. For clarity of exposition, we consider only the simplest proposed version of this operator \eqref{eq:kernelop1} but note that similar ideas may also be applied to \eqref{eq:kernelop2} and \eqref{eq:kernelop3}. Furthermore, in order to focus on learning the kernel $\kappa$, here we drop \(\sigma\), \(W\), and \(b\) from \eqref{eq:onelayercompute} and simply consider the linear update
\begin{equation}
    \label{eq:onelayerlinear}
    u(x) = \int_D \kappa(x,y) v(y) \: \text{d} \nu(y) \qquad \forall x \in D.
\end{equation}
To demonstrate the computational challenges associated with \eqref{eq:onelayerlinear}, let \(\{x_1,\dots,x_J\} \subset D\) be a uniformly-sampled \(J\)-point discretization of \(D\). Recall that we assumed \(\text{d} \nu(y) = \text{d} y\) and, for simplicity, suppose that \(\nu(D)=1\), then the Monte Carlo approximation of \eqref{eq:onelayerlinear} is
\begin{equation}
    \label{eq:onelayerlinear_MC}
    u(x_j) = \frac{1}{J} \sum_{l=1}^J \kappa(x_j, x_l) v(x_l), \qquad j=1,\dots,J
\end{equation}
The integral in \eqref{eq:onelayerlinear} can be approximated using any other integral approximation methods, including the celebrated Riemann sum for which $u(x_j) = \sum_{l=1}^J \kappa(x_j, x_l) v(x_l)\Delta x_l$ and $\Delta x_l$ is the Riemann sum coefficient associated with $\nu$ at $x_l$. For the approximation methods, to compute \(u\) on the entire grid requires \(\mathcal{O}(J^2)\) matrix-vector multiplications. Each of these matrix-vector multiplications requires \(\mathcal{O}(mn)\) operations; for the rest of the discussion, we treat \(mn = \mathcal{O}(1)\) as constant and consider only the cost with respect to \(J\) the discretization parameter since $m$ and $n$ are fixed by
the architecture choice whereas $J$ varies depending on
required discretization accuracy and hence may be arbitrarily
large. This cost is not specific to the Monte Carlo approximation but is generic for quadrature rules which use the entirety of the data. Therefore, when \(J\) is large, computing \eqref{eq:onelayerlinear} becomes intractable and new ideas are needed in order to alleviate this. Subsections~\ref{sec:graphneuraloperator}-\ref{sec:fourier} propose different approaches to the solution to this problem,
inspired by classical methods in numerical analysis.  We finally remark that, in contrast, computations with \(W\), \(b\), and \(\sigma\) only require \(\mathcal{O}(J)\) operations which justifies our focus on computation with the kernel integral operator.

\paragraph{Kernel Matrix.} It will often times be useful to consider the kernel matrix associated to \(\kappa\) for the discrete points \(\{x_1,\dots,x_J\} \subset D\). We define the kernel matrix \(K \in \R^{mJ \times nJ}\) to be the \(J \times J\) block matrix with each block given by the value of the kernel i.e. 
\[K_{jl} = \kappa(x_j, x_l) \in \R^{m \times n}, \qquad j,l=1,\dots,J\]
where we use \((j,l)\) to index an individual block rather than a matrix element. Various numerical algorithms for the efficient computation of \eqref{eq:onelayerlinear} can be derived based on assumptions made about the structure of this matrix, for example, bounds on its rank or sparsity.

\subsection{Graph Neural Operator (GNO)}
\label{sec:graphneuraloperator}

We first outline the Graph Neural Operator (GNO) which approximates \eqref{eq:onelayerlinear} by combining a Nystr\"om approximation with domain truncation and is implemented with the standard framework of graph neural networks. This construction was originally proposed in \cite{li2020neural}.

\paragraph{Nystr\"om Approximation. }
A simple yet effective method to alleviate the cost of computing \eqref{eq:onelayerlinear} is employing a Nystr\"om approximation. This amounts to sampling uniformly at random the points over which we compute the output function \(u\). In particular, let \(x_{k_1},\dots,x_{k_{J'}} \subset \{x_1,\dots,x_J\}\) be \(J' \ll J\) randomly selected points and, assuming \(\nu(D) = 1\), approximate \eqref{eq:onelayerlinear} by
\[u(x_{k_j}) \approx \frac{1}{J'} \sum_{l=1}^{J'} \kappa(x_{k_j}, x_{k_l}) v(x_{k_l}), \qquad j=1,\dots,J'.\]
We can view this as a low-rank approximation to the kernel matrix \(K\), in particular,
\begin{equation}
    \label{eq:nystrom_matrix}
    K \approx K_{J J'} K_{J' J'} K_{J' J}
\end{equation}
where \(K_{J' J'}\) is a \(J' \times J'\) block matrix and \(K_{J J'}\), \(K_{J' J}\) are interpolation matrices, for example, linearly extending the function to the whole domain from the random nodal points. 
\postponed{This last sentence is unclear for me: aren't all the
items in above identity matrices? How can two of them
extend to the whole domain?} 
\postponed{Response: sorry which are identity matrices?}
The complexity of this computation is \(\mathcal{O}(J'^2)\) hence it remains quadratic but only in the number of subsampled points \(J'\) which we assume is much less than the number of points \(J\) in the original  discretization.

\paragraph{Truncation. }
Another simple method to alleviate the cost of computing \eqref{eq:onelayerlinear} is to truncate the integral to a sub-domain of \(D\) which depends on the point of evaluation \(x \in D\). Let \(s : D \to \mathcal{B}(D)\) be a mapping of the points of \(D\) to the Lebesgue measurable subsets of \(D\) denoted \(\mathcal{B}(D)\). Define \(\text{d} \nu (x,y) = \mathds{1}_{s(x)} \text{d}y\) then \eqref{eq:onelayerlinear} becomes
\begin{equation}
    \label{eq:onelayertruncation}
    u(x) = \int_{s(x)} \kappa(x,y) v(y) \: \text{d} y \qquad \forall x \in D.
\end{equation}
If the size of each set \(s(x)\) is smaller than \(D\) then the cost of computing \eqref{eq:onelayertruncation} is \(\mathcal{O}(c_sJ^2)\) where \(c_s  < 1\) is a constant depending on \(s\). While the cost remains quadratic in \(J\), the constant \(c_s\) can have a significant effect in practical computations, as we demonstrate in Section~\ref{sec:numerics}. For simplicity and ease of implementation, we only consider \(s(x) = B(x,r) \cap D\) where \(B(x,r) = \{y \in \R^d : \|y - x\|_{\R^d} < r\}\) for some fixed \(r > 0\). With this choice of \(s\) and assuming that \(D = [0,1]^d\), we can explicitly calculate that \(c_s \approx r^d\).

Furthermore notice that we do not lose any expressive power when we make this approximation so long as we combine it with composition. To see this, consider the example of the previous paragraph where, if we let \(r=\sqrt{2}\), \done{Isn't $r\geq\sqrt{2}$ critical in Euclidean norm when
$D=2$ etc.?} then \eqref{eq:onelayertruncation} reverts to \eqref{eq:onelayerlinear}. Pick 
\(r < 1\) and let \(L \in \N\) with \(L \geq 2\) be the smallest integer such that \(2^{L-1}r \geq 1\). Suppose that \(u(x)\) is computed by composing the right hand side of \eqref{eq:onelayertruncation} \(L\) times with a different kernel every time. The domain of influence of \(u(x)\) is then \(B(x,2^{L-1}r) \cap D = D\) hence it is easy to see that there exist \(L\) kernels such that computing this composition is equivalent to computing \eqref{eq:onelayerlinear} for any given kernel with appropriate regularity. Furthermore the cost of this computation is \(\mathcal{O}(Lr^d J^2)\) and therefore the truncation is beneficial if \(r^d (\log_2 1/r + 1) < 1\)
which holds for any \(r < 1/2\) when \(d = 1\) and any \(r < 1\) when \(d \geq 2\). Therefore we have shown that we can always reduce the cost of computing \eqref{eq:onelayerlinear} by truncation and composition. From the perspective of the kernel matrix, truncation enforces a sparse, block diagonally-dominant structure at each layer. We further explore the hierarchical nature of this computation using the multipole method in subsection~\ref{sec:multipole}.

Besides being a useful computational tool, truncation can also be interpreted as explicitly building local structure into the kernel \(\kappa\). For problems where such structure exists, explicitly enforcing it makes learning more efficient, usually requiring less data to achieve the same generalization error. Many physical systems such as interacting particles in an electric potential exhibit strong local behavior that quickly decays, making truncation a natural approximation technique. 

\paragraph{Graph Neural Networks. }
We utilize the standard architecture of message passing graph networks employing edge features as introduced in \cite{gilmer2017neural} to efficiently implement \eqref{eq:onelayerlinear} on arbitrary discretizations of the domain \(D\). To do so, we treat a discretization \(\{x_1,\dots,x_J\} \subset D\) as the nodes of a weighted, directed graph and assign edges to each node using the function \(s : D \to \mathcal{B}(D)\) which, recall from the section on truncation, assigns to each point a domain of integration. In particular, for \(j=1,\dots,J\), we assign the node \(x_j\) the value \(v(x_j)\) and emanate from it edges to the nodes \(s(x_j) \cap \{x_1,\dots,x_J\} = \mathcal{N}(x_j)\) which we call the neighborhood of \(x_j\). If \(s(x) = D\) then the graph is fully-connected. Generally, the sparsity structure of the graph determines the sparsity of the kernel matrix \(K\), indeed, the adjacency matrix of the graph and the block kernel matrix have the same zero entries. The weights of each edge are assigned as the arguments of the kernel. In particular, for the case of \eqref{eq:onelayerlinear}, the weight of the edge between nodes \(x_j\) and \(x_k\) is simply the concatenation \((x_j,x_k) \in \R^{2d}\).  More complicated weighting functions are considered for the implementation of the integral kernel operators \eqref{eq:kernelop2} or \eqref{eq:kernelop3}.

With the above definition the message passing algorithm of \cite{gilmer2017neural}, with averaging aggregation, updates the value \(v(x_j)\) of the node \(x_j\) to the value \(u(x_j)\) as 
\[u(x_j) = \frac{1}{|\mathcal{N}(x_j)|}\sum_{y \in \mathcal{N}(x_j)} \kappa(x_j, y) v(y), \qquad j=1,\dots,J\]
which corresponds to the Monte-Carlo approximation of the integral \eqref{eq:onelayertruncation}. More sophisticated quadrature rules and adaptive meshes can also be implemented using the general framework of message passing on graphs, see, for example, \cite{pfaff2020learning}. We further utilize this framework in subsection~\ref{sec:multipole}.

\paragraph{Convolutional Neural Networks. } Lastly, we compare and contrast the GNO framework to standard convolutional neural networks (CNNs). 
In computer vision, the success of CNNs has largely been attributed to their ability to capture local features such as edges that can be used to distinguish different objects in a natural image. This property is obtained by enforcing the convolution kernel to have local support, an idea similar to our truncation approximation. Furthermore by directly using a translation invariant kernel, a CNN architecture becomes translation equivariant; this is a desirable feature for many vision models e.g. ones that perform segmentation. We will show that similar ideas can be applied to the neural operator framework to obtain an architecture with built-in local properties and translational symmetries that, unlike CNNs, remain consistent in function space.

To that end, let \(\kappa(x,y) = \kappa(x - y)\) and suppose that \(\kappa : \R^d \to \R^{m \times n}\) is supported on \(B(0,r)\). Let \(r^* > 0\) be the smallest radius such that \(D \subseteq B(x^*, r^*)\) where \(x^* \in \R^d\) denotes the center of mass of \(D\) and suppose \(r \ll r^*\). Then \eqref{eq:onelayerlinear} becomes the convolution
\begin{equation}
    \label{eq:onelayercnn}
    u(x) = (\kappa * v)(x) = \int_{B(x,r) \cap D} \kappa(x-y) v(y) \: \text{d} y \qquad \forall x \in D.
\end{equation}
Notice that \eqref{eq:onelayercnn} is precisely \eqref{eq:onelayertruncation} when \(s(x) = B(x,r) \cap D\) and \(\kappa(x,y) = \kappa(x - y)\). When the kernel is parameterized by e.g. a standard neural network and the radius \(r\) is chosen independently of the data discretization, \eqref{eq:onelayercnn} becomes a layer of a convolution neural network that is consistent in function space. Indeed the parameters of \eqref{eq:onelayercnn} do not depend on any discretization of \(v\). The choice \(\kappa(x,y) = \kappa(x - y)\) enforces translational equivariance in the output while picking \(r\) small enforces locality in the kernel; 
hence we obtain the distinguishing features of a CNN model.

We will now show that, by picking a parameterization that is \textit{inconsistent} 
in function space and applying a Monte Carlo approximation to the integral, \eqref{eq:onelayercnn}  becomes a standard CNN. This is most easily demonstrated when \(D = [0,1]\) and the discretization \(\{x_1,\dots,x_J\}\) is equispaced i.e. \(|x_{j+1} - x_j| = h\) for any \(j=1,\dots,J-1\). Let \(k \in \N\) be an odd filter size 
and let \(z_1,\dots,z_k \in \R\) be the points \(z_j = (j-1-(k-1)/2)h\) for \(j=1,\dots,k\). It is easy to see that \(\{z_1,\dots,z_k\} \subset \bar{B}(0,(k-1)h/2)\) which we choose as the support of \(\kappa\). Furthermore, we parameterize \(\kappa\) directly by its pointwise values which are \(m \times n\) matrices at the locations \(z_1,\dots,z_k\) thus yielding \(kmn\) parameters. Then \eqref{eq:onelayercnn} becomes
\[u(x_j)_p \approx \frac{1}{k} \sum_{l=1}^k \sum_{q=1}^n \kappa(z_l)_{pq} v(x_j - z_l)_q, \qquad j=1,\dots,J, \:\: p=1,\dots,m\]
where we define \(v(x) = 0\) if \(x \not \in \{x_1,\dots,x_J\}\). Up to the constant factor \(1/k\) which can be re-absorbed into the parameterization of \(\kappa\), this is precisely the update of a stride 1 CNN with \(n\) input channels, \(m\) output channels, and zero-padding so that the input and output signals have the same length. This example can easily be generalized to higher dimensions and different CNN structures, we made the current choices for simplicity of exposition. Notice that if we double the amount of discretization points for \(v\) i.e. \(J \mapsto 2J\) and \(h \mapsto h/2\), the support of \(\kappa\) becomes \(\bar{B}(0,(k-1)h/4)\) hence the model changes due to the discretization of the data. Indeed, if we take the limit to the continuum \(J \to \infty\), we find \(\bar{B}(0,(k-1)h/2) \to \{0\}\) hence the model becomes completely local. To fix this, we may try to increase the filter size \(k\) (or equivalently add more layers) simultaneously with \(J\), but then the number of parameters in the model goes to infinity as \(J \to \infty\) since, as we previously noted, there are \(kmn\) parameters in this layer. Therefore standard CNNs are not consistent models in function space. We demonstrate their inability to generalize to different resolutions in Section~\ref{sec:numerics}.

\subsection{Low-rank Neural Operator (LNO)}
\label{sec:lowrank}
By directly imposing that the kernel \(\kappa\) is of a tensor product form, we obtain a layer with \(\mathcal{O}(J)\) computational complexity. We term this construction the Low-rank Neural Operator (LNO) due to its equivalence to directly parameterizing a finite-rank operator. We start by assuming \(\kappa : D \times D \to \R\) is scalar valued and later generalize to the vector valued setting. We express the kernel as
\[\kappa(x,y) = \sum_{j=1}^r \varphi^{(j)}(x) \psi^{(j)}(y) \qquad \forall x,y \in D\]
for some functions \(\varphi^{(1)},\psi^{(1)},\dots,\varphi^{(r)},\psi^{(r)} : D \to \R\) that are normally given as the components of two neural networks \(\varphi, \psi : D \to \R^r\) or a single neural network \(\Xi : D \to \R^{2r}\) which couples all functions through its parameters. With this definition, and supposing that \(n=m=1\), we have that \eqref{eq:onelayerlinear} becomes
\begin{align*}
    u(x) &= \int_D \sum_{j=1}^r \varphi^{(j)}(x) \psi^{(j)}(y) v(y) \: \text{d}y \\ 
    &= \sum_{j=1}^r \int_D \psi^{(j)} (y) v(y) \: \text{d}y \: \varphi^{(j)}(x) \\
    &= \sum_{j=1}^r \langle \psi^{(j)}, v \rangle \varphi^{(j)} (x)
\end{align*}
where \(\langle \cdot, \cdot \rangle\) denotes the \(L^2(D;\R)\) inner product. Notice that the inner products can be evaluated independently of the evaluation point \(x \in D\) hence the computational complexity of this method is \(\mathcal{O}(rJ)\) which is linear in the discretization. 

We may also interpret this choice of kernel as directly parameterizing a rank \(r \in \N\) operator on \(L^2(D;\R)\). Indeed, we have
\begin{equation}
    \label{eq:finiteranksvd}
    u = \sum_{j=1}^r (\varphi^{(j)} \otimes \psi^{(j)}) v 
\end{equation}
which corresponds preceisely to applying the SVD of a rank \(r\) operator to the function \(v\). Equation \eqref{eq:finiteranksvd} makes natural the vector valued generalization. Assume \(m, n \geq 1\) and \(\varphi^{(j)}: D \to \R^m\) and \(\psi^{(j)} : D \to \R^n\) for \(j=1,\dots,r\) then, \eqref{eq:finiteranksvd} defines an operator mapping 
\(L^2(D;\R^n) \to L^2(D;\R^m)\) that can be evaluated as
\[u(x) = \sum_{j=1}^r \langle \psi^{(j)}, v \rangle_{L^2(D;\R^n)} \varphi^{(j)}(x) \qquad \forall x  \in D.\]
We again note the linear computational complexity of this parameterization. Finally, we observe that this method can be interpreted as directly imposing a rank \(r\) structure on the kernel matrix. Indeed,
\[K = K_{Jr} K_{rJ}\]
where \(K_{Jr}, K_{rJ}\) are \(J \times r\) and \(r \times J\) block matricies respectively. 
This construction is similar to the DeepONet construction of \cite{lu2019deeponet} discussed in Section~\ref{sec:deeponets}, but parameterized to be consistent in function space.


\postponed{We have not really discussed our learning problem using the
solution manifold perspective of DeVore/Dahmen as far as I recall;
maybe rephrase in terms of Bochner norm that we have discussed,
or introduce that manifold perspective if it adds something. I am not convinced it does because it  actually is a graph (if output unique
for every input which we implicitly assume) and Bochner
measures the graph or errors between graphs (our approximation
is also a graph).}

\subsection{Multipole Graph Neural Operator (MGNO)}
\label{sec:multipole}
A natural extension to directly working with kernels in a tensor product form as in Section~\ref{sec:lowrank} is to instead consider kernels that can be well approximated by such a form. This assumption gives rise to the fast multipole method (FMM) which employs a multi-scale decomposition of the kernel in order to achieve linear complexity in computing \eqref{eq:onelayerlinear}; for a detailed discussion see e.g. \citep[Section 3.2]{e2011principles}. FMM can be viewed as a systematic approach to combine the sparse and low-rank approximations to the kernel matrix. Indeed, the kernel matrix is decomposed into different ranges and a hierarchy of low-rank structures is imposed on the long-range components.  We employ this idea to construct hierarchical, multi-scale graphs, without being constrained to particular forms of the kernel.  We will elucidate the workings of the FMM through matrix factorization.
This approach was first outlined in \cite{li2020multipole} and is referred as the Multipole Graph Neural Operator (MGNO).

The key to the fast multipole method's linear complexity lies in the subdivision of the kernel matrix according to the range of interaction, as shown in Figure \ref{fig:hmatrix}: 
\begin{equation}\label{eq:decomposition_matrix}
K = K_1 + K_2 + \ldots + K_L,
\end{equation}
where $K_\ell$ with $\ell=1$ corresponds to the shortest-range interaction, and $\ell=L$  corresponds to the longest-range interaction; more generally index $\ell$ is ordered by the
range of interaction. While the uniform grids depicted in Figure \ref{fig:hmatrix} produce an orthogonal decomposition of the kernel, the decomposition may be generalized to arbitrary discretizations  by allowing slight overlap of the ranges.


\begin{figure}[h]
    {\centering
    \includegraphics[width=\textwidth]{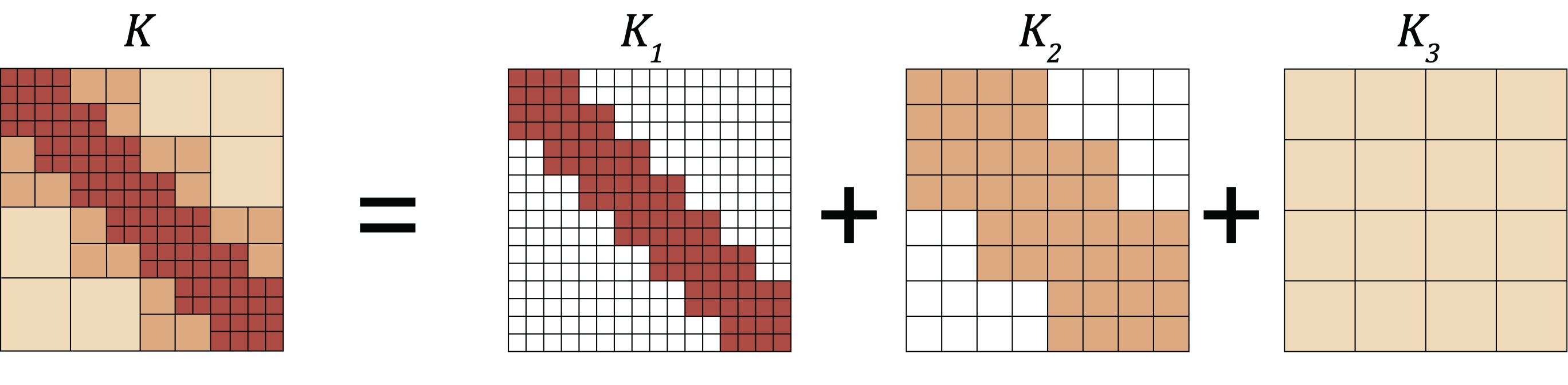}
    \caption{Hierarchical matrix decomposition}
    \label{fig:hmatrix}
    \small{
    The kernel matrix $K$ is decomposed with respect to its interaction ranges. $K_1$ corresponds to short-range interaction; it is sparse but full-rank. $K_3$ corresponds to long-range interaction; it is dense but low-rank.
    }
    }
\end{figure}

\paragraph{Multi-scale Discretization.}
We produce a hierarchy of $L$ discretizations with a decreasing number of nodes $J_1 \geq \ldots \geq J_L$ and increasing kernel integration radius $r_1  \leq \ldots \leq r_L$. Therefore, the shortest-range interaction $K_1$ has a fine resolution but is truncated locally, while the longest-range interaction $K_L$ has a coarse resolution, but covers the entire domain. This is shown pictorially in Figure \ref{fig:hmatrix}.
The number of nodes $J_1 \geq \ldots \geq J_L$, and the integration radii $r_1  \leq \ldots \leq r_L$ are hyperparameter choices and can be picked so that the total computational complexity is linear in \(J\). 

A special case of this construction is when the grid is uniform. Then our formulation reduces to the standard fast multipole algorithm and the kernels $K_l$ form an orthogonal decomposition of the full kernel matrix $K$.
Assuming the underlying discretization \(\{x_1,\dots,x_J\} \subset D\) is a uniform grid with resolution $s$ such that $s^d = J$, the $L$ multi-level discretizations will be grids with resolution $s_l = s/2^{l-1}$, and consequentially $J_l = s_l^d = (s/2^{l-1})^d$ . In this case $r_l$ can be chosen as $1/s$ for \(l=1,\dots,L\).
To ensure orthogonality of the discretizations, the fast multipole algorithm sets the integration domains to be $B(x, r_l) \setminus B(x, r_{l-1})$ for each level \(l=2,\dots,L\), so that the discretization on level $l$ does not overlap with the one on level $l-1$. Details of this algorithm can be found in e.g. \citet{greengard1997new}.

\paragraph{Recursive Low-rank Decomposition.}
The coarse discretization representation can be understood as recursively applying an inducing points approximation \citep{quinonero2005aunifying}: starting from a discretization with $J_1 = J$ nodes, we impose inducing points of size $J_2, J_3,\dots, J_L$  which all admit a low-rank kernel matrix decomposition of the form (\ref{eq:nystrom_matrix}). 
The original $J \times J$ kernel matrix $K_l$ is represented by a much smaller $J_l \times J_l$ kernel matrix, denoted by $K_{l,l}$.
As shown in Figure \ref{fig:hmatrix},  $K_1$ is full-rank but very sparse while $K_L$ is dense but low-rank. Such structure can be achieved by applying equation (\ref{eq:nystrom_matrix}) recursively to equation (\ref{eq:decomposition_matrix}), leading to the multi-resolution matrix factorization \citep{kondor2014multiresolution}:
\begin{equation}\label{eq:hierarchy_decomposition}
K \approx K_{1,1} + K_{1,2}K_{2,2}K_{2,1} + K_{1,2}K_{2,3}K_{3,3}K_{3,2}K_{2,1} + \cdots
\end{equation}
where $ K_{1,1} = K_1$ represents the shortest range, $K_{1,2}K_{2,2}K_{2,1} \approx K_2$, represents the  second shortest range, etc. The center matrix $K_{l,l}$ is a $J_l \times J_l$ kernel matrix corresponding to the $l$-level of the discretization described above. The matrices $K_{l+1,l}, K_{l,l+1}$ are  $J_{l+1} \times J_l$ and $J_{l} \times J_{l+1}$  wide and long respectively block transition matrices. 
Denote $v_l \in R^{J_l \times n}$ for the representation of the input \(v\) at each level of the discretization for \(l=1,\dots,L\), and $u_l \in R^{J_l \times n}$ for the output (assuming the inputs and outputs has the same dimension). We define the matrices $K_{l+1,l}, K_{l,l+1}$ as moving the representation \(v_l\) between different levels of the discretization via an integral kernel that we learn. 
Combining with the truncation idea introduced in subsection~\ref{sec:graphneuraloperator}, we define the transition matrices as discretizations of the following integral kernel operators:

\begin{align}
\label{eq:all1}
&K_{l,l}: v_l \mapsto u_{l} = \int_{B(x,r_{l,l})} \kappa_{l,l}(x,y)v_l(y) \: \text{d}y \\ \label{eq:all2}
&K_{l+1,l}: v_l \mapsto u_{l+1} = \int_{B(x,r_{l+1,l})} \kappa_{l+1, l}(x,y)v_l(y) \: \text{d}y \\
\label{eq:all3}
&K_{l, l+1}: v_{l+1} \mapsto u_{l} = \int_{B(x,r_{l,l+1})} \kappa_{l,l+1}(x,y)v_{l+1}(y) \: \text{d}y 
\end{align}
where each kernel \(\kappa_{l,l'} : D \times D \to \R^{n \times n}\) is parameterized as a neural network and learned.

\begin{figure}[t]
    {\centering
    \includegraphics[width=\textwidth]{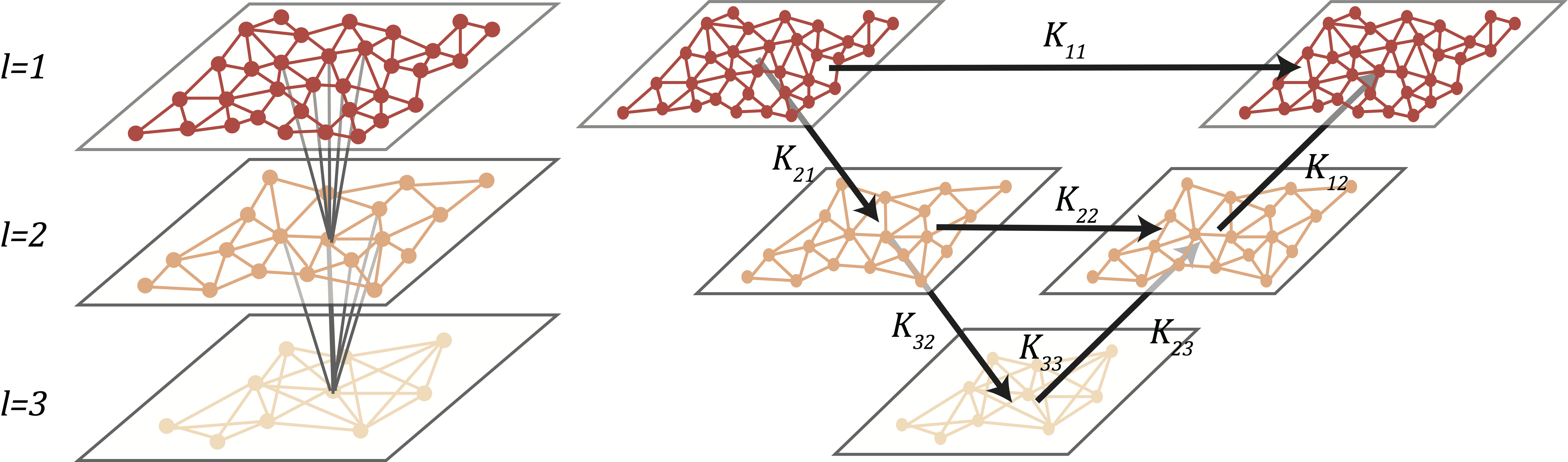}
    \caption{V-cycle}\label{fig:vcycle}
    \label{fig:multigraph}
    \small{
    {\bf Left}: the multi-level discretization.  {\bf Right}: one V-cycle iteration for the multipole neural operator.
    }
    }
\end{figure}

\paragraph{V-cycle Algorithm}
We present a V-cycle algorithm, see Figure \ref{fig:vcycle}, for efficiently computing \eqref{eq:hierarchy_decomposition}. It consists of two steps: the \textit{downward pass} and the \textit{upward pass}. Denote the representation in downward pass and upward pass by $\vdown$ and $\vup$ respectively. In the downward step, the algorithm starts from the fine discretization representation $\vdown_1$ and updates it by applying a downward transition 
$\vdown_{l+1} = K_{l+1,l} \vdown_{l}$.
In the upward step, the algorithm starts from the coarse presentation $\vup_{L}$ and updates it by applying an upward transition and the center kernel matrix
$\vup_{l} = K_{l,l-1} \vup_{l-1} + K_{l,l} \vdown_l$. Notice that 
applying one level downward and upward exactly computes $K_{1,1} + K_{1,2}K_{2,2}K_{2,1}$, and a full $L$-level V-cycle leads to the multi-resolution decomposition \eqref{eq:hierarchy_decomposition}.

Employing \eqref{eq:all1}-\eqref{eq:all3}, we use $L$ neural networks $\kappa_{1,1}, \ldots, \kappa_{L,L}$ to approximate the kernel operators associated to $K_{l,l}$, and $2(L-1)$  neural networks $\kappa_{1,2}, \kappa_{2,1}, \dots$ to approximate the transitions  $K_{l+1,l}, K_{l,l+1}$.
Following the iterative architecture \eqref{eq:F}, we introduce the linear operator \(W \in \R^{n \times n}\)
(denoting it by \(W_l\) for each corresponding resolution) to help regularize the iteration, as well as the nonlinear activation function $\sigma$ to increase the expensiveness. \done{Previous sentence ends without making sense.} Since $W$ acts pointwise (requiring $J$ remains the same for input and output), we employ it only along with the kernel \(K_{l,l}\) and not the transitions. At each layer \(t=0,\dots,T-1\), we perform a full V-cycle as:

\begin{itemize}
    \item  Downward Pass
    \begin{equation}\label{eq:up}
    \textnormal{For}\ l = 1, \ldots, L: 
    \hspace{0.7in}
         \vdown_{l+1}^{(t+1)} = \sigma(\vup_{l+1}^{(t)} +  K_{l+1,l} \vdown_{l}^{(t+1)} )
     \hspace{1.3in}
    \end{equation}
    \item  Upward Pass
    \begin{equation}\label{eq:down}
    \textnormal{For}\ l = L, \ldots, 1: 
    \hspace{0.7in}
       \vup_{l}^{(t+1)} = \sigma( (W_l+ K_{l,l}) \vdown_{l}^{(t+1)}  +  K_{l,l-1} \vup_{l-1}^{(t+1)}).
   \hspace{0.5in}
    \end{equation}
\end{itemize}
Notice that one full pass of the V-cycle algorithm defines a mapping \(v \mapsto u\).

\paragraph{Multi-level Graphs.}
We emphasize that we view the discretization \(\{x_1,\dots,x_J\} \subset D\) as a graph in order to facilitate an efficient implementation through the message passing graph neural network architecture. Since the V-cycle algorithm works at different levels of the discretization, we build multi-level graphs to represent the coarser and finer discretizations. 
We present and utilize two constructions of multi-level graphs, the orthogonal multipole graph and the generalized random graph.
The orthogonal multipole graph is the standard grid construction used in the fast multiple method which is adapted to a uniform grid, see e.g. \citep{greengard1997new}. In this construction, the decomposition in \eqref{eq:decomposition_matrix} is orthogonal in that the finest graph only captures the closest range interaction, the second finest graph captures the second closest interaction minus the part already captured in the previous graph and so on, recursively. In particular, the ranges of interaction for each kernel do not overlap. While this construction is usually efficient, it is limited to uniform grids which may be a bottleneck for certain applications.
Our second construction is the generalized random graph as shown in Figure \ref{fig:hmatrix} where the ranges of the kernels are allowed to overlap. The generalized random graph is very flexible as it can be applied on any domain geometry and discretization. Further it can also be combined with random sampling methods to work on problems where \(J\) is very large or combined with an active learning method to adaptively choose the regions where a finer discretization is needed. 

\paragraph{Linear Complexity.}
Each term in the decomposition \eqref{eq:decomposition_matrix} is represented by the kernel matrix   \(K_{l,l}\) for $l = 1,\ldots, L $, and \(K_{l+1,l}, K_{l,l+1}\) for $l = 1, \ldots, L-1 $ corresponding to the appropriate sub-discretization.
Therefore the complexity of the multipole method is 
 $\sum_{l=1}^L \mathcal{O}(J^2_l r_l^d) + \sum_{l=1}^{L-1}\mathcal{O}(J_l J_{l+1} r_l^d) = \sum_{l=1}^L \mathcal{O}(J^2_l r_l^d)$.
By designing the sub-discretization so that $\mathcal{O}(J^2_l r_l^d) \leq \mathcal{O}(J)$, we can obtain  complexity linear in $J$. 
For example, when $d=2$, pick \(r_l = 1/\sqrt{J_l}\) and \(J_l = \mathcal{O}(2^{-l} J)\) such that \(r_L\) is large enough so that there exists a ball of radius \(r_L\) containing \(D\). Then clearly $\sum_{l=1}^L \mathcal{O}(J^2_l r_l^d) = \mathcal{O}(J)$. 
By combining with a Nystr\"om approximation, we can obtain $\mathcal{O}(J')$ complexity for some \(J' \ll J\).

\subsection{Fourier Neural Operator (FNO)}
\label{sec:fourier}
Instead of working with a kernel directly on the domain \(D\), we may consider its representation in Fourier space and directly parameterize it there. This allows us to utilize Fast Fourier Transform (FFT) methods in order to compute the action of the kernel integral operator  \eqref{eq:onelayerlinear} with almost linear complexity. A similar idea was used in \citep{nelsen2020random} to construct random features in function space The method we outline was first described in \cite{li2020fourier} and is termed the Fourier Neural Operator (FNO). We note that the theory of Section 4 is designed for general kernels and does not apply to the FNO formulation; however, similar universal approximation results were developed for it in \citep{kovachki2021universal} when the input and output spaces are Hilbert space.
For simplicity, we will assume that \(D = \mathbb{T}^d\) is the unit torus and all functions are complex-valued.
Let \(\cG : L^2(D;\C^n) \to \ell^2(\Z^d;\C^n)\) denote the Fourier transform of a function $v: D \to \C^n$ and $\cG^{-1}$ its inverse. For \(v \in L^2(D;\C^n)\) and \(w \in \ell^2(\Z^d;\C^n)\), we have 
\begin{align*}
    &(\cG v)_j (k) = \langle v_j, \psi_k \rangle_{L^2(D;\C)}, \qquad \qquad j \in \{1,\dots,n\}, \quad k \in \Z^d, \\
    &(\cG^{-1} w)_j (x) = \sum_{k \in \Z^d} w_j (k) \psi_k (x), \qquad j \in \{1,\dots,n\}, \quad x \in D
\end{align*}
where, for each \(k \in \Z^d\), we define
\[\psi_k (x) = e^{2\pi i k_1 x_1} \cdots e^{2 \pi i k_d x_d}, \qquad x \in D\]
with \(i = \sqrt{-1}\) the imaginary unit. By letting $\kappa(x,y) = \kappa(x-y)$ for some \(\kappa : D \to \C^{m \times n}\) in \eqref{eq:onelayerlinear} and applying the convolution theorem, we find that
\[u(x) = \cG^{-1} \bigl( \cG(\kappa) \cdot \cG(v) \bigr )(x) \qquad \forall x \in D.\]
We therefore propose to directly parameterize $\kappa$ by its Fourier coefficients. We write
\begin{equation}
    \label{eq:fourierlayer}
    u(x) = \cG^{-1} \bigl( R_\phi \cdot \cG(v) \bigr )(x) \qquad \forall x \in D
\end{equation}
where $R_\phi$ is the Fourier transform of a periodic function $\kappa: D \to \C^{m \times n}$ parameterized by some \(\phi \in \R^p\).

\begin{figure}
    \centering
    \includegraphics[width=\textwidth]{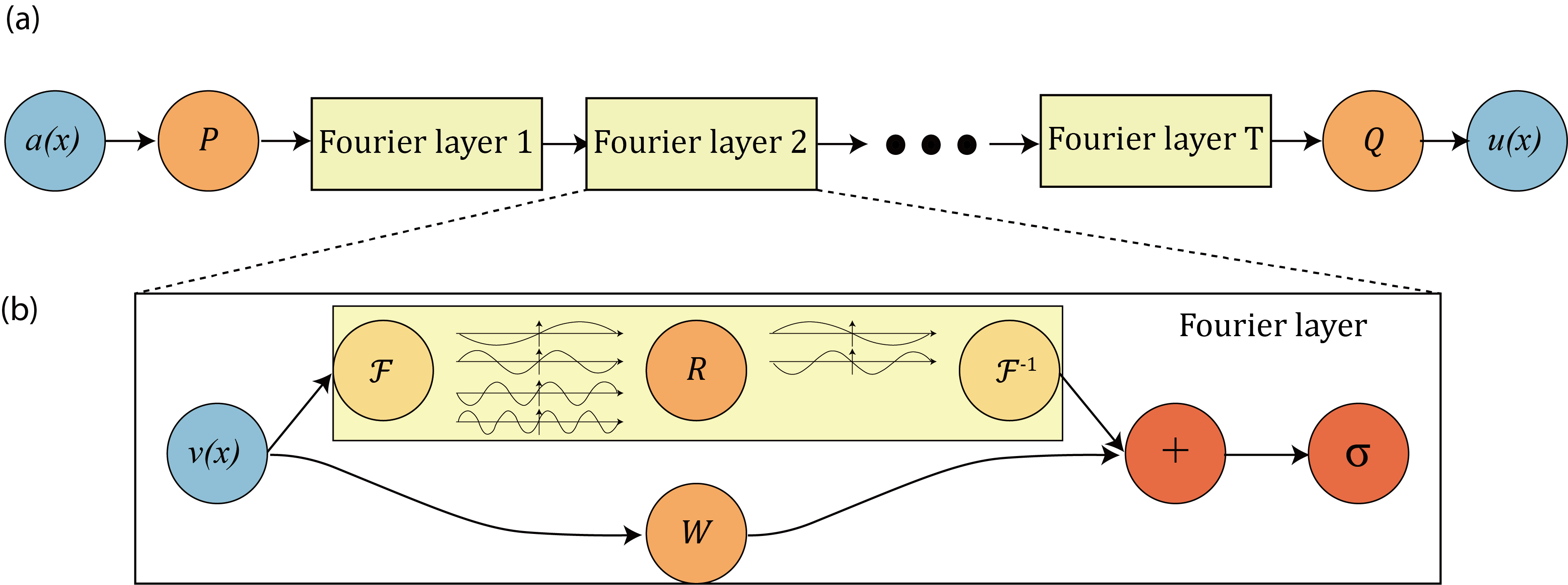}
        \caption{ {\bf top:} The architecture of the neural operators; \textbf{bottom:} Fourier layer.}
    \label{fig:arch}
    \small{
    {\bf (a) The full architecture of neural operator}: start from input $a$. 1. Lift to a higher dimension channel space by a neural network $\mathcal{P}$. 2. Apply $T$ (typically $T=4$) layers of integral operators and activation functions. 3. Project back to the target dimension by a neural network $Q$. Output $u$.
    {\bf (b) Fourier layers}: Start from input $v$. On top: apply the Fourier transform $\cG$; a linear transform $R$ on the lower Fourier modes which also filters out the higher modes; then apply the inverse Fourier transform $\cG^{-1}$. On the bottom: apply a local linear transform $W$.
    }
\end{figure}

For frequency mode \(k \in \Z^d\), we have $(\cG v)(k) \in \C^{n}$ and $R_\phi(k) \in \C^{m \times n}$. We pick a finite-dimensional parameterization by truncating the Fourier series at a maximal number of modes 
\(k_{\text{max}} = |Z_{k_{\text{max}}}| = |\{k \in \mathbb{Z}^d : |k_j| \leq k_{\text{max},j}, \text{ for } j=1,\dots,d\}|.\) This choice improves the empirical performance and sensitivity of the resulting model with respect to the choices of discretization.
We thus parameterize $R_\phi$ directly as complex-valued $(k_{\text{max}} \times m \times n)$-tensor comprising a collection of truncated Fourier modes and therefore drop $\phi$ from our notation. 
In the case where we have real-valued  \(v\) and we want \(u\) to also be real-valued, we impose that $\kappa$ is real-valued by enforcing conjugate symmetry in the parameterization i.e.
\[R(-k)_{j,l} = R^*(k)_{j,l} \qquad \forall k \in Z_{k_{\text{max}}}, \quad j=1,\dots,m, \:\: l=1,\dots,n.\]
We note that the set $Z_{k_{\text{max}}}$ is not the canonical choice for the low frequency modes of $v_t$. Indeed, the low frequency modes are usually defined by placing an upper-bound on the $\ell_1$-norm of $k \in \mathbb{Z}^d$. We choose $Z_{k_{\text{max}}}$ as above since it allows for an efficient implementation.  Figure~\ref{fig:arch} gives a pictorial representation of an entire Neural Operator architecture employing Fourier layers.


\paragraph{The Discrete Case and the FFT.}
Assuming the domain $D$ is discretized with $J \in \mathbb{N}$ points, we can treat $v \in \C^{J \times n}$ and $\cG (v) \in \C^{J \times n}$.
Since we convolve $v$ with a function which only has $k_{\text{max}}$ Fourier modes, we may simply truncate the higher modes to obtain $\cG (v) \in \C^{k_{\text{max}} \times n}$. Multiplication by the weight tensor $R \in \C^{k_{\text{max}} \times m \times n}$ is then
\begin{equation}
\label{eq:fft_mult}
\bigl( R \cdot (\cG v_t) \bigr)_{k,l} = \sum_{j=1}^{n} R_{k,l,j}  (\cG v)_{k,j}, \qquad k=1,\dots,k_{\text{max}}, \quad l=1,\dots,m.
\end{equation}
When the discretization is uniform with resolution \(s_1 \times  \cdots \times s_d = J\), $\cG$ can be replaced by the Fast Fourier Transform. For $v \in \C^{J \times n}$,   $k = (k_1, \ldots, k_{d}) \in \bbZ_{s_1} \times \cdots \times \bbZ_{s_d}$, and $x=(x_1, \ldots, x_{d}) \in D$, the FFT $\hat{\cG}$ and its inverse $\hat{\cG}^{-1}$ are defined as
\begin{align*}
    (\hat{\cG} v)_l(k) = \sum_{x_1=0}^{s_1-1} \cdots \sum_{x_{d}=0}^{s_d-1} v_l(x_1, \ldots, x_{d}) e^{- 2i \pi \sum_{j=1}^{d} \frac{x_j k_j}{s_j} }, \\
    (\hat{\cG}^{-1} v)_l(x) = \sum_{k_1=0}^{s_1-1} \cdots \sum_{k_{d}=0}^{s_d-1} v_l(k_1, \ldots, k_{d}) e^{2i \pi \sum_{j=1}^{d} \frac{x_j k_j}{s_j} }
\end{align*}
for $l=1,\dots,n$. 
In this case, the set of truncated modes becomes
\[Z_{k_{\text{max}}} = \{(k_1, \ldots, k_{d}) \in \bbZ_{s_1} \times \cdots \times \bbZ_{s_d} \mid k_j \leq k_{\text{max},j} \text{ or }\ s_j-k_j \leq k_{\text{max},j}, \text{ for } j=1,\dots,d\}.\]
When implemented, $R$ is treated as a $(s_1 \times \cdots \times s_d \times m \times n)$-tensor and the above definition of $Z_{k_{\text{max}}}$ corresponds to the ``corners'' of $R$, which allows for a straight-forward parallel implementation of (\ref{eq:fft_mult}) via matrix-vector multiplication. 
In practice, we have found the choice $k_{\text{max},j}$ roughly around $\frac{1}{3}$ to $\frac{2}{3}$ of the maximum number of Fourier modes in the Fast Fourier Transform of the grid valuation of the input function provides desirable performance. In our empirical studies, we set $k_{\text{max},j}=12$  which yields $k_{\text{max}} = 12^d$ parameters per channel, to be sufficient for all the tasks that we consider.

\paragraph{Choices for $R$.} In general, $R$ can be defined to depend on $(\cG a)$, the Fourier transform of the input \(a \in \A\) to parallel our construction \eqref{eq:kernelop2}.
Indeed, we can define $R_\phi: \mathbb{Z}^d \times \C^{d_a} \to \C^{m \times n}$
as a parametric function that maps \(\bigl(k,(\cG a)(k))\) to the values of the appropriate Fourier modes. We have experimented with the following parameterizations of $R_\phi$:
\begin{itemize}
    \item \textit{Direct. } Define the parameters $\phi_k \in \C^{m \times n }$ for each wave number $k$:
    \[ R_\phi \bigl(k,(\cG a)(k)\bigr) := \phi_k.\]
    \item \textit{Linear. } Define the parameters \(\phi_{k_1} \in \C^{m \times n \times d_a}\), \(\phi_{k_2} \in \C^{m \times n}\) for each wave number $k$:
    \[R_\phi \bigl(k,(\cG a)(k)\bigr) := \phi_{k_1} (\cG a)(k) + \phi_{k_2}.\]
    \item \textit{Feed-forward neural network.}
    Let $\Phi_\phi:  \mathbb{Z}^d \times \C^{d_a} \to \C^{m \times n}$ be a neural network with parameters $\phi$:
    \[R_\phi \bigl(k,(\cG a)(k)\bigr) := \Phi_\phi(k, (\cG a)(k)). \]
\end{itemize}
We find that the \textit{linear} parameterization has a similar performance to the \textit{direct} parameterization above, \done{use label to point to this direct parameterization}
however, it is not as efficient both in terms of computational complexity and the number of parameters required.
On the other hand, we find that the \textit{feed-forward neural network} parameterization has a worse performance.  This is likely due to the discrete structure of the space $\mathbb{Z}^d;$ numerical evidence suggests neural networks are not adept at handling this structure. Our experiments in this work focus on the direct parameterization presented above.

\paragraph{Invariance to Discretization.}
The Fourier layers are discretization-invariant because they can learn from and evaluate functions which are discretized in an arbitrary way. Since parameters are learned directly in Fourier space, resolving the functions in physical space simply amounts to projecting on the basis elements $e^{2\pi i \langle x, k \rangle}$; these are well-defined everywhere on $\C^d$. 


\paragraph{Quasi-linear Complexity.}
The weight tensor $R$ contains $k_{\text{max}} < J$ modes, so the inner multiplication has complexity $\mathcal{O}(k_{\text{max}})$. Therefore, the majority of the computational cost lies in computing the Fourier transform $\cG(v)$ and its inverse. General Fourier transforms have complexity $\mathcal{O}(J^2)$, however, since we truncate the series the complexity is in fact $\mathcal{O}(J k_{\text{max}})$, while the FFT has complexity $\mathcal{O}(J \log J)$. Generally, we have found using FFTs to be very efficient, however, a uniform discretization is required.

\paragraph{Non-uniform and Non-periodic Geometry.}
The Fourier neural operator model is defined based on Fourier transform operations accompanied by local residual operations and potentially additive bias function terms. These operations are mainly defined on general geometries, function spaces, and choices of discretization. They are not limited to rectangular domains, periodic functions, or uniform grids. In this paper, we instantiate these operations on uniform grids and periodic functions in order to develop fast implementations that enjoy spectral convergence and utilize methods such as fast Fourier transform. In order to maintain a fast and memory-efficient method, our implementation of the Fourier neural operator relies on the fast Fourier transform which is only defined on uniform mesh discretizations of \(D = \mathbb{T}^d\), or for functions on the square satisfying homogeneous Dirichlet (fast Fourier sine transform) or
homogeneous Neumann (fast Fourier cosine transform) boundary conditions.
However, the fast implementation of Fourier neural operator  can be applied in more general geometries via Fourier continuations. Given any compact manifold $D = \mathcal{M}$, we can always embed it into a periodic cube (torus), 
\[i: \mathcal{M} \to \mathbb{T}^d\]
where the regular FFT can be applied. Conventionally, in numerical analysis
applications, the embedding $i$ is defined through a continuous extension by fitting polynomials \citep{bruno2007accurate}. However, in the Fourier neural operator, the idea can be applied simply by padding the input with zeros. The loss is computed only on the original space during training. The Fourier neural operator will automatically generate a smooth extension to the padded domain in the output space.

\subsection{Summary}
We summarize the main computational approaches presented in this section and their complexity:
\begin{itemize}

\item \textbf{GNO}:    Subsample $J'$ points from the $J$-point discretization and compute the truncated integral
\begin{equation}
u(x)=   
\int_{B(x,r)} \kappa (x,y) v(y) \: \text{d}y 
\end{equation}
at a \(\mathcal{O}(J J')\) complexity.

\item \textbf{LNO}:   Decompose the kernel function tensor product form and compute
\begin{equation}
u(x) =   \sum_{j=1}^r \langle \psi^{(j)}, v \rangle \varphi^{(j)}(x)
\end{equation}
at a \(\mathcal{O}(J)\) complexity.

\item \textbf{MGNO}: Compute a multi-scale decomposition of the kernel 
\begin{align}
\begin{split}
K &= K_{1,1} + K_{1,2}K_{2,2}K_{2,1} + K_{1,2}K_{2,3}K_{3,3}K_{3,2}K_{2,1} + \cdots \\
u(x) &= (Kv)(x)
\end{split}
\end{align}
at a \(\mathcal{O}(J)\) complexity.

\item \textbf{FNO}: Parameterize the kernel in the Fourier domain and compute the using the FFT 
\begin{equation}
u(x) = \cG^{-1} \bigl( R_\phi \cdot \cG(v) \bigr )(x)
\end{equation}
at a \(\mathcal{O}(J \log J)\) complexity.
\end{itemize}

%
%


%
%

\section{Neural Operators and Other Deep Learning Models}
\label{sec:framework}

In this section, we provide a discussion on the recent related methods, in particular, DeepONets, and demonstrate that their architectures are subsumed by generic neural operators when neural operators are parametrized inconsistently ~\ref{sec:deeponets}. When only applied and queried on fixed grids, we show neural operator architectures subsume neural networks and, furthermore, we show how transformers are special cases of neural operators~\ref{sec:transformers}.

\subsection{DeepONets}
\label{sec:deeponets}

We will now draw a parallel between the recently proposed DeepONet architecture in \citet{lu2019deeponet}, a map from finite-dimensional spaces to function spaces, and the neural operator framework. \nk{We will show that if we use a particular, point-wise parameterization of the first kernel in a NO and discretize the integral operator, we obtain a DeepONet. However, such a parameterization breaks the notion of discretization invariance because the number of parameters depends on the discretization of the input function. Therefore such a model cannot be applied to arbitrarily discretized functions and its number of parameters goes to infinity as we take the limit to the continuum. This phenomenon is similar to our discussion in subsection~\ref{sec:graphneuraloperator} where a NO parametrization which is inconsistent in function space and breaks discretization invariance yields a CNN.
We propose a modification to the DeepONet architecture, based on the idea of the LNO, which addresses this issue and gives a discretization invariant neural operator.}

\begin{proposition}
\label{prop:deeponet}
A neural operator with a point-wise parameterized first kernel and discretized integral operators yields a DeepONet.
\end{proposition}







\begin{proof}
We work with \eqref{eq:singlehiddenlayer} where we choose \(W_0 = 0\) and denote \(b_0\) by \(b\). For simplicity, we will consider only real-valued functions i.e. \(d_a = d_u = 1\) and set \(d_{v_0} = d_{v_1} = n \) and \( d_{v_2} = p\)
for some \(n, p \in \N\).
Define \(\cP : \R \to \R^n\) by \(\cP(x) = (x,\dots,x)\) and 
\(\cQ: \R^p \to \R\) by \(\cQ(x) = x_1 + \dots + x_p\). Furthermore let \(\kappa^{(1)}: D' \times D \to \R^{p \times n}\) be defined by some \(\kappa^{(1)}_{jk} : D' \times D \to \R\) for \(j=1,\dots,p\) and \(k=1,\dots,n\). Similarly let 
\(\kappa^{(0)}: D \times D \to \R^{n \times n}\) be given as \(\kappa^{(0)}(x,y) = \text{diag}(\kappa^{(0)}_1(x,y), \dots, \kappa^{(0)}_n(x,y))\) for some \(\kappa^{(0)}_1, \dots \kappa^{(0)}_n : D \times D \to \R\). Then \eqref{eq:singlehiddenlayer} becomes 
\[(\G_\theta (a))(x) = \sum_{k=1}^p \sum_{j=1}^n \int_D \kappa^{(1)}_{jk}(x, y) \sigma \left ( \int_D \kappa^{(0)}_j (y,z) a(z) \: \text{d}z + b_j(y) \right ) \: \text{d}y\]
where \(b(y) = (b_1(y),\dots,b_n(y))\) for some \(b_1,\dots,b_n : D \to \R\). Let \(x_1,\dots,x_q \in D\) be the points at which the input function \(a\) is evaluated and denote by \(\tilde{a} = \big( a(x_1), \dots, a(x_q) \big ) \in \R^q\) the vector of evaluations.
Choose \(\kappa^{(0)}_{j} (y,z) = \mathds{1}(y) w_j(z)\) for some \(w_1,\dots,w_n : D \to \R\) where \(\mathds{1}\) denotes the constant function taking the value one. Let
\[w_j(x_l) = \frac{q}{|D|} \tilde{w}_{jl}\]
for \(j=1,\dots,n\) and \(l=1,\dots,q\) where \(\tilde{w}_{jl} \in \R\) are some constants. Furthermore let \(b_j (y) = \tilde{b}_j \mathds{1}(y)\) for some constants \(\tilde{b}_j \in \R\). Then the Monte Carlo approximation of the inner-integral yields
\[(\G_\theta (a))(x) = \sum_{k=1}^p \sum_{j=1}^n \int_D \kappa^{(1)}_{jk}(x, y) \sigma \left ( \langle \tilde{w}_j, \tilde{a} \rangle_{\R^q} + \tilde{b}_j \right ) \mathds{1}(y) \: \text{d}y\]
where \(\tilde{w}_j = \big ( \tilde{w}_{j1}, \dots, \tilde{w}_{jq} \big ) \). Choose \(\kappa^{(1)}_{jk}(x,y) = (\tilde{c}_{jk} / |D|) \varphi_k (x) \mathds{1}(y)\) for some constants \(\tilde{c}_{jk} \in \R\) and functions \(\varphi_1, \dots, \varphi_p : D' \to \R\). Then we obtain
\begin{equation}
\label{eq:likedeeponet}
(\G_\theta (a))(x) = \sum_{k=1}^p \left ( \sum_{j=1}^n \tilde{c}_{jk} \sigma \left ( \langle \tilde{w}_j, \tilde{a} \rangle_{\R^q} + \tilde{b}_j \right ) \right ) \varphi_k(x) = \sum_{k=1}^p G_k(\tilde{v}) \varphi_k (x)
\end{equation}
where \(G_k : \R^q \to \R\) can be viewed as the components of a single hidden layer neural network \(G : \R^q \to \R^p\) with parameters \(\tilde{w}_{jl}, \tilde{b}_j, \tilde{c}_{jk}\). The set of maps \(\varphi_1,\dots,\varphi_p\) form the \textit{trunk net} while \(G\) is the \textit{branch net} of a DeepONet. Our construction above can clearly be generalized to yield arbitrary depth branch nets by adding more kernel integral layers, and, similarly, the trunk net can be chosen arbitrarily deep by parameterizing each \(\varphi_k\) as a deep neural network. 
\end{proof}

\nk{
Since the mappings \(w_1,\dots,w_n\) are point-wise parametrized based on the input values \(\tilde{a}\), it is clear that the construction in the above proof is not discretization invariant. In order to make this model a discretization invariant neural operator, we propose DeepONet-Operator where, for each $j$, we replace the inner product in the finite dimensional space $\langle \tilde{w}_j, \tilde{a} \rangle_{\R^q}$ with an appropriate inner product in the function space $\langle w_j, a \rangle$. 

\begin{equation}
\label{eq:likedeeponet-operator}
(\G_\theta (a))(x) = \sum_{k=1}^p \left ( \sum_{j=1}^n \tilde{c}_{jk} \sigma \left ( \langle {w}_j, {a} \rangle + \tilde{b}_j \right ) \right ) \varphi_k(x)
\end{equation}
This operation is a projection of function $a$ onto $w_j$. Parametrizing $w_j$ by neural networks makes DeepONet-Operator a discretization invariant model.
}

\nk{
There are other ways in which the issue can be resolved for DeepONets. For example, by fixing the set of points on which the input function is evaluated independently of its discretization, by taking local spatial averages as in \citep{lanthaler2021error} or more generally by taking a set of linear functionals on $\A$ as input to
a finite-dimensional branch neural network (a generalization to DeeONet-Operator) as in
the specific PCA-based variant on DeepONet in \citep{de2022cost}.} We demonstrate numerically in Section~\ref{sec:numerics} that, when applied in the standard way, the error incurred by DeepONet(s) grows with the  discretization of \(a\) while it remains constant for neural operators.

\paragraph{Linear Approximation and Nonlinear Approximation.}
We point out that parametrizations of the form \eqref{eq:likedeeponet} fall within the class of \textit{linear} approximation methods since the nonlinear space \(\G^\dagger(\A)\) is approximated by the linear space \(\text{span}\{\varphi_1,\dots,\varphi_p\}\) \citep{devore1998nonlinear}. The quality of the best possible linear approximation to a nonlinear space is given by the Kolmogorov \(n\)-width where \(n\) is the dimension of the linear space used in the approximation \citep{pinkus1985nwidths}. The rate of decay of the \(n\)-width as a function of \(n\) quantifies how well the linear space approximates the nonlinear one. It is well know that for some problems such as the flow maps of advection-dominated PDEs, the \(n\)-widths decay very slowly; hence a very large \(n\) is needed for a good approximation for such problems \citep{cohendevore}. This can be limiting in practice as more parameters are needed in order to describe more basis functions \(\varphi_j\) and therefore more data is needed to fit these parameters. 

On the other hand, we point out that parametrizations of the form
\eqref{eq:F}, and the particular case \eqref{eq:singlehiddenlayer},  constitute (in general) a form of \textit{nonlinear} approximation. The benefits of nonlinear approximation are well understood in the setting of function approximation, see e.g. \citep{devore1998nonlinear}, however the theory for the operator setting is still in its infancy \citep{bonito2020nonlinear, cohen2020optimal}. We observe numerically in Section~\ref{sec:numerics} that nonlinear parametrizations such as \eqref{eq:singlehiddenlayer} outperform linear ones such as DeepONets or the low-rank method introduced in Section \ref{sec:lowrank} when implemented with similar numbers of parameters.
\hold{We acknowledge, however, that the theory presented in Section~\ref{sec:approximation} is based on the reduction to a linear approximation and therefore does not capture the benefits of the nonlinear approximation. Furthermore, in practice, we have found that deeper architectures than \eqref{eq:singlehiddenlayer} (usually four to five layers are used in the experiments of Section~\ref{sec:numerics}), perform better. The benefits of depth are again not captured in our analysis in Section~\ref{sec:approximation}
either. We leave further theoretical studies of approximation
properties as an interesting avenue of investigation for future work.}

\paragraph{Function Representation.}
An important difference between neural operators, introduced here, PCA-based
operator approximation, introduced in \cite{Kovachki} and DeepONets, introduced
in \cite{lu2019deeponet}, is the manner in which the output function space is finite-dimensionalized. Neural operators as implemented in this
paper typically use the same finite-dimensionalization in both the input and output function spaces; however different variants of the neural operator idea use different
 finite-dimensionalizations. As discussed in Section \ref{sec:four_schemes}, the GNO and MGNO are finite-dimensionalized using pointwise values as the nodes of graphs; the FNO is
 finite-dimensionalized in Fourier space, requiring finite-dimensionalization on a uniform grid in real space; the Low-rank neural operator is finite-dimensionalized on a product space formed from the
 Barron space of neural networks. The PCA approach finite-dimensionalizes in the span of PCA modes. DeepONet, on the other hand, uses different input and output space finite-dimensionalizations; in its basic form it uses pointwise (grid)
 values on the input (branch net) whilst its output (trunk net) is represented as a function in Barron space. There also exist POD-DeepONet variants that 
 finite-dimensionalize the output in the span of PCA modes \cite{lu2021comprehensive}, bringing them
 closer to the method introduced in \cite{Kovachki}, but with a different 
 finite-dimensionalization of the input space.

As is widely quoted, ``all models are wrong, but some are useful'' \cite{box1976science}.
For operator approximation, 
each finite-dimensionalization has its own induced biases and limitations,
and therefore works best on a subset of problems. Finite-dimensionalization introduces a trade-off between flexibility and representation power of the resulting approximate architecture.
The Barron space representation (Low-rank operator and DeepONet) is usually the most generic and flexible as it is widely applicable. However this can lead to induced biases and reduced representation power on specific problems; in practice, DeepONet sometimes needs problem-specific feature engineering and architecture choices as studied in \cite{lu2021comprehensive}. We conjecture that these problem-specific features compensate
for the induced bias and reduced representation power that the basic form of the
method \citep{lu2019deeponet} sometimes exhibits.
The PCA (PCA operator, POD-DeepONet) and graph-based (GNO, MGNO) discretizations are also generic, but more specific compared to the DeepONet representation; for this reason
POD-DeepONet can outperform DeepONet on some problems \citep{lu2021comprehensive}.
On the other hand, the uniform grid-based representation FNO is the most specific of
all those operator approximators considered in this paper: in its basic form it
applies by discretizing the input functions, assumed to be specified on a periodic
domain, on a uniform grid. As shown in Section \ref{sec:numerics} FNO usually works
out of the box on such problems. But, as a trade-off, it requires substantial 
additional treatments to work well on non-uniform geometries, such as extension, interpolation (explored in
\cite{lu2021comprehensive}), and Fourier continuation \citep{bruno2007accurate}.

\subsection{Transformers as a Special Case of Neural Operators}
\label{sec:transformers}
\nk{
We will now show that our neural operator framework can be viewed as a  continuum generalization to the popular transformer architecture \citep{vaswani2017attention} which has been extremely successful in natural language processing tasks \citep{devlin2018bert, brown2020language} and, more recently, is becoming a popular choice in computer vision tasks \citep{dosovitskiy2020image}. The parallel stems from the fact that we can view sequences of arbitrary length as arbitrary discretizations of functions. Indeed, in the context of natural language processing,  we may think of a sentence as a ``word''-valued function on, for example, the domain \([0,1]\). Assuming our function is linked to a sentence with a fixed semantic meaning, adding or removing words from the sentence simply corresponds to refining or coarsening the discretization of \([0,1]\). We will now make this intuition precise in the proof of the following statement. 
\begin{proposition}\label{prop:attention}
The attention mechanism in transformer models is a special case of a neural operator layer.
\end{proposition}
\begin{proof}
We will show that by making a particular choice of the nonlinear integral kernel operator \eqref{eq:kernelop3} and discretizing the integral by a Monte-Carlo approximation, a neural operator layer reduces to a pre-normalized, single-headed attention, transformer block as originally proposed in \citep{vaswani2017attention}.
For simplicity, we assume \(d_{v_t} = n \in \N\) and that \(D_t = D\) for any \(t=0,\dots,T\), the bias term is zero, and \(W = I\) is the identity. Furthermore, to simplify notation, we will drop the layer index \(t\) from \eqref{eq:onelayer} and, employing \eqref{eq:kernelop3}, obtain
\begin{equation}
    \label{eq:liketransformer}
    u(x) = \sigma \left( v(x) + \int_D \kappa_v(x,y,v(x),v(y)) v(y) \: \text{d}y \right ) \qquad \forall x \in D
\end{equation}
a single layer of the neural operator where \(v: D \to \R^n \) is the input function to the layer and we denote by \(u: D \to \R^n\) the output function. We use the notation \(\kappa_v\) to indicate that the kernel depends on the entirety of the function \(v\) as well as on its pointwise values $v(x)$ and $v(y).$ While this is not explicitly done in \eqref{eq:kernelop3}, it is a straightforward generalization. We now pick a specific form for kernel, in particular, we assume \(\kappa_v : \R^n \times \R^n \to \R^{n \times n}\) does not explicitly depend on the spatial variables \((x,y)\) but only on the input pair \((v(x),v(y))\). Furthermore, we let 
\[\kappa_v(v(x),v(y)) = g_v(v(x),v(y)) R\]
where \(R \in \R^{n \times n}\) is a matrix of free parameters i.e. its entries are concatenated in \(\theta\) so they are learned, and \(g_v: \R^n \times \R^n \to \R\) is defined as
\[g_v(v(x),v(y)) = \left ( \int_D \text{exp}\left( \frac{\langle A v(s), Bv(y) \rangle}{\sqrt{m}} \right ) \: \text{d}s \right )^{-1} \text{exp} \left ( \frac{\langle A v(x), Bv(y) \rangle}{\sqrt{m}} \right ).\]
Here \(A,B \in \R^{m \times n}\) are again matrices of free parameters, \(m \in \N\) is a hyperparameter, and \(\langle \cdot, \cdot \rangle\) is the Euclidean inner-product on \(\R^m\). Putting this together, we find that \eqref{eq:liketransformer} becomes
\begin{equation}
    \label{eq:continoustransformer}
    u(x) = \sigma \left ( v(x) + \int_D \frac{\text{exp} \left ( \frac{\langle A v(x), Bv(y) \rangle}{\sqrt{m}} \right )}{\int_D \text{exp}\left( \frac{\langle A v(s), Bv(y) \rangle}{\sqrt{m}} \right ) \: \text{d}s} R v(y) \: \text{d}y \right ) \qquad \forall x \in D.
\end{equation}
Equation \eqref{eq:continoustransformer} can be thought of as the continuum limit of a transformer block. To see this, we will discretize to obtain the usual transformer block.

To that end, let \(\{x_1,\dots,x_k\} \subset D\) be a uniformly-sampled, \(k\)-point discretization of \(D\) and denote \(v_j = v(x_j) \in \R^n\) and
\(u_j = u(x_j) \in \R^n\) for \(j=1,\dots,k\). Approximating the inner-integral in \eqref{eq:continoustransformer} by Monte-Carlo, we have
\[\int_D \text{exp}\left( \frac{\langle A v(s), Bv(y) \rangle}{\sqrt{m}} \right ) \: \text{d}s \approx \frac{|D|}{k} \sum_{l=1}^k \text{exp} \left ( \frac{\langle A v_l, B v(y) \rangle}{\sqrt{m}} \right ).\]
Plugging this into \eqref{eq:continoustransformer} and using the same approximation for the outer integral yields
\begin{equation}
    \label{eq:discretetransformer}
    u_j = \sigma \left ( v_j + \sum_{q=1}^k \frac{\text{exp} \left ( \frac{\langle A v_j, Bv_q \rangle}{\sqrt{m}} \right )}{\sum_{l=1}^k \text{exp} \left ( \frac{\langle A v_l, B v_q \rangle}{\sqrt{m}} \right )} R v_q \right ), \qquad j=1,\dots,k.
\end{equation}
Equation \eqref{eq:discretetransformer} can be viewed as a Nystr{\"o}m approximation of \eqref{eq:continoustransformer}. Define the vectors \(z_q \in \R^k\) by
\[z_q = \frac{1}{\sqrt{m}} (\langle Av_1, Bv_q \rangle, \dots, \langle Av_k, Bv_q \rangle), \qquad q=1,\dots,k.\]
Define \(S : \R^k \to \Delta_k\), where \(\Delta_k\) denotes the \(k\)-dimensional probability simplex, as the softmax function
\[S(w) = \left ( \frac{\text{exp}(w_1)}{\sum_{j=1}^k \text{exp}(w_j)}, \dots, \frac{\text{exp}(w_k)}{\sum_{j=1}^k \text{exp}(w_j)} \right ), \qquad \forall w \in \R^k.\]
Then we may re-write \eqref{eq:discretetransformer} as
\[u_j = \sigma \left ( v_j + \sum_{q=1}^k S_j (z_q) R v_q \right ), \qquad j=1,\dots,k.\]
Furthermore, if we re-parametrize \(R = R^{\text{out}} R^{\text{val}}\) where \(R^{\text{out}} \in \R^{n \times m}\) and \(R^{\text{val}} \in \R^{m \times n}\) are matrices of free parameters, we obtain
\[u_j = \sigma \left ( v_j + R^{\text{out}}\sum_{q=1}^k S_j (z_q) R^{\text{val}} v_q \right ), \qquad j=1,\dots,k\]
which is precisely the single-headed attention, transformer block with no layer normalization applied inside the activation function. In the language of transformers, the matrices \(A\), \(B\), and \(R^{\text{val}}\) correspond to the \textit{queries}, \textit{keys}, and \textit{values} functions respectively. We note that tricks such as layer normalization \citep{ba2016layer} can be  adapted in a
straightforward manner to the continuum setting and incorporated into \eqref{eq:continoustransformer}. Furthermore multi-headed self-attention can be realized by simply allowing \(\kappa_v\) to be a sum over multiple functions with form \(g_v R\) all of which have separate trainable parameters. Including such generalizations yields the continuum limit of the transformer as implemented in practice. We do not pursue this here as our goal is simply to draw a parallel between the two methods.
\end{proof}

Even though transformers are special cases of neural operators, the standard attention mechanism is memory and computation intensive, as seen in Section~\ref{sec:problems}, compared to neural operator architectures developed here \eqref{eq:kernelop1}-\eqref{eq:kernelop3}. The high computational complexity of transformers is evident is \eqref{eq:continoustransformer} since we must evaluate a \textit{nested} integral of \(v\) for each \(x \in D\). Recently, efficient attention mechanisms have been explored, e.g. long-short~\cite{zhu2021long} and adaptive FNO-based attention mechanisms~\citep{guibas2021adaptive}. However, many of the efficient vision transformer architectures~\citep{choromanski2020rethinking,dosovitskiy2020image} like ViTs are {\em not} special cases of neural operators since they use CNN layers to generate tokens, which are not discretization invariant. 


}

\section{Test Problems}
\label{sec:problems}


A central application of neural operators is learning solution operators defined by parametric partial differential equations. In this section, we define four test problems for which we numerically study the approximation properties of neural operators. To that end, let \((\A,\U,\mathcal{F})\) be a triplet of Banach spaces. The first two problem classes considered are derived from the following general class of PDEs: 
\begin{align}
\label{eq:generalpde}
\begin{split}
\LL_a u &= f
\end{split}
\end{align}
where, for every \(a \in \A\), \(\LL_a : \U \to \mathcal{F}\) is a, possibly nonlinear, partial differential operator, and \(u \in \U\) corresponds to the solution of the PDE \eqref{eq:generalpde} when \(f \in \mathcal{F}\) and appropriate boundary conditions are imposed.
The second class will be evolution equations with initial
condition \(a \in \A\) and solution \(u(t) \in \U\) at every time $t>0.$
We seek to learn the map from $a$ to $u:=u(\tau)$
for some fixed time $\tau>0;$ we will also study maps on 
paths (time-dependent solutions).

Our goal will be to learn the mappings
\[\Ftrue: a \mapsto u \qquad \text{or} \qquad \Ftrue : f \mapsto u;\]
we will study both cases, depending on the test problem considered. We will define a probability measure \(\mu\) on \(\A\) or \(\mathcal{F}\) which will serve to define a model for likely input data. Furthermore, measure \(\mu\) will define a topology on the space of mappings in which \(\Ftrue\) lives, using the Bochner norm \eqref{eq:bochner_error}. 
We will assume that each of the spaces \((\A,\U,\mathcal{F})\) are Banach spaces of functions defined on a bounded domain \(D \subset \R^d\). All reported errors will be Monte-Carlo estimates of the relative error
\[\mathbb{E}_{a \sim \mu} \frac{\|\Ftrue(a) - \G_\theta(a)\|_{L^2(D)}}{\|\Ftrue (a)\|_{L^2(D)}}\]
or equivalently replacing \(a\) with \(f\) in the above display and with the assumption that \(\U \subseteq L^2(D)\). The domain $D$ will be discretized,  usually uniformly, with $J \in \N$ points. 

\subsection{Poisson Equation}
\label{ssec:poisson}
First we consider the one-dimensional Poisson equation with a zero boundary condition.
In particular, \eqref{eq:generalpde} takes the form
\done{I prefer to make operator positive; please make any changes
needed from the sign-change}
\begin{align}
\label{eq:poisson}
\begin{split}
    -\frac{d^2}{dx^2} u(x) &= f(x),  \qquad x \in (0,1)\\
    u(0) = u(1) &= 0
\end{split}
\end{align}
for some source function \(f: (0,1) \to \R)\). In particular, for \(D(\LL):=H^1_0((0,1);\R) \cap H^2((0,1);\R)\), we have \(\LL: D(\LL) \to L^2((0,1);\R)\) defined as $-d^2/dx^2$, noting that 
that $\LL$ has no dependence on any parameter \(a \in \A\) in this case. We will consider the weak form of \eqref{eq:poisson} with source function \(f \in H^{-1}((0,1);\R)\) and therefore the solution operator  \(\Ftrue : H^{-1}((0,1);\R) \to H^1_0((0,1);\R)\) defined as
\[\Ftrue: f \mapsto u.\]
We define the probability measure \(\mu = N(0,C)\) where
\[C = \Bigl ( \LL + I \Bigl)^{-2},\]
defined through the spectral theory of self-adjoint operators.  Since \(\mu\) charges a subset of \(L^2((0,1);\R)\), we will learn \(\Ftrue : L^2((0,1);\R) \to H^1_0((0,1);\R)\) in the topology induced by \eqref{eq:bochner_error}.

In this setting, \(\Ftrue\) has a closed-form solution given as
\[\Ftrue(f) = \int_0^1 G(\cdot,y) f(y) \: \text{d}y\]
where
\[G(x,y) = \frac{1}{2} \left ( x + y - |y-x| \right ) - xy, \qquad \forall (x,y) \in [0,1]^2\]
is the Green's function. Note that while \(\Ftrue\) is a linear operator, the Green's function \(G\) is non-linear
as a function of its arguments. We will consider only a single layer of \eqref{eq:F} with \(\sigma_1 = \text{Id}\), \(\cP = \text{Id}\), \(\cQ = \text{Id}\), \(W_0 = 0\), \(b_0 = 0\), and 
\[\mathcal{K}_0(f) = \int_0^1 \kappa_{\theta} (\cdot,y) f(y) \: \text{d}y\]
where \(\kappa_\theta : \R^2 \to \R\) will be parameterized as a standard neural network with parameters \(\theta\).

The purpose of the current example is two-fold. First we will test the efficacy of the neural operator framework in a simple setting where an exact solution is analytically available. Second we will show that by building in the right inductive bias, in particular, paralleling the form of the Green's function solution, we obtain a model that generalizes outside the distribution \(\mu\). That is, once trained, the model will generalize to any \(f \in L^2((0,1);\R)\) that may be outside the support of \(\mu\). For example, as defined, the random variable \(f \sim \mu\) is a continuous function, however, if \(\kappa_\theta\) approximates the Green's function well then the model \(\Ftrue\) will approximate the solution to \eqref{eq:poisson} accurately even for discontinuous inputs.

To create the dataset used for training, solutions to \eqref{eq:poisson} are obtained by numerical integration using the Green's function on a uniform grid with \(85\) collocation points. We use \(N = 1000\) training examples. 

\subsection{Darcy Flow}
\label{ssec:darcy}
We consider the steady state of Darcy Flow in two dimensions which is the second order elliptic equation
\begin{align}
\label{eq:darcy}
\begin{split}
- \nabla \cdot (a(x) \nabla u(x)) &= f(x), \qquad x \in D \\
u(x) &= 0, \qquad \quad \:\: x \in \partial D
\end{split}
\end{align}
where \(D = (0,1)^2\) is the unit square. In this setting \(\A = L^\infty(D;\R_+)\), \(\U = H_0^1(D;\R)\), and \(\mathcal{F} = H^{-1}(D;\R)\).
We fix \(f \equiv 1\) and consider the weak form of \eqref{eq:darcy} and therefore the solution operator \(\Ftrue : L^\infty(D;\R^+) \to H^1_0(D;\R)\) defined as
\begin{equation}
    \label{eq:darcysolutionop}
    \Ftrue:  a \mapsto u.
\end{equation}
Note that while \eqref{eq:darcy} is a linear PDE, the solution operator \(\Ftrue\) is nonlinear. We define the probability measure \(\mu = T_\sharp N(0,C)\) \nk{as the pushforward of a Gaussian measure under the operator \(T\)} where \nk{the covariance of the Gaussian is}
\[C = (-\Delta + 9I)^{-2}\]
with $D(-\Delta)$ defined to impose zero Neumann boundary on the Laplacian. We 
\nk{define} \(T\) to be a Nemytskii operator acting on functions, defined through the map \(T : \R \to \R_+\) defined as
\[T(x) = \begin{cases}
12, & x \geq 0 \\
3, & x < 0
\end{cases}.\]
The random variable \(a \sim \mu\) is a piecewise-constant function with random interfaces given by the underlying Gaussian random field. Such constructions are prototypical models for many physical systems such as permeability in sub-surface flows and (in a vector
generalization) material microstructures in elasticity. 

To create the dataset used for training, solutions to \eqref{eq:darcy} are obtained using a second-order finite difference scheme on a uniform grid of size \(421 \times 421\). All other resolutions are downsampled from this data set. We use \(N = 1000\) training examples.

\subsection{Burgers' Equation}
\label{ssec:burgers}
We consider the one-dimensional viscous Burgers' equation
\begin{align}
    \label{eq:burgers}
    \begin{split}
    \frac{\partial}{\partial t} u(x,t) + \frac{1}{2} \frac{\partial}{\partial x} \bigl ( u(x,t) \bigl)^2 &= \nu \frac{\partial^2}{\partial x^2} u(x,t), \qquad x \in (0,2\pi), t \in (0,\infty) \\
    u(x,0) &= u_0(x), \qquad \qquad \:\:\: x \in (0,2\pi)
    \end{split}
\end{align}
with periodic boundary conditions and a fixed viscosity \(\nu = 10^{-1}\). Let \(\Psi : L^2_{\text{per}}((0,2\pi);\R) \times \R_+ \to H^s_{\text{per}}((0,2\pi);\R)\), for any \(s > 0\), be the flow map associated to \eqref{eq:burgers}, in particular,
\[\Psi(u_0, t) = u(\cdot,t), \qquad t > 0.\]
We consider the solution operator defined by evaluating \(\Psi\) at a fixed time. Fix any $s \ge 0.$ Then we may define \(\Ftrue : L^2_{\text{per}}((0,2\pi);\R) \to H^s_{\text{per}}((0,2\pi);\R)\) by
\begin{equation}
    \label{eq:burgerssolutionop}
    \Ftrue : u_0 \mapsto \Psi(u_0,1).
\end{equation}
We define the probability measure \(\mu = N(0,C)\) where
\[C = 625 \Bigl ( -\frac{d^2}{dx^2} + 25I \Bigl)^{-2}\]
with domain of the Laplacian defined to impose periodic boundary conditions. We chose the initial condition for \eqref{eq:burgers} by drawing \(u_0 \sim \mu\), noting that \(\mu\) charges a subset of \(L^2_{\text{per}}((0,2\pi);\R)\).

To create the dataset used for training, solutions to \eqref{eq:burgers} are obtained using a pseudo-spectral split step method where the heat equation part is solved exactly in Fourier space and then the non-linear part is advanced using a forward Euler method with a very small time step. We use a uniform spatial grid with \(2^{13} = 8192\) collocation points and subsample all other resolutions from this data set. We use \(N = 1000\) training examples.

\subsection{Navier-Stokes Equation}
\label{ssec:nse}

\done{I am worried about this formulation of the equation. I'd rather
you formulate the equation in velocity space, define the vorticity,
and simply mention that there is an equation for the vorticity
which can be solved directly. Look at how Chandler-Kerswell write it.
(But use $\Delta$ not $\nabla^2$!
The point of using streamfunction is that it automatically
imposes divergence free; and then vorticity is related to it
by Laplacian so same thing applies. Is the forcing as stated
the forcing in vorticity space? More natural to define forcing
in velocity space.}

We consider the two-dimensional Navier-Stokes equation for a viscous, incompressible fluid 
\begin{align}
\label{eq:navierstokes-velocity}
\begin{split}
\partial_t u(x,t) + u(x,t) \cdot \nabla u(x,t) + \nabla p(x,t) &= \nu \Delta u(x,t) + f(x), \qquad x \in \mathbb{T}^2, t \in (0,\infty)  \\
\nabla \cdot u(x,t) &= 0, \qquad \qquad \qquad \qquad \:\:\: x \in \mathbb{T}^2, t \in [0,\infty)  \\
u(x,0) &= u_0(x), \qquad \qquad \qquad \:\:\: x \in \mathbb{T}^2 
\end{split}
\end{align}
where \(\mathbb{T}^2\) is the unit torus i.e. \([0,1]^2\) equipped with periodic boundary conditions, and \(\nu \in \R_+\) is a fixed viscosity.
Here $u : \mathbb{T}^2 \times \R_+ \to \R^2$ is the velocity field,  
$p : \mathbb{T}^2 \times \R_+ \to \R^2$ is the pressure field, and
 \(f: \mathbb{T}^2 \to \R\) is a fixed forcing function.

Equivalently, we study the vorticity-streamfunction formulation of the equation
\begin{subequations}
\label{eq:navierstokes}
\begin{align}
\partial_t w(x,t) + \nabla^{\perp}\psi \cdot \nabla w(x,t) &= \nu \Delta w(x,t) + g(x), \qquad x \in \mathbb{T}^2, t \in (0,\infty),  \\
-\Delta \psi&=\omega, \qquad\qquad\qquad\qquad\:\:\: x \in \mathbb{T}^2, t \in (0,\infty),  \\
 w(x,0) &= w_0(x), \qquad \qquad \qquad \:\:\: x \in \mathbb{T}^2, 
\end{align}
\end{subequations}
where \(w\) is the out-of-plane component of the vorticity field \(\nabla \times u: \mathbb{T}^2 \times \R_+ \to \R^3\). Since, when viewed in three dimensions, 
$u=\bigl(u_1(x_1,x_2),u_2(x_1,x_2),0\bigr)$, it follows that
$\nabla \times u=(0,0,\omega).$ 
The stream function $\psi$ is related to the velocity by 
$u=\nabla^{\perp}\psi$, enforcing the divergence-free condition.
Similar considerations as for the curl of $u$ apply to the curl of $f$, 
showing that $\nabla \times f=(0,0,g).$
\nk{Note that in (\ref{eq:navierstokes}b) $\psi$ is undetermined up to an (irrelevant, since $u$ is computed from $\psi$ by taking a gradient) additive constant; uniqueness is restored
by seeking $\psi$ with spatial mean $0.$}
 We define the forcing term as
\[g(x_1,x_2) = 0.1(\sin(2\pi(x_1+x_2)) + \cos(2\pi(x_1 + x_2))), \qquad \forall (x_1,x_2) \in \mathbb{T}^2.\]
 The corresponding Reynolds number is estimated as 
 $Re = \frac{\sqrt{0.1}}{\nu (2 \pi)^{3/2}}$ \citep{chandler2013invariant}.
Let \(\Psi : L^2(\mathbb{T}^2;\R) \times \R_+ \to H^s(\mathbb{T}^2;\R)\), for any \(s > 0\), be the flow map associated to \eqref{eq:navierstokes}, in particular,
\[\Psi(w_0,t) = w(\cdot,t), \qquad t > 0.\]
Notice that this is well-defined for any \(w_0 \in L^2(\mathbb{T};\R)\). 

We will define two notions of the solution operator. In the first, we will proceed as in the previous examples, in particular, \(\Ftrue : L^2(\mathbb{T}^2;\R) \to H^s(\mathbb{T}^2;\R)\) is defined as
\begin{equation}
\label{eq:nssolutionop_add}
\Ftrue : w_0 \mapsto \Psi(w_0, T)
\end{equation}
for some fixed \(T > 0\). In the second, we will map an initial part of the trajectory to a later part of the trajectory. In particular, we define
\(\Ftrue : L^2(\mathbb{T}^2;\R) \times C \bigl ( (0,10]; H^s(\mathbb{T}^2;\R) \bigl) \to C \bigl( (10,T]; H^s(\mathbb{T}^2;\R) \bigl)\) by
\begin{equation}
    \label{eq:nssolutionop2}
    \Ftrue : \bigl ( w_0, \Psi(w_0,t)|_{t \in (0,10]} \bigl) \mapsto \Psi(w_0,t)|_{t \in (10,T]}
\end{equation}
for some fixed  \(T > 10\). We define the probability measure \(\mu = N(0,C)\)
where 
\[C = 7^{3/2}  (-\Delta + 49I)^{-2.5} \]
with periodic boundary conditions on the Laplacian. We model the initial vorticity \(w_0 \sim \mu\) to \eqref{eq:navierstokes} as \(\mu\) charges a subset of \(L^2(\mathbb{T}^2;\R)\).
Its pushforward onto $\Psi(w_0,t)|_{t \in (0,10]}$ is required to define the measure on
input space in the second case defined by \eqref{eq:nssolutionop2}.

To create the dataset used for training, solutions to \eqref{eq:navierstokes} are obtained using a pseudo-spectral split step method where the viscous terms are advanced using a Crank–Nicolson update and the nonlinear and forcing terms are advanced using Heun's method. Dealiasing is used with the \(2/3\) rule. For further details on this approach see \citep{chandler2013invariant}. Data is obtained on a uniform \(256 \times 256\) grid and all other resolutions are subsampled from this data set. We experiment with different viscosities \(\nu\), final times \(T\), and amounts of training data \(N\).

\subsubsection{Bayesian Inverse Problem}
\label{sec:problem:bayesian}

As an application of operator learning, we consider the inverse problem of recovering the initial vorticity in the Navier-Stokes equation \eqref{eq:navierstokes} from partial, noisy observations of the vorticity at a later time. Consider the first solution operator defined in subsection~\ref{ssec:nse}, in particular, \(\Ftrue : L^2(\mathbb{T}^2;\R) \to H^s(\mathbb{T}^2;\R)\) defined as
\[\Ftrue : w_0 \mapsto \Psi(w_0, 50)\]
where \(\Psi\) is the flow map associated to \eqref{eq:navierstokes}. We then consider the inverse problem 
\begin{equation}
    \label{eq:inverseproblem}
    y = O \bigl ( \Ftrue(w_0) \bigl ) + \eta 
\end{equation}
of recovering \(w_0 \in L^2(\mathbb{T}^2;\R)\) where \(O : H^s(\mathbb{T}^2;\R) \to \mathbb{R}^{49}\) is the evaluation operator on a uniform \(7 \times 7\) interior grid, and \(\eta \sim N(0,\Gamma)\) is observational noise with covariance \(\Gamma = (1/\gamma^2)I\) and \(\gamma = 0.1\). We view \eqref{eq:inverseproblem} as the Bayesian inverse problem mapping prior measure \(\mu\) on $w_0$ to posterior measure \(\pi^y\) on $w_0/y.$
In particular, \(\pi^y\) has density with respect to \(\mu\), given by the Randon-Nikodym derivative
\[\frac{\text{d} \pi^y}{\text{d} \mu} (w_0) \propto \text{exp} \Bigl( -\frac12\|y - O \bigl ( \Ftrue(w_0) \bigl ) \|_{\Gamma}^2 \Bigr)\]
where \(\|\cdot\|_{\Gamma} = \|\Gamma^{-1/2} \cdot \|\) and \(\|\cdot\|\) is the Euclidean norm in \(\R^{49}\). For further details on Bayesian inversion for functions see \citep{cotter2009bayesian,stuart_2010}, and see \citep{Cotter_2013} for MCMC methods adapted
to the function-space setting.

We solve \eqref{eq:inverseproblem} by computing the posterior mean \(\mathbb{E}_{w_0 \sim \pi^y} [w_0]\) using the pre-conditioned Crank–Nicolson (pCN) MCMC method described in \cite{Cotter_2013} for this task. We employ pCN in two cases: (i) using \(\G^\dagger\) evaluated with the pseudo-spectral method described in section~\ref{ssec:nse}; and (ii)
using \(\G_\theta\), the neural operator approximating \(\G^\dagger\). After a 5,000 sample burn-in period, we generate 25,000 samples from the posterior using both approaches and use them to compute the posterior mean.

\subsubsection{Spectra}
\label{app:sepctral}
Because of the constant-in-time forcing term the energy reaches
a non-zero equilibrium in time which is statistically
reproducible for different initial conditions.
To compare the complexity  of the solution to the Navier-Stokes problem outlined in subsection~\ref{ssec:nse} we show, 
in Figure~\ref{fig:spectral1}, the Fourier spectrum of the solution data at time \(t=50\) for three different choices of the viscosity \(\nu\). The figure demonstrates that, for a wide range of wavenumbers $k$, which grows as $\nu$ decreases, the rate of decay of the spectrum is \(-5/3\), matching what is expected in the turbulent regime \citep{kraichnan67inertial}. 
This is a statistically stationary property of the equation, sustained for all positive times.

\begin{figure}[h]
    \centering
    \includegraphics[width=4.5cm]{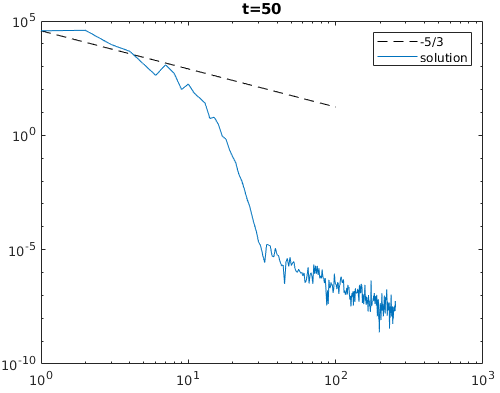}
    \includegraphics[width=4.5cm]{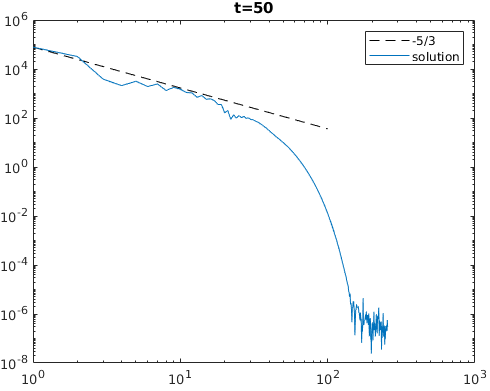}
    \includegraphics[width=4.5cm]{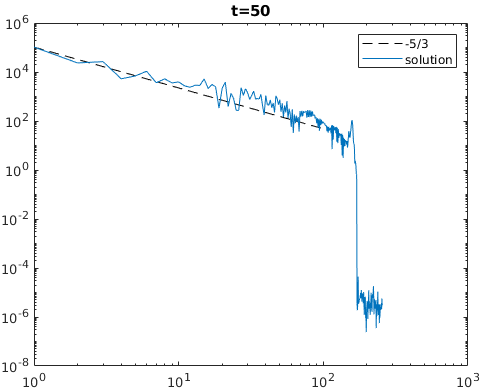}
    \caption{Spectral Decay of Navier-Stokes.}
    \label{fig:spectral1}
    \small{ The spectral decay of the Navier-stokes equation data. The y-axis is represents the value of each mode; the x-axis is the wavenumber $|k| = k_1 + k_2$. From left to right, the solutions have viscosity \(\nu = 10^{-3}, \: 10^{-4}, \: 10^{-5}\) respectively.}
\end{figure}

\subsection{Choice of Loss Criteria}
In general, the model has the best performance when trained and tested using the same loss criteria. If one trains the model using one norm and tests with another norm, the model may overfit in the training norm. Furthermore, the choice of loss function plays a key role.
In this work, we use the relative $L_2$ error to measure the performance in all our problems. Both the $L_2$ error and its square, the mean squared error (MSE), are common choices of the testing criteria in the numerical analysis and machine learning literature. We observed that using the relative error to train the model has a good normalization and regularization effect that prevents overfitting. In practice, training with the relative $L_2$ loss results in around half the testing error rate compared to training with the MSE loss.

%
%

\section{Numerical Results}
\label{sec:numerics}

In this section, we compare the proposed  neural operator with
other supervised learning approaches, using the four test problems outlined in Section~\ref{sec:problems}. 
In Subsection~\ref{ssec:result_green} we study the Poisson
equation, and learning a Greens function; Subsection~\ref{ssec:darcyburgers} considers the coefficient to
solution map for steady Darcy flow, and the initial condition
to solution at positive time map for Burgers equation.
In subsection~\ref{ssec:result_nse} we study the Navier-Stokes
equation.

We compare with a variety of
architectures found by discretizing the data and applying finite-dimensional approaches, as well as with other operator-based approximation methods; \nk{further detailed
comparison of other operator-based approximation methods may be found in \cite{de2022cost},
where the issue of error versus cost (with cost defined in various ways such as
evaluation time of the network, amount of data required) is studied.}
We do not compare against traditional solvers (FEM/FDM/Spectral),
although our methods, once trained, enable evaluation of the
input to output map orders of magnitude more quickly than by
use of such traditional solvers on complex problems.
We demonstrate the benefits of this speed-up in
a prototypical application, Bayesian inversion, in Subsubection \ref{sec:bayesian}.

All the computations are carried on a single Nvidia V100 GPU with 16GB memory.
The code is available at \url{https://github.com/zongyi-li/graph-pde} and \url{https://github.com/zongyi-li/fourier_neural_operator}.


\paragraph{Setup of the Four Methods:}
We construct the neural operator by stacking four integral operator layers as specified in \eqref{eq:F} with the ReLU activation. No batch normalization is needed. Unless otherwise specified, we use $N=1000$ training instances and $200$ testing instances. We use the
Adam optimizer to train for $500$ epochs with an initial learning rate of $0.001$ that is halved every $100$ epochs. We set the channel dimensions $d_{v_0} = \dots = d_{v_3} = 64$ for all one-dimensional problems and $d_{v_0} = \dots = d_{v_3} = 32$ for all two-dimensional problems. The kernel networks $\kappa^{(0)},\dots,\kappa^{(3)}$ are standard feed-forward neural networks with three layers and widths of $256$ units. We use the following abbreviations to denote the methods introduced in Section~\ref{sec:four_schemes}.
\begin{itemize}
    \item {\bf GNO:} The method introduced in subsection~\ref{sec:graphneuraloperator}, truncating the integral to a ball with radius \(r=0.25\) and using the Nystr\"om approximation with \(J' = 300\) sub-sampled nodes.
    \item {\bf LNO:} The low-rank method introduced in subsection~\ref{sec:lowrank} with rank $r=4$.
    \item {\bf MGNO:} The multipole method introduced in subsection~\ref{sec:multipole}. On  the Darcy flow problem, we use the random construction with three graph levels, each sampling $J_1 = 400, J_2 = 100, J_3 =25$ nodes nodes respectively. On the Burgers' equation problem, we use the orthogonal construction without sampling.
    \item {\bf FNO:} The Fourier method introduced in subsection~\ref{sec:fourier}. We set $k_{\text{max},j} = 16$ for all one-dimensional problems and $k_{\text{max},j} = 12$ for all two-dimensional problems.
\end{itemize}

\paragraph{Remark on the Resolution.}
Traditional PDE solvers such as FEM and FDM approximate a single function and therefore their error to the continuum decreases as the resolution is increased. The figures we show here exhibit
something different: the error is independent of resolution,
once enough resolution is used, but is not zero.
This reflects the fact that there
is a residual approximation error, in the infinite dimensional
limit, from the use of a finite-parametrized neural
operator, trained on a finite amount of data. Invariance of the error with respect to
(sufficiently fine) resolution is a desirable property that
demonstrates that an intrinsic approximation of the
operator has been learned, independent of any specific
discretization; see Figure \ref{fig:error}. Furthermore, resolution-invariant operators can do zero-shot super-resolution, as shown in Subsubection \ref{sec:superresolution}.

\subsection{Poisson Equation}
\label{ssec:result_green}

Recall the Poisson equation \eqref{eq:poisson} introduced in subsection~\ref{ssec:poisson}. We use a zero hidden layer neural operator construction without lifting the input dimension. In particular, we simply learn a kernel \(\kappa_\theta : \R^2 \to \R\) parameterized as a standard feed-forward neural network with parameters \(\theta\). Using only \(N = 1000\)
training examples, we obtain a relative test error of \(10^{-7}\). The neural operator gives an almost perfect approximation to the true solution operator in the topology of \eqref{eq:bochner_error}.

To examine the quality of the approximation in the much stronger uniform topology, we check whether the kernel \(\kappa_\theta\) approximates the Green's function for this problem. To see why this is enough, let \(K \subset L^2([0,1];\R)\) be a bounded set i.e.
\[\|f\|_{L^2([0,1];\R)} \leq M, \qquad \forall f \in K\]
and suppose that
\[\sup_{(x,y) \in [0,1]^2} |\kappa_\theta(x,y) - G(x,y)| < \frac{\epsilon}{M}.\]
for some \(\epsilon > 0\). Then it is easy to see that
\[\sup_{f \in K} \|\Ftrue (f) - \G_\theta (f)\|_{L^2([0,1];\R)} < \epsilon,\]
in particular, we obtain an approximation in the topology of uniform convergence over bounded sets, while having trained only in the topology of the Bochner norm \eqref{eq:bochner_error}. Figure~\ref{fig:kernel1d_nystrom} shows the results from which we can see that \(\kappa_\theta\) does indeed approximate the Green's function well. This result implies that by constructing a suitable architecture, we can generalize to the entire space and data that is well outside the support of the training set.



\begin{figure}[t]
    \centering
    \includegraphics[width=0.48\textwidth]{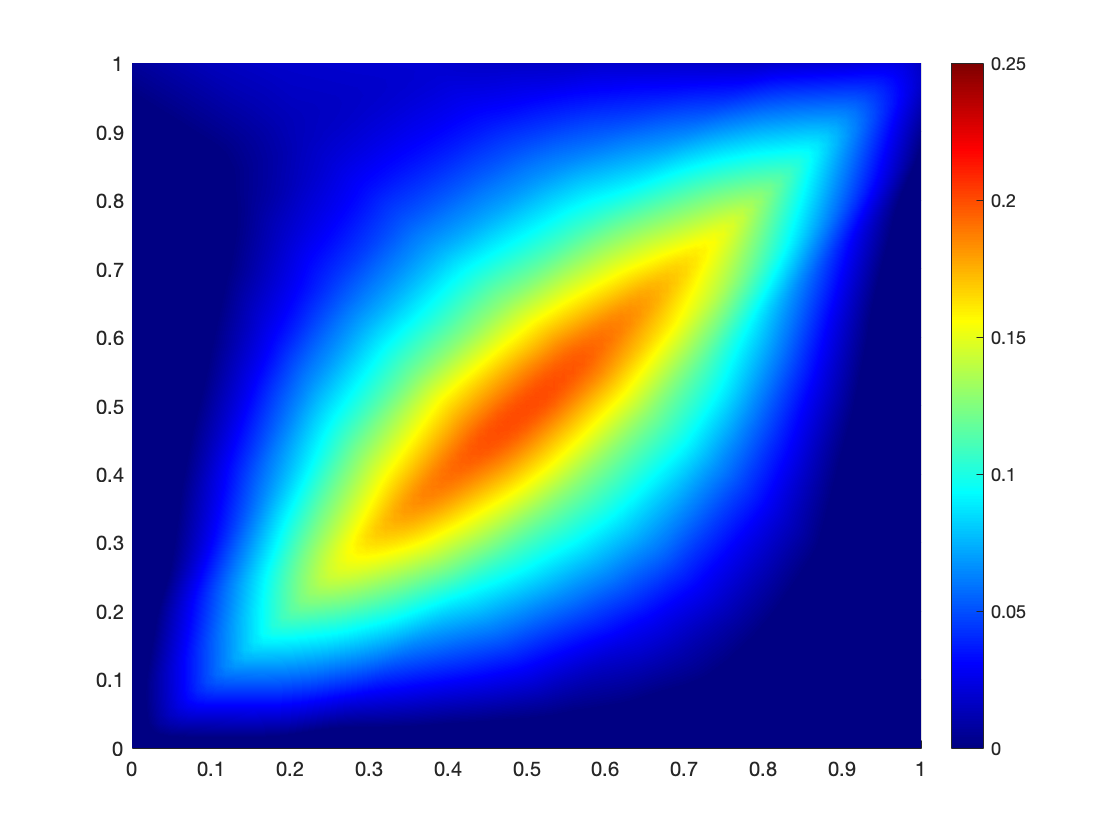}
    \includegraphics[width=0.48\textwidth]{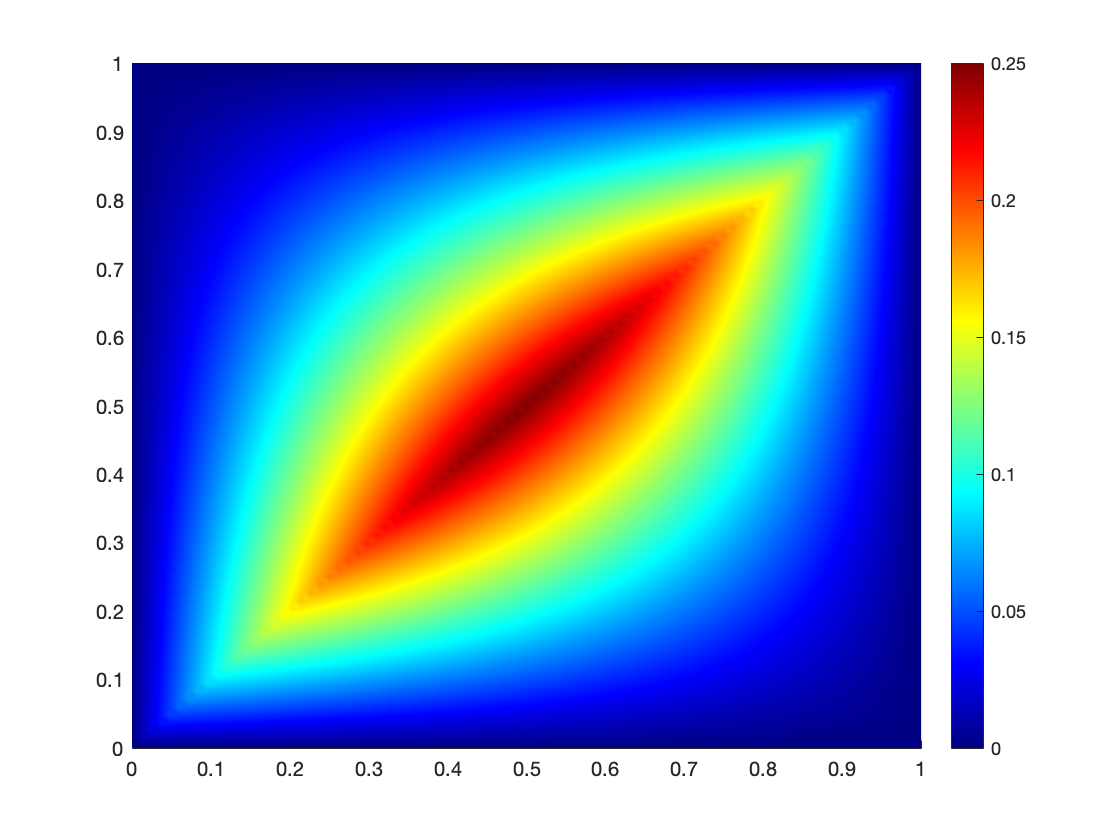}
        \caption{Kernel for one-dimensional Green's function, with the Nystrom approximation method}
    \label{fig:kernel1d_nystrom}
    \small{
    {\bf left:} learned kernel function; {\bf right:} the analytic Green's function.\\
    This is a proof of concept of the graph kernel network on $1$ dimensional Poisson equation and the comparison
    of learned and truth kernel.
    }
\end{figure}

\subsection{Darcy and Burgers Equations}
\label{ssec:darcyburgers}

\done{Is $s$ defined anywhere (see Tables); even if it is recall
its definition in the table captions?}

\postponed{These Darcy and Burgers results show FNO is the best but
in what sense are they "fair" comparisons? Are the other methods
using same number of parameters; same amount of training time;
same ammount of evaluation time? This discussion needs to be
sharper.}

In the following section, we compare four methods presented in this paper, with different operator approximation benchmarks; we study the Darcy flow problem introduced in Subsection~\ref{ssec:darcy} and the Burgers' equation problem introduced in Subsection~\ref{ssec:burgers}. The solution operators of interest are defined by \eqref{eq:darcysolutionop}
and \eqref{eq:burgerssolutionop}. We use the following abbreviations for the methods against which we benchmark.

\begin{itemize}
    \item {\bf NN} is a standard point-wise feedforward neural network. It is mesh-free, but performs badly due to lack of neighbor information. We use standard fully connected neural networks with 8 layers and width 1000.
    \item {\bf FCN} is the state of the art neural network method based on Fully Convolution Network \citep{Zabaras}. It has a dominating performance for small grids $s=61$. But fully convolution networks are mesh-dependent and therefore their error grows when moving to a larger grid.
    \item {\bf PCA+NN} is an instantiation of the methodology proposed in \cite{Kovachki}: using PCA as an autoencoder on both the input and output spaces and interpolating the latent spaces with a standard fully connected neural network with width 200. The method provably obtains mesh-independent error and can learn purely from data, however the solution can only be evaluated on the same mesh as the training data.  
    \item {\bf RBM} is the classical Reduced Basis Method (using a PCA basis), which is widely used in applications and provably obtains mesh-independent error \citep{DeVoreReducedBasis}. This method has good performance, but the solutions can only be evaluated on the same mesh as the training data and one needs knowledge of the PDE to employ it.
    \item {\bf DeepONet} is the Deep Operator network \citep{lu2019deeponet} that comes equipped with an approximation theory \citep{lanthaler2021error}. We use the unstacked version with width 200 which is precisely defined in the original work \citep{lu2019deeponet}. We use standard fully connected neural networks with 8 layers and width 200. 
\end{itemize}

\begin{figure}
    \centering
    \includegraphics[width=\textwidth]{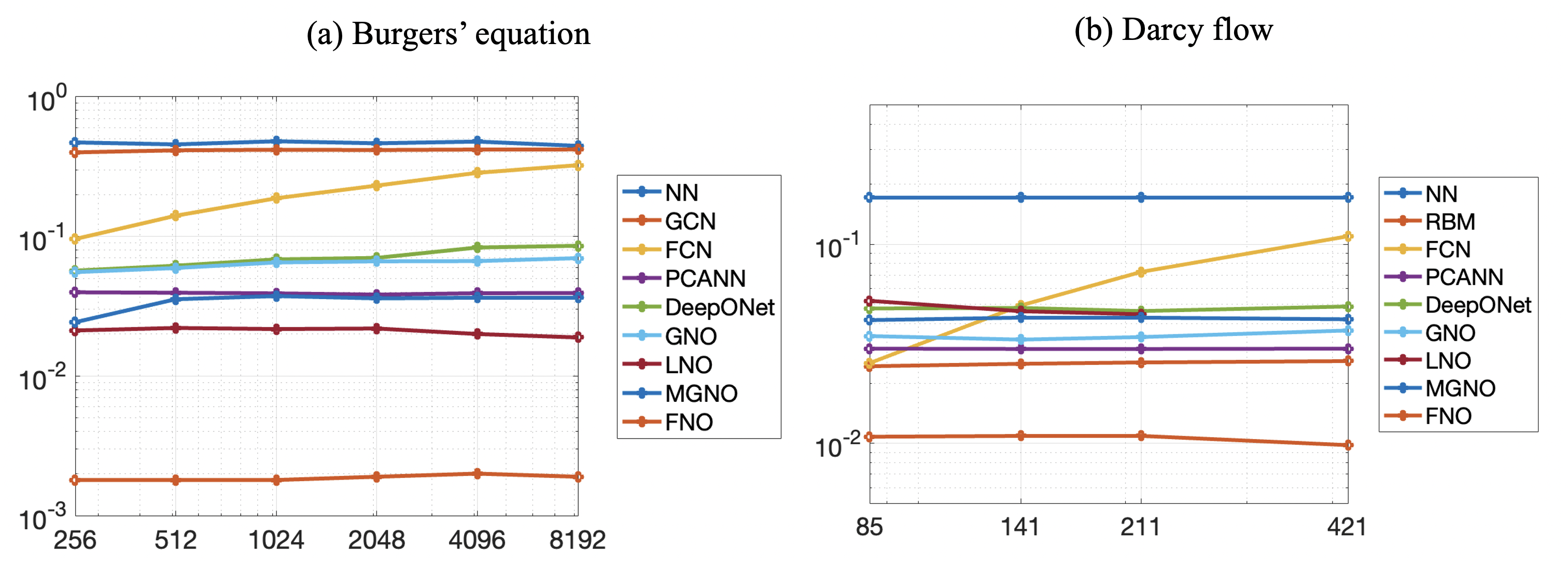}
    \small{ \textbf{(a)} benchmarks on Burgers equation;  \textbf{(b)} benchmarks on Darcy Flow for different resolutions; Train and test on the same resolution.
    For acronyms, see Section \ref{sec:numerics}; details in Tables \ref{table:burgers}, \ref{table:darcy}.} 
    \caption{Benchmark on Burger's equation and Darcy Flow \done{It is confusing to have the two left hand panels with the right-most panel; the left most panels who resolution
    invariance very clear;y, the NSE panel is somewhat different. I suggest separate figures.}}
    \label{fig:error}
\end{figure}
\subsubsection{Darcy Flow}
\label{ssec:DF}

The results of the experiments on Darcy flow are shown in Figure \ref{fig:error} and Table \ref{table:darcy}. All the methods, except for FCN, achieve invariance of
the error with respect to the resolution $s$. In the experiment, we tune each model across of range of different widths and depth to obtain
the choices used here; for DeepONet for example this leads to
8 layers and width 200 as reported above.

Within our hyperparameter search, the Fourier neural operator (FNO) obtains the lowest relative error. 
The Fourier based method likely sees this advantage because the output functions are smooth in these test problems.
We also note that is also possible to obtain better results on each model using modified architectures and problem specific feature engineering.
For example for DeepONet, using CNN on the branch net and PCA on the trunk net (the latter being similar to the method used in
\cite{Kovachki}) can achieve $0.0232$ relative $L_2$ error, as shown in \cite{lu2021comprehensive}, about half the size of the error we obtain here, but for a very coarse grid with $s=29.$
In the experiments the different approximation architectures are
such their training cost are similar across all the methods considered, for given $s.$ Noting this, and for example
comparing the graph-based neural operator methods such as GNO and MGNO that use Nystr\"om sampling in physical space with FNO, we
see that FNO is more accurate.
 
\begin{table}[ht]

\begin{center}
\begin{tabular}{l|llll}
\multicolumn{1}{c}{\bf Networks} 
&\multicolumn{1}{c}{\bf $s=85$}
&\multicolumn{1}{c}{\bf $s=141$} 
&\multicolumn{1}{c}{\bf $s=211$}
&\multicolumn{1}{c}{\bf $s=421$}\\
\hline
NN       &$0.1716$  &$0.1716$  &$0.1716$ &$0.1716$\\
FCN       &$0.0253$  &$0.0493$  &$0.0727$ & $0.1097$\\
PCANN      &$0.0299$  &$0.0298$  &$0.0298$ & $0.0299$\\
RBM    &$0.0244$ &$0.0251$ &$0.0255$ &$0.0259$ \\
DeepONet    &$0.0476$ &$0.0479$ &$0.0462$ &$0.0487$ \\
\hline 
GNO     &$0.0346$   &$0.0332$  &$0.0342$ &$0.0369$\\
LNO     &$0.0520$  &$0.0461$  &$0.0445$ &$-$\\
MGNO     &$0.0416$   &$0.0428$  &$0.0428$ &$0.0420$\\
FNO     &$\textbf{0.0108}$  &$\textbf{0.0109}$  &$\textbf{0.0109}$ &$\textbf{0.0098}$\\
\hline 
\end{tabular}
\end{center}
\caption{Relative error on 2-d Darcy Flow for different resolutions $s$.}
\label{table:darcy}
\end{table}

\subsubsection{Burgers' Equation}
\label{ssec:BE}


The results of the experiments on Burgers' equation are shown in Figure \ref{fig:error} and Table \ref{table:burgers}. As for the Darcy problem, our instantiation of the Fourier neural operator obtains nearly one order of magnitude lower relative error compared to any benchmarks. The Fourier neural operator has standard deviation $0.0010$ and mean training error $0.0012$. If one replaces the ReLU activation by GeLU, the test error of the FNO is further reduced from $0.0018$ to \textbf{$0.0007$}.
We again observe the invariance of the error with respect to the resolution. 
It is possible to improve the performance on each model using modified architectures and problem specific feature engineering.
Similarly, the PCA-enhanced DeepONet with a proper scaling can achieve $0.0194$ relative $L_2$ error, as shown in \cite{lu2021comprehensive}, on a grid of resolution $s=128$.

\begin{table}[ht]

\begin{center}
\begin{tabular}{l|llllllll}
\multicolumn{1}{c}{\bf Networks}
&\multicolumn{1}{c}{\bf $s=256$}
&\multicolumn{1}{c}{\bf $s=512$}
&\multicolumn{1}{c}{\bf $s=1024$}
&\multicolumn{1}{c}{\bf $s=2048$} 
&\multicolumn{1}{c}{\bf $s=4096$}
&\multicolumn{1}{c}{\bf $s=8192$}\\
\hline 
NN       &$0.4714$ &$0.4561$
&$0.4803$ &$0.4645$ &$0.4779$ &$0.4452$ \\
GCN           &$0.3999$ &$0.4138$
&$0.4176$   &$0.4157$  &$0.4191$ &$0.4198$\\
FCN         &$0.0958$ &$0.1407$
&$0.1877$   &$0.2313$  &$0.2855$ &$0.3238$\\
PCANN       &$0.0398$ &$0.0395$
&$0.0391$   &$0.0383$  &$0.0392$ &$0.0393$\\
DeepONet       &$0.0569$ &$0.0617$
&$0.0685$   &$0.0702$  &$0.0833$ &$0.0857$\\
\hline 
GNO     &$0.0555$ &$0.0594$ &$0.0651$   &$0.0663$  &$0.0666$ &$0.0699$\\
LNO      &$0.0212$ &$0.0221$
   &$0.0217$  &$0.0219$ &$0.0200$ &$0.0189$\\
MGNO      &$0.0243$ &$0.0355$
   &$0.0374$  &$0.0360$ &$0.0364$ &$0.0364$\\
FNO     &$\textbf{0.0018}$ &$\textbf{0.0018}$
   &$\textbf{0.0018}$  &$\textbf{0.0019}$ &$\textbf{0.0020}$ &$\textbf{0.0019}$\\
\hline 
\end{tabular}
\end{center}
\caption{ Relative errors on 1-d Burgers' equation for different resolutions $s$.} 
\label{table:burgers}
\end{table}

\begin{figure}[ht]
    \centering
    \includegraphics[width=\textwidth]{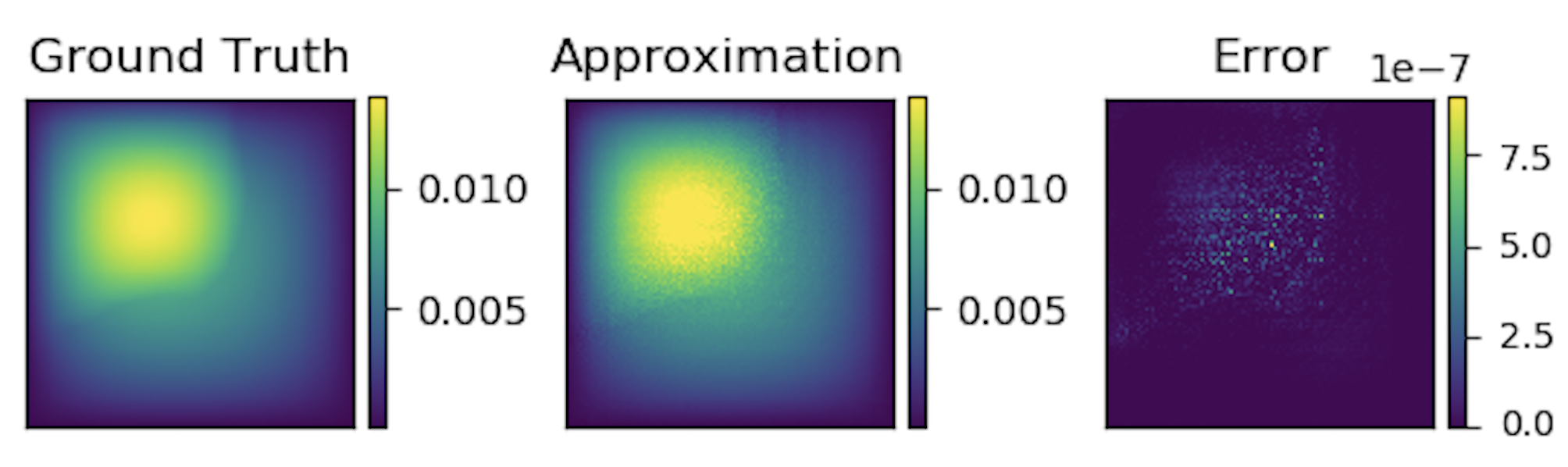}
        \caption{Darcy, trained on $16 \times 16$, tested on $241 \times 241$}\label{fig:super1}
    \small{
    Graph kernel network for the solution of \eqref{ssec:darcy}. It can be trained on a small resolution and will generalize to a large one. The Error is point-wise absolute squared error.}
\end{figure}

\subsubsection{Zero-shot super-resolution.}
\label{sec:superresolution1}
The neural operator is mesh-invariant, so it can be trained on a lower resolution and evaluated at a higher resolution, without seeing any higher resolution data (zero-shot super-resolution).
Figure  \ref{fig:super1} shows an example of the Darcy Equation where we train the GNO model on $16 \times 16$ resolution data in the setting above and transfer to $256 \times 256$ resolution, demonstrating super-resolution in space.

\subsection{Navier-Stokes Equation}
\label{ssec:result_nse}

\postponed{I think we should include brief results concerning \eqref{eq:nssolutionop_add} (time step to time step) as well and before the more
complicated \eqref{eq:nssolutionop2} (period to period).}

In the following section, we compare our four methods with different benchmarks on the Navier-Stokes equation introduced in subsection~\ref{ssec:nse}. The operator of interest is given by \eqref{eq:nssolutionop2}. We use the following abbreviations for the methods against which we benchmark.

\begin{itemize}
    \item {\bf ResNet:} $18$ layers of 2-d convolution with residual connections \cite{he2016deep}.
    \item {\bf U-Net:} A popular choice for image-to-image regression tasks consisting of four blocks with 2-d convolutions and deconvolutions \cite{ronneberger2015u}.
    \item {\bf TF-Net:} A network designed for learning turbulent flows based on a combination of spatial and temporal convolutions \cite{wang2020towards}.
    \done{I am not sure I see it this way; it is designed to approximate a solution operator and velocity versus
    vorticity is just about the space in which one studies the problem. I would move this comment somewhere else. I do understand why it is there and that it needs to be stated somewhere.}
    \item {\bf FNO-2d:} 2-d Fourier neural operator with an auto-regressive structure in time. We use the Fourier neural operator to model the local evolution from the previous $10$ time steps to the next one time step, and iteratively apply the model to get the long-term trajectory. We set and $k_{\text{max},j} = 12, d_v=32$.
    \item {\bf FNO-3d:} 3-d Fourier neural operator that directly convolves in space-time. We use the Fourier neural operator to model the global evolution from the initial $10$ time steps  directly to the long-term trajectory. We set  $k_{\text{max},j} = 12, d_v=20$.
\end{itemize}
   
\begin{figure}
    \centering
    \includegraphics[width=8cm]{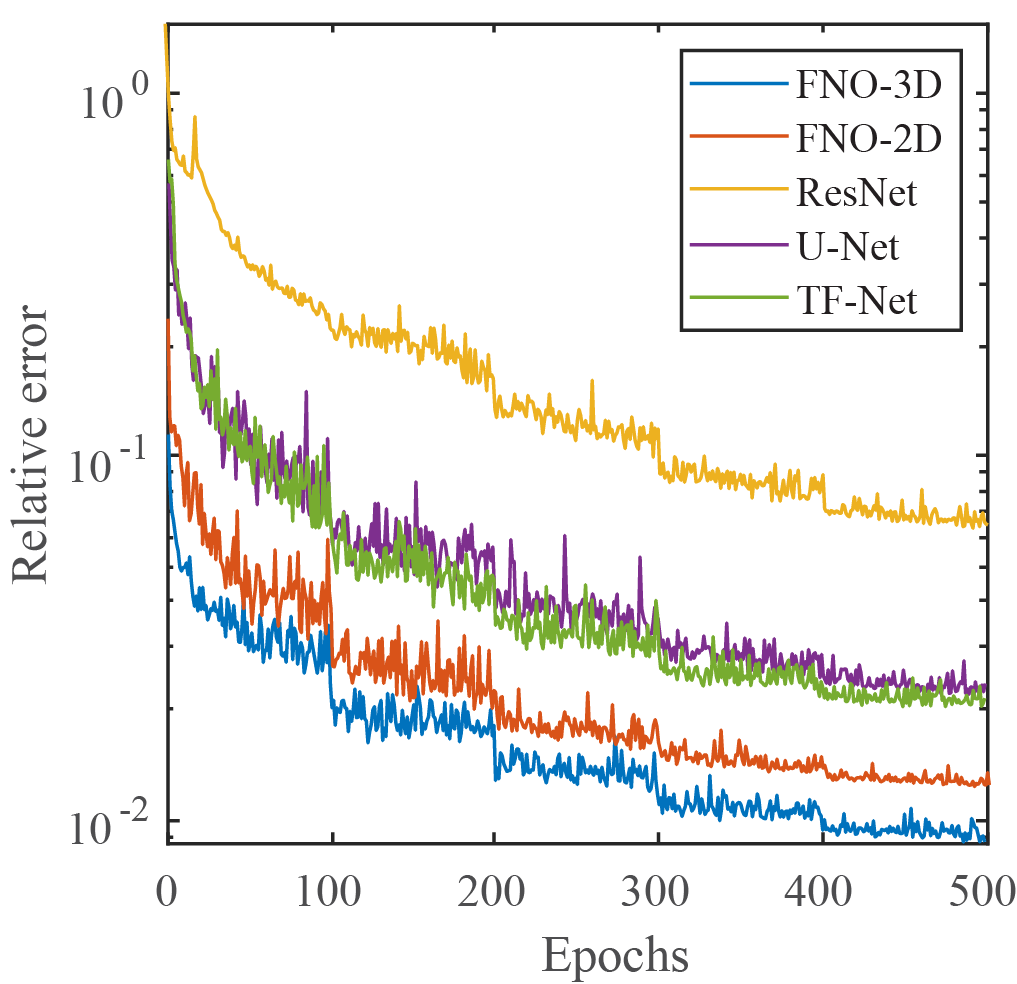}
        \caption{Benchmark on the Navier-Stokes}
    \label{fig:ns-error}
    \small{The learning curves on Navier-Stokes $\nu=1\mathrm{e}{-3}$ with different benchmarks. Train and test on the same resolution.
    For acronyms, see Section \ref{sec:numerics}; details in Tables \ref{table:ns}.} 
\end{figure}

\begin{table}[ht]
\begin{center}
\begin{tabular}{l|rc|cccc}
\multicolumn{1}{c}{} 
&\multicolumn{1}{c}{}
&\multicolumn{1}{c}{} 
&\multicolumn{1}{c}{} 
&\multicolumn{1}{c}{}\\
 & {\bf Parameters}& {\bf Time}& $\nu=10^{-3}$ &$\nu=10^{-4}$ &$\nu=10^{-4}$ & $\nu=10^{-5}$\\
 {\bf Config}&& {\bf per} &$T=50$ &$T=30$ &$T=30$ & $T=20$\\
 && {\bf epoch} &$ N=1000$ &$ N=1000$ &$N=10000$ & $ N=1000$\\
\hline 
FNO-3D    & $6,558,537$ & $38.99s$ &${\bf 0.0086}$ &$0.1918$ &${\bf 0.0820}$  &$0.1893$ \\
FNO-2D    & $414,517$ & $127.80s$ &$0.0128 $ &${\bf 0.1559}$ &$0.0834$  &${\bf 0.1556}$ \\
U-Net       & $24,950,491$ & $48.67s$ &$0.0245 $ &$0.2051$ &$0.1190$  &$0.1982$ \\
TF-Net       & $7,451,724$ & $47.21s$ &$0.0225 $ &$0.2253$ &$0.1168$  &$0.2268$ \\
ResNet     &$266,641$ & $78.47s$ &$0.0701 $ &$0.2871$ &$0.2311$  &$0.2753$ \\
\hline 
\end{tabular}
\end{center}
\caption{Benchmarks on Navier Stokes (fixing resolution $64 \times 64$ for both training and testing).}
\label{table:ns}
\end{table}

As shown in Table \ref{table:ns}, the FNO-3D has the best performance when there is sufficient data ($\nu=10^{-3}, N=1000$ and $\nu=10^{-4}, N=10000$). For the configurations where the amount of data is insufficient ($\nu=10^{-4}, N=1000$ and $\nu=10^{-5}, N=1000$), all methods have $>15\%$ error with FNO-2D achieving the lowest among our hyperparameter search. Note that we only present results for spatial resolution $64 \times 64$ since all benchmarks we compare against are designed for this resolution. Increasing the spatial resolution degrades their performance while FNO achieves the same errors.

\paragraph{Auto-regressive (2D) and Temporal Convolution (3D).}
We investigate two standard formulation to model the time evolution: the auto-regressive model (2D) and the temporal convolution model (3D).
\textbf{Auto-regressive models}: FNO-2D, U-Net, TF-Net, and ResNet all do 2D-convolution in the spatial domain and recurrently propagate in the time domain (2D+RNN). The operator maps the solution at previous time steps to the next time step (2D functions to 2D functions). 
\textbf{Temporal convolution models}: on the other hand, FNO-3D performs convolution in space-time -- it approximates the integral in time by a convolution. FNO-3D maps the initial time interval  directly to the full trajectory (3D functions to 3D functions). 
The 2D+RNN structure can propagate the solution to any arbitrary time $T$ in increments of a fixed interval length $\Delta t$, while the Conv3D structure is fixed to the interval $[0, T]$ but can transfer the solution to an arbitrary time-discretization. We find the 2D method work better for short time sequences while the 3D  method more expressive and easier to train on longer sequences.  


\begin{table}[ht]

\begin{center}
\begin{tabular}{l|lll}
\multicolumn{1}{c}{\bf Networks}
&\multicolumn{1}{c}{\bf $s=64$}
&\multicolumn{1}{c}{\bf $s=128$}
&\multicolumn{1}{c}{\bf $s=256$}\\
\hline 
FNO-3D    & $0.0098$ & $0.0101$ &$0.0106$  \\
FNO-2D    & $0.0129$ & $0.0128$ &$0.0126$ \\
U-Net       & $0.0253$ & $0.0289$ &$0.0344$ \\
TF-Net       & $0.0277$ & $0.0278$ &$0.0301$ \\
\hline 
\end{tabular}
\end{center}

\caption{ Resolution study on Navier-stokes equation ($\nu=10^{-3}$, $N=200$, $T=20$.)} 
\label{table:nse}
\end{table}

\subsubsection{Zero-shot super-resolution.}
\label{sec:superresolution}

The neural operator is mesh-invariant, so it can be trained on a lower resolution and evaluated at a higher resolution, without seeing any higher resolution data (zero-shot super-resolution).
Figure 
\ref{fig:super2} shows an example where we train the FNO-3D model on $64 \times 64 \times 20$ resolution data in the setting above with ($\nu=10^{-4}, N=10000$) and transfer to $256 \times 256 \times 80$ resolution, demonstrating super-resolution in space-time. The
Fourier neural operator is the only model among the benchmarks (FNO-2D, U-Net, TF-Net, and ResNet) that can do zero-shot super-resolution;
the method works well not only on the spatial but also on the temporal domain.

\begin{figure}[ht]
    \centering
    \includegraphics[width=\textwidth]{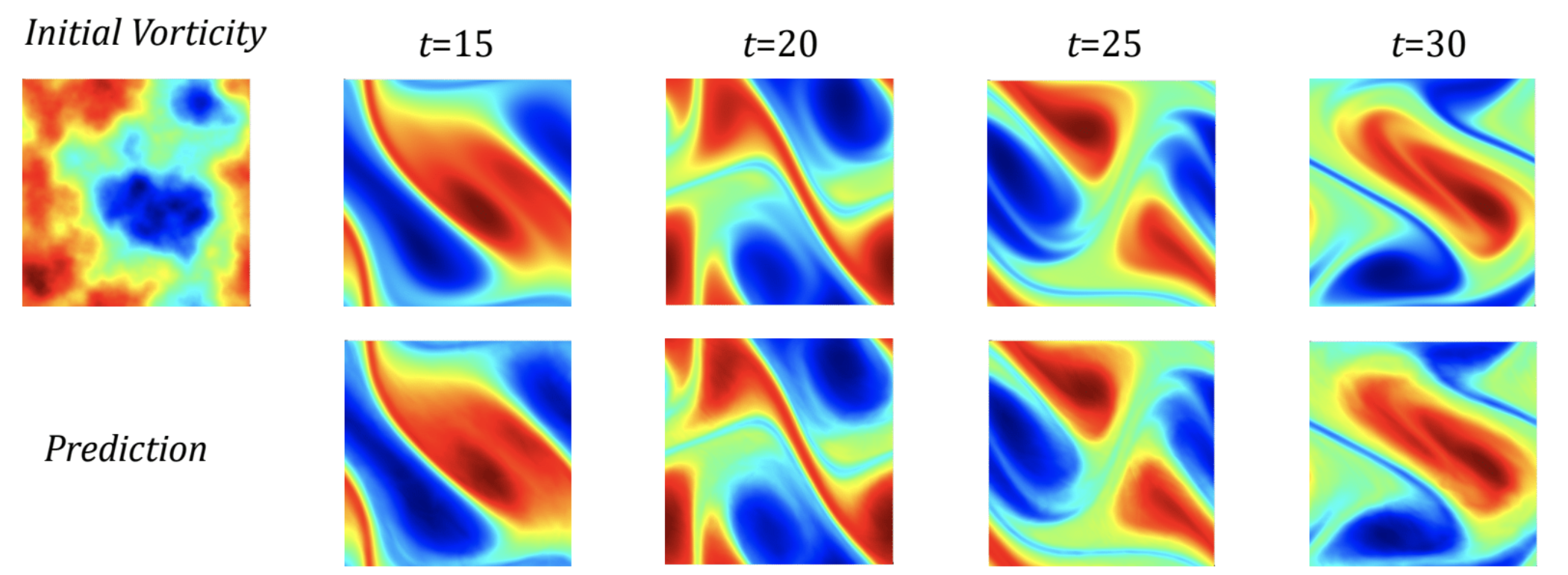}
    \caption{ Zero-shot super-resolution}
    \label{fig:super2} 
    \small{ Vorticity field of the solution to the two-dimensional Navier-Stokes equation with viscosity $\nu =10^{4}$
    (Re$\approx200$); Ground truth on top and prediction on bottom. The model is trained on data that is discretized on a uniform $64\times64$ spatial grid and on a $20$-point uniform temporal grid. The model is evaluated with a different initial condition that is discretized on a uniform $256\times256$ spatial grid and a $80$-point uniform temporal grid. 
    }
\end{figure}

\subsubsection{Spectral analysis}
\begin{figure}[t]
    \centering
    \includegraphics[width=0.48\textwidth]{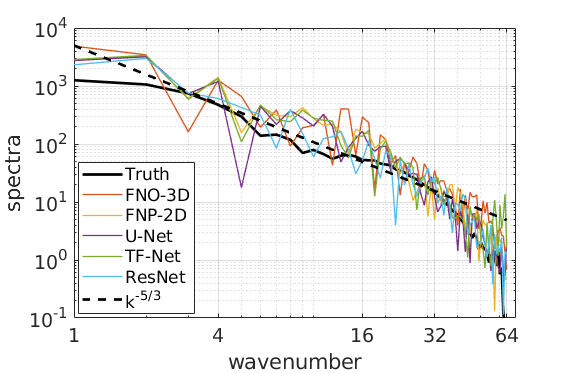}
    \includegraphics[width=0.48\textwidth]{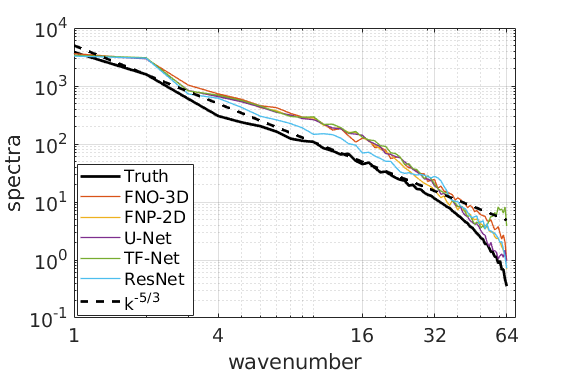}
        \caption{The spectral decay of the predictions of different methods}\label{fig:spectra}
    \small{
    The spectral decay of the predictions of different models on the Navier-Stokes equation. The y-axis is the spectrum; the x-axis is the wavenumber. Left is the spectrum of one trajectory; right is the average of $40$ trajectories. }
\end{figure}

\begin{figure}[t]
    \centering
    \includegraphics[width=7cm]{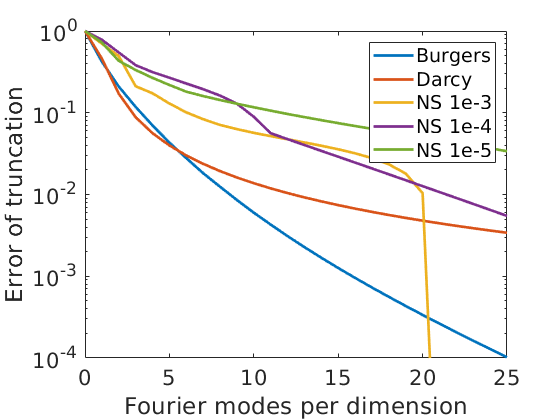}
        \caption{Spectral Decay in term of $k_{max}$ }
    \label{fig:spectral2}
    \small{ The error of truncation in one single Fourier layer without applying the linear transform $R$. The y-axis is the normalized truncation error; the x-axis is the truncation mode $k_{max}$.  }
\end{figure}

Figure \ref{fig:spectra} shows that all the methods are able to capture the spectral decay of the Navier-Stokes equation.
Notice that, while the Fourier method truncates the higher frequency modes during the convolution, FNO can still recover the higher frequency components in the final prediction. 
Due to the way we parameterize $R_\phi$, the function output by \eqref{eq:fourierlayer} has at most $k_{\text{max},j}$ Fourier modes per channel. This, however, does not mean that the Fourier neural operator can only approximate functions up to $k_{\text{max},j}$ modes. Indeed, the activation functions which occurs between integral operators and the final decoder network $Q$ recover the high frequency modes.
As an example, consider a solution to the Navier-Stokes equation with viscosity $\nu=10^{-3}$. Truncating this function at $20$ Fourier modes yields an error around $2\%$ as shown in Figure \ref{fig:spectral2}, while the Fourier neural operator learns the parametric dependence and produces approximations to an error of $\leq 1\%$ with only $k_{\text{max},j}=12$ parameterized modes.

\subsubsection{Non-periodic boundary condition.} Traditional Fourier methods work only with periodic boundary conditions. However, the Fourier neural operator does not have this limitation. This is due to the linear transform $W$ (the bias term) which keeps the track of non-periodic boundary. As an example, the Darcy Flow and the time domain of Navier-Stokes have non-periodic boundary conditions, and the Fourier neural operator still learns the solution operator with excellent accuracy.

\subsubsection{Bayesian Inverse Problem}
\label{sec:bayesian}
As discussed in Section \ref{sec:problem:bayesian},
we use the pCN method of \cite{Cotter_2013} to draw samples from the posterior distribution of initial vorticities in the Navier-Stokes equation given sparse, noisy observations at time $T=50$. We compare the Fourier neural operator acting as a surrogate model with the traditional solvers used to generate our train-test data (both run on GPU). We generate 25,000 samples from the posterior (with a 5,000 sample burn-in period), requiring 30,000 evaluations of the forward operator.

As shown in Figure \ref{fig:baysian}, FNO and the traditional solver recover almost the same posterior mean which, when pushed forward, recovers well the later-time solution of the Navier-Stokes equation.
In sharp contrast, FNO takes $0.005s$ to evaluate a single instance while the traditional solver, after being optimized to use the largest possible internal time-step which does not lead to blow-up, takes $2.2s$. This amounts to $2.5$ minutes for the MCMC using FNO and over $18$ hours for the traditional solver. Even if we account for data generation and training time (offline steps) which take $12$ hours, using FNO is still faster. Once trained, FNO can be used to quickly perform multiple MCMC runs for different initial conditions and observations, while the traditional solver will take $18$ hours for every instance. Furthermore, since FNO is differentiable, it can easily be applied to PDE-constrained optimization problems in which
adjoint calculations are used as part of the solution procedure.


\begin{figure}[ht]
    \centering
    \begin{subfigure}[b]{0.32\textwidth}
        \includegraphics[width=\textwidth]{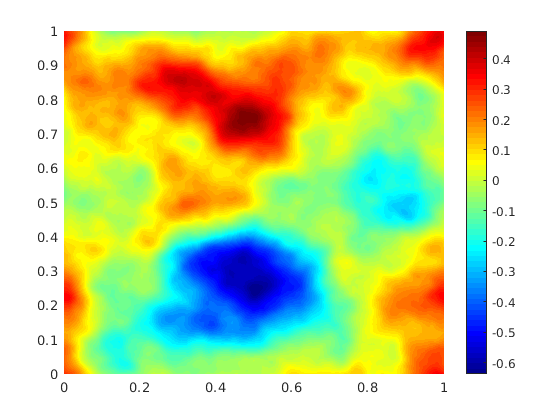}
    \end{subfigure}
    \begin{subfigure}[b]{0.32\textwidth}
        \includegraphics[width=\textwidth]{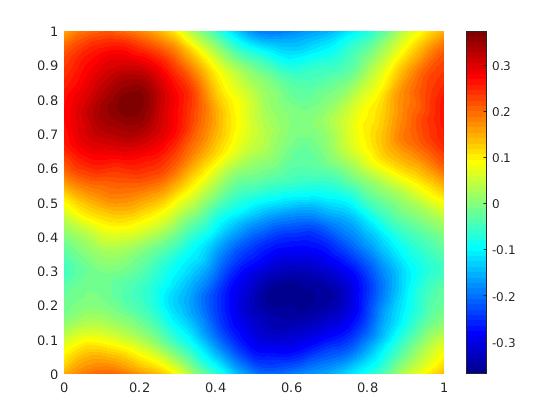}
    \end{subfigure}
    \begin{subfigure}[b]{0.32\textwidth}
        \includegraphics[width=\textwidth]{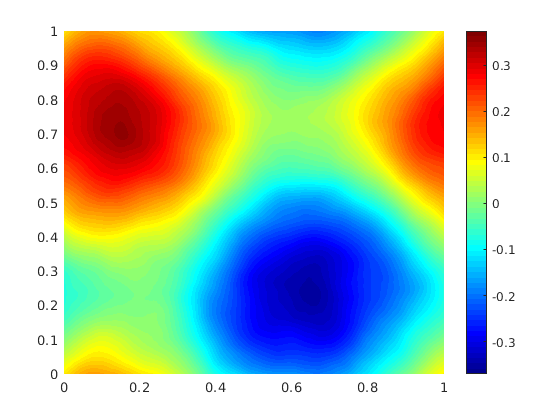}
    \end{subfigure}
    
    \begin{subfigure}[b]{0.32\textwidth}
        \includegraphics[width=\textwidth]{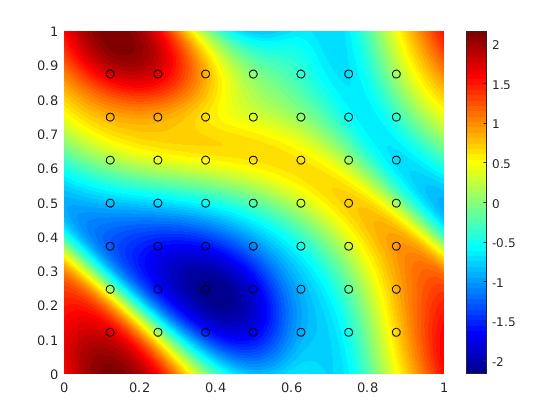}
    \end{subfigure}
    \begin{subfigure}[b]{0.32\textwidth}
        \includegraphics[width=\textwidth]{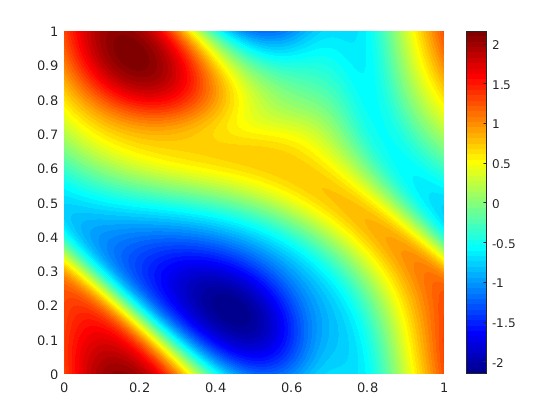}
    \end{subfigure}
    \begin{subfigure}[b]{0.32\textwidth}
        \includegraphics[width=\textwidth]{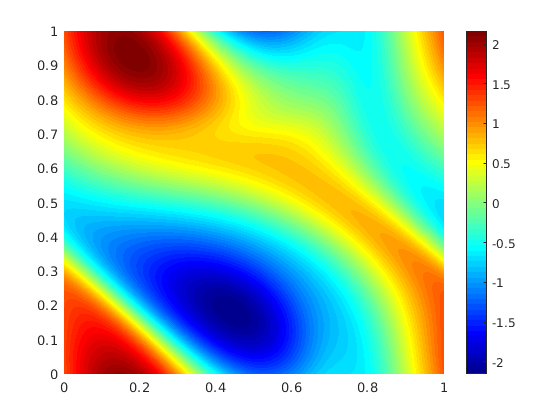}
    \end{subfigure}
        \caption{Results of the Bayesian inverse problem for the Navier-Stokes equation. } 
    \label{fig:baysian} 
\small{ The top left panel shows the true initial vorticity while bottom left panel shows the true observed vorticity at $T=50$ with black dots indicating the locations of the observation points placed on a $7 \times 7$ grid. The top middle panel shows the posterior mean of the initial vorticity given the noisy observations estimated with MCMC using the traditional solver, while the top right panel shows the same thing but using FNO as a surrogate model. The bottom middle and right panels show the vorticity at $T=50$ when the respective approximate posterior means are used as initial conditions. }
\end{figure}

\subsection{Discussion and Comparison of the Four methods}
\label{ssec:comparsion}

In this section we will compare the four methods in term of expressiveness, complexity, refinabilibity, and ingenuity.

\subsubsection{Ingenuity}
First we will discuss ingenuity, in other words, the design of the frameworks. The first method, GNO, relies on the Nystr\"om approximation of the kernel, or the Monte Carlo approximation of the integration. It is the most simple and straightforward method. The second method, LNO, relies on the low-rank decomposition of the kernel operator. It is efficient when the kernel has a near low-rank structure.  The third method, MGNO, is the combination of the first two. It has a hierarchical, multi-resolution decomposition of the kernel. The last one, FNO, is different from the first three;
it restricts the integral kernel to induce a convolution.

GNO and MGNO are implemented using graph neural networks, which helps to define sampling and integration. The graph network library also allows sparse and distributed message passing. The LNO and FNO don't have sampling. They are faster without using the graph library.


\begin{table}[h]

\begin{center}
\begin{tabular}{l|lll}
& scheme
& graph-based
& kernel network \\  
\hline 
GNO        &  Nystr\"om  approximation & Yes   & Yes  \\
LNO       &  Low-rank approximation & No  & Yes \\
MGNO       &  Multi-level graphs on GNO & Yes   & Yes \\
FNO       & Convolution theorem; Fourier features  & No  & No \\
\hline
\end{tabular}
\end{center}
\caption{Ingenuity.}
\label{table:ingenuity}
\end{table}

\subsubsection{Expressiveness}
We measure the expressiveness by the training and testing error of the method. The full $O(J^2)$ integration always has the best results, but it is usually too expensive. As shown in the experiments \ref{ssec:DF} and \ref{ssec:BE}, GNO usually has good accuracy, but its performance suffers from sampling.
LNO works the best on the 1d problem (Burgers equation). It has difficulty on the 2d problem because it doesn't employ sampling
to speed-up evaluation. MGNO has the multi-level structure, which gives it the benefit of the first two. Finally, FNO has overall the best performance. It is also the only method that can capture the challenging Navier-Stokes equation.


\subsubsection{Complexity}
The complexity of the four methods are listed in Table \ref{table:complexity}. GNO and MGNO have sampling. Their complexity depends on the number of nodes sampled $J'$. When using the full nodes. They are still quadratic. LNO has the lowest complexity $O(J)$. FNO, when using the fast Fourier transform, has complexity $O(J \log J)$.

In practice. FNO is faster then the other three methods because it doesn't have the kernel network $\kappa$. MGNO is relatively slower because of its multi-level graph structure. 

\begin{table}[h]

\begin{center}
\begin{tabular}{l|cc}
& Complexity
& Time per epochs in training\\
\hline 
GNO        &  $O(J'^2 r^2)$ &  $\ 4s$ \\
LNO       &  $O(J)$ & $\ 20s$ \\
MGNO       &  $\sum_l O(J_l^2 r_l^2)  \sim O(J)$ &  $\ 8s$\\
FNO       & $(J \log J)$  &  $\ 4s$\\
\hline
\end{tabular}
\end{center}
\caption{Complexity (roundup to second on a single Nvidia V100 GPU)}
\label{table:complexity}

\end{table}

\subsubsection{Refinability}
Refineability measures the number of parameters used in the framework. Table \ref{table:refinability} lists the accuracy of the relative error on Darcy Flow with respect to different number of parameters. Because GNO, LNO, and MGNO have the kernel networks, the slope of their error rates are flat: they can work with a very small number of parameters. On the other hand, FNO does not have the sub-network. It needs at a larger magnitude of parameters to obtain an acceptable error rate.

\begin{table}[h]

\begin{center}
\begin{tabular}{l|llll}
Number of parameters
& $10^3$
&$10^4$
&$10^5$ 
& $10^6$\\
\hline 
GNO        & $\ 0.075$  & $\ 0.065$  & $\ 0.060$ & $\ 0.035$ \\
LNO       & $\ 0.080$  & $\ 0.070$  & $\ 0.060$ & $\ 0.040$ \\
MGNO       & $\ 0.070$  & $\ 0.050$  & $\ 0.040$ & $\ 0.030$ \\
FNO      & $\ 0.200$  & $\ 0.035$  & $\ 0.020$ & $\ 0.015$ \\
\hline
\end{tabular}
\caption{Refinability.}
\label{table:refinability}
\small{The relative error on Darcy Flow with respect to different number of parameters. The errors above are approximated value roundup to $0.05$. They are the lowest test error achieved by the model, given the model's number of parameters $|\theta|$ is bounded by $10^3, 10^4, 10^5, 10^6$ respectively.}
\end{center}
\end{table}

\subsubsection{Robustness}
We conclude with experiments investigating the robustness of Fourier neural operator
to noise. We study: a) training on clean (noiseless) data and testing with clean and
noisy data; b) training on clean (noiseless) data and testing with clean and
noisy data. When creating noisy data we map $a$ to noisy $a'$ as follows: 
at every grid-point $x$ we set
\[a(x)' = a(x) + 0.1 \cdot \|a\|_{\infty}\xi,\]
where  $\xi \sim \mathcal{N}(0, 1)$ is drawn i.i.d. at every grid point;
this is similar to the setting adopted in \cite{lu2021comprehensive}. 
We also study the 1d advection 
equation as an additional test case, following the setting in \cite{lu2021comprehensive}
in which the input data is a random square wave, defined by an $\R^3$-valued random
variable.

\begin{table}[h]
\begin{center}
\begin{tabular}{l|lll}
Problems
& Training error
& Test (clean)
& Test (noisy)  \\
\hline 
Burgers        & $\ 0.002$  & $\ 0.002$  & $\ 0.018$   \\
Advection        & $\ 0.002$  & $\ 0.002$  & $\ 0.094$  \\
Darcy       & $\ 0.006$  & $\ 0.011$  & $\ 0.012$  \\
Navier-Stokes       & $\ 0.024$  & $\ 0.024$  & $\ 0.039$  \\
\hline
Burgers  (train with noise)      & $\ 0.011$  & $\ 0.004$  & $\ 0.011$  \\
Advection  (train with noise)      & $\ 0.020$  & $\ 0.010$  & $\ 0.019$  \\
Darcy    (train with noise)   & $\ 0.007$  & $\ 0.012$  & $\ 0.012$  \\
Navier-Stokes   (train with noise)    & $\ 0.026$  & $\ 0.026$  & $\ 0.025$  \\
\hline
\end{tabular}\\
\end{center}
\caption{Robustness.}
\label{table:robustness}
\end{table}

As shown in the top half of Table \ref{table:robustness} and Figure \ref{fig:robustness}, we observe the Fourier neural operator is robust with respect to the (test) noise level on all four problems. In particular, on the advection problem, it has about 10\% error with 10\% noise. The Darcy and Navier-Stokes operators are smoothing, and the Fourier neural operator obtains lower than 10\% error in all scenarios. However the FNO is less robust on the advection equation, which is not smoothing, and on Burgers equation which, whilst smoothing also forms steep fronts.

A straightforward approach to enhance the robustness is to train the model with noise. As shown in the bottom half of Table \ref{table:robustness}, the Fourier neural operator has no gap between the clean data and noisy data when training with noise. However, noise in training may degrade the performance on the clean data, as a trade-off. In general, augmenting the training data with noise leads to robustness. For example, in the auto-regressive modeling of dynamical systems, training the model with noise will reduce error accumulation in time, and thereby help the model to predict over longer time-horizons \citep{pfaff2020learning}.
We also observed that other regularization techniques such as early-stopping and weight decay improve robustness. Using a higher spatial resolution also helps.  

The advection problem is a hard problem for the FNO since it has discontinuities;
similar issues arise when using spectral methods for conservation laws. One can modify the architecture to address such discontinuities accordingly. For example,  \cite{wen2021u} enhance the FNO by composing a CNN or UNet branch with
the Fourier layer; the resulting composite model outperforms the basic FNO on multiphase flow with high contrast and sharp shocks. However the CNN and UNet take the method out of the realm
of discretization-invariant methods; further work is required to design
discretization-invariant image-processing tools, such as the identification of
discontinuities. 

\begin{figure}[ht]
    \centering
    \includegraphics[width=0.7\textwidth]{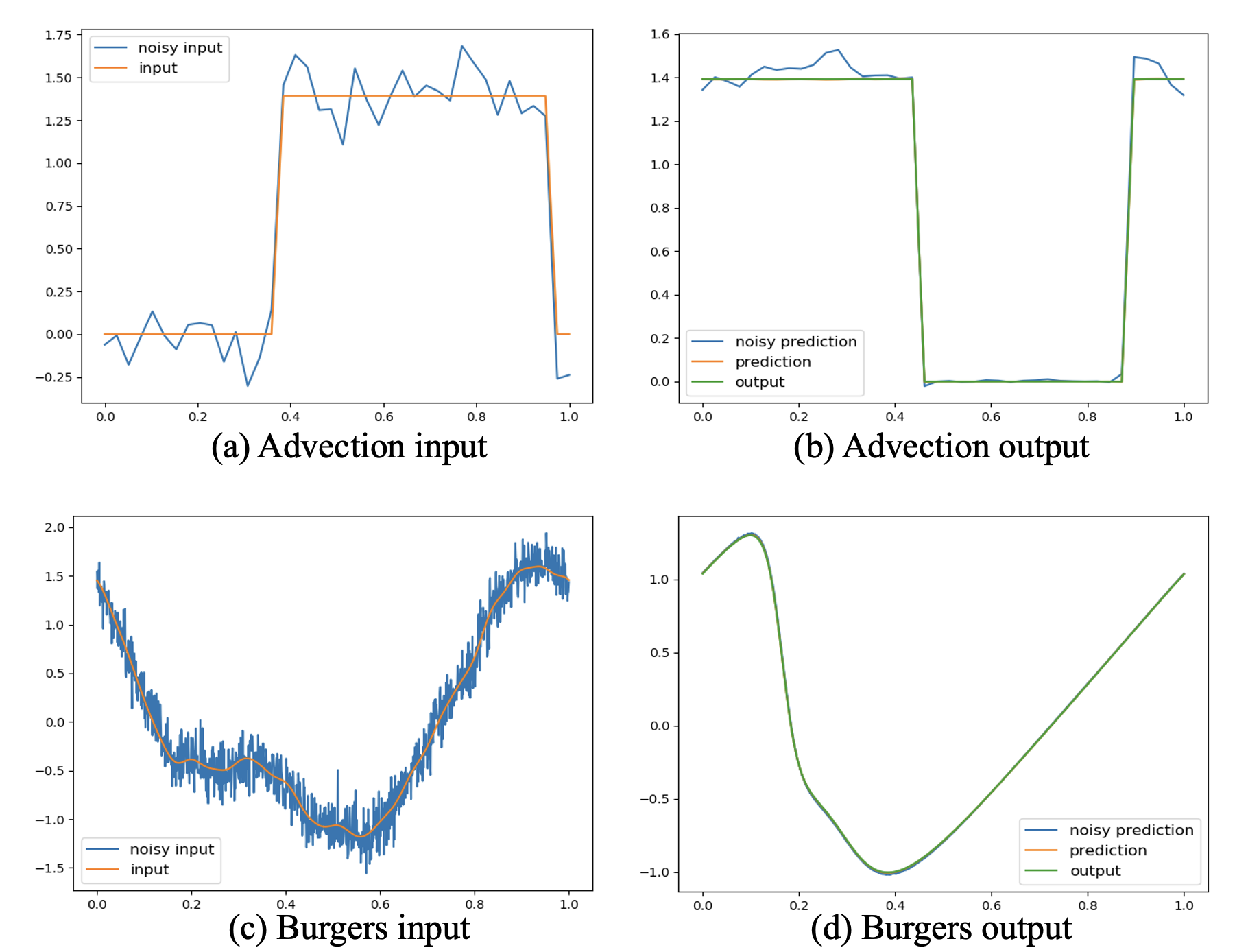}
        \caption{Robustness on Advection and Burgers equations}\label{fig:robustness}
    \small{
    (a) The input of Advection equation ($s=40$). The orange curve is the clean input; the blue curve is the noisy input. (b) The output of Advection equation.  The green curve is the ground truth output; the orange curve is the prediction of FNO with clean input (overlapping with the ground truth); the blue curve is the prediction on the noisy input. Figure (c) and (d) are for Burgers' equation ($s=1000$) correspondingly.}
\end{figure}

\section{Approximation Theory}
\label{sec:approximation}

The paper by \cite{chen1995universal} provides the first universal approximation theorem
for operator approximation via neural networks, and the paper by
\cite{Kovachki} provides an alternative architecture and approximation result.
The analysis of \cite{chen1995universal} was recently extended
in significant ways in the paper by \cite{lanthaler2021error} where, for the first
time, the curse of dimensionality is addressed, and resolved, for certain specific
operator learning problems, using the DeepOnet generalization
\cite{lu2019deeponet,lu2021learning} of \cite{chen1995universal}.
The paper \cite{lanthaler2021error} was generalized to study
operator approximation, and the curse of dimensionality, for the FNO,
in \cite{kovachki2021universal}.

Unlike the finite-dimensional setting, the choice of input 
and output spaces \(\A\) and \(\U\) for the mapping \(\G^\dagger\) play a crucial role in the approximation theory due to the distinctiveness of the induced norm topologies.
In this section, we prove universal approximation theorems for neural operators both 
with respect to the topology of uniform convergence over compact sets and 
with respect to the topology induced by
the Bochner norm \eqref{eq:bochner_error}.  We focus our attention on the Lebesgue, Sobolev, continuous, and continuously differentiable function classes as they have numerous applications in scientific computing and machine learning problems. Unlike the results of \cite{Kovachki, kovachki2021universal} which rely on the Hilbertian 
structure of the input and output spaces or the results of \cite{chen1995universal, lanthaler2021error} which rely on the continuous  functions, our results extend to more general Banach spaces as specified by Assumptions \ref{assump:input} and \ref{assump:output} \nk{(stated in Section \ref{sec:approximation_main})} and are, to the best of our knowledge, the first of their kind
to apply at this level of generality.

Our method of proof proceeds by making use of the following two observations.
First we establish the Banach space approximation property 
\cite{grothendieck1955produits} for the input and output spaces of interest, which allows for a finite dimensionalization of the problem. In particular, we prove that the
Banach space approximation property holds for various function spaces defined on Lipschitz domains; the precise result we need, while unsurprising, seems to be missing from the functional analysis literature and so we provide statement
and proof. Details are given in Appendix A. Second, we establish that integral kernel operators with smooth kernels can be used to approximate linear functionals of various input spaces. In doing so, we establish a Riesz-type representation theorem for the continuously differentiable functions. Such a result is not surprising and mimics the well-known result for Sobolev spaces; however in the form we need it we could not find the
result in the functional analysis literature and so we provide statement
and proof. Details are given in Appendix B. With these two facts, we construct a neural operator which linearly maps any input function to a finite vector then non-linearly maps this vector to a new finite vector which is then used to form the coefficients of a basis expansion for the output function. We reemphasize that our approximation theory uses the fact that neural operators can be reduced to a linear method of approximation (as pointed out in Section \ref{sec:deeponets}) and does not capture
any benefits of nonlinear approximation. However these benefits \emph{are} present in the architecture and are exploited by the trained networks we find in practice.
Exploiting their nonlinear nature to potentially obtain improved rates of approximation remains an interesting direction for future research.

The rest of this section is organized as follows. 
In Subsection \ref{sec:approximation_nos}, we define allowable activation functions and the set of neural operators used in our theory, noting that they constitute a subclass of the neural operators defined in Section \ref{sec:framework}.
In Subsection \ref{sec:approximation_main}, we state and prove our main universal approximation theorems.

\subsection{Neural Operators}
\label{sec:approximation_nos}

For any \(n \in \N\) and \(\sigma :\R \to \R\), we define the set of real-valued \(n\)-layer neural networks
on \(\R^d\) by
\begin{align*}
\NN_n (\sigma;\R^d) \coloneqq \{ f : \R^d \to \R : \: &f(x) = W_n \sigma ( \dots W_1 \sigma (W_0 x + b_0) + b_1 \dots ) + b_n, \\
&W_0 \in \R^{d_0 \times d}, W_1 \in \R^{d_1 \times d_0}, \dots, W_n \in \R^{1 \times d_{n-1}}, \\
&b_0 \in \R^{d_0}, b_1 \in \R^{d_1}, \dots, b_n \in \R, \: d_0,d_1,\dots,d_{n-1} \in \N\}.
\end{align*}
We define the set of \(\R^{d'}\)-valued neural networks simply by stacking real-valued networks
\[\NN_n(\sigma;\R^d, \R^{d'}) \coloneqq \{f : \R^d \to \R^{d'} : f(x) = \bigl ( f_1(x), \dots, f_{d'}(x) \bigl), \: f_1,\dots,f_{d'} \in \NN_n(\sigma;\R^d)\}.\]
We remark that we could have defined \(\NN_n(\sigma;\R^d, \R^{d'})\) by letting \(W_n \in \R^{d' \times d_n}\) and \(b_n \in \R^{d'}\) in 
the definition of \(\NN_n(\sigma;\R^d)\) because we allow arbitrary width, making the two definitions equivalent; however the definition as presented is  more  convenient  for  our  analysis. We also employ the preceding definition with $\R^d$ and $\R^{d'}$ 
replaced by spaces of matrices.
For any \(m \in \N_0\), we define the set of allowable activation functions as the continuous \(\R \to \R\) maps which make
neural networks dense in \(C^m(\R^d)\) on compacta at any fixed depth,
\[\mathsf{A}_m \coloneqq \{\sigma \in C(\R) : \exists n \in \N \text{ s.t. } \NN_n(\sigma;\R^d) \text{ is dense in } C^m(K) \:\: \forall K \subset \R^d \text{ compact}\}.\]
It is shown in \cite[Theorem 4.1]{pinkus1999approximation} that \(\{\sigma \in C^m(\R) : \sigma \text{ is not a polynomial}\} \subseteq \mathsf{A}_m\) with \(n = 1\).
Clearly \(\mathsf{A}_{m+1} \subseteq \mathsf{A}_{m}\).

We define the set of linearly bounded activations as
\[\mathsf{A}^{\text{L}}_m \coloneqq \left \{ \sigma \in \mathsf{A}_m : \sigma \text{ is Borel measurable }, \:\: \sup_{x \in \R} \frac{|\sigma(x)|}{1 + |x|} < \infty \right \},\]
noting that any globally Lipschitz, non-polynomial, \(C^m\)-function is contained in \(\mathsf{A}^{\text{L}}_m\). Most activation functions 
used in practice fall within this class, for example, \(\text{ReLU} \in \mathsf{A}^{\text{L}}_0\), \(\text{ELU} \in \mathsf{A}^{\text{L}}_1\)
while \(\tanh, \text{sigmoid} \in \mathsf{A}^{\text{L}}_m\) for any \(m \in \N_0\). 

For approximation in a Bochner norm, we will be interested in constructing 
globally bounded neural networks which can approximate the identity over compact sets as done in \citep{lanthaler2021error, Kovachki}. This allows us to control the potential unboundedness of the support of the input measure by exploiting the fact that the probability of an input must decay to zero in unbounded regions. Following \citep{lanthaler2021error}, we introduce the forthcoming definition which uses the notation of the diameter of a set. In particular, the diameter of any set \(S \subseteq \R^d\) is defined as, for
$|\cdot|_2$ the Euclidean norm on $\R^d$,
\[\text{diam}_2(S) \coloneqq \sup_{x,y \in S} |x - y|_2.\]

\begin{definition}
\label{def:bounded_activations}
We denote by \(\mathsf{BA}\) the set of maps \(\sigma \in \mathsf{A}_0\) such that,
for any compact set \(K \subset \R^d\), \(\epsilon > 0\), and \(C \geq \emph{\text{diam}}_2(K)\) 
, there exists 
a number \(n \in \N\) and a neural network \(f \in \NN_n(\sigma;\R^d,\R^{d'})\) such that
\begin{align*}
|f(x) - x|_2 &\leq \epsilon, \qquad \forall x \in K, \\
|f(x)|_2 &\leq C, \quad \:\:\: \forall x \in \R^d.
\end{align*} 
\end{definition}
It is shown in \cite[Lemma C.1]{lanthaler2021error} that \(\text{ReLU} \in \mathsf{A}^{\text{L}}_0 \cap \mathsf{BA}\) with \(n=3\).

We will now define the specific class of neural operators for which we prove
a universal approximation theorem. It is important to note that the class
with which we work is a simplification of the one given in \eqref{eq:F}. In particular, the lifting and projection operators \(\mathcal{Q}, \mathcal{P}\), together with the final 
activation function \(\sigma_n\), are set to the identity, and the local linear operators \(W_0,\dots,W_{n-1}\) are set to zero. In our numerical studies we have in any case
typically set \(\sigma_n\) to the identity. However we have found that learning the local
operators \(\mathcal{Q}, \mathcal{P}\) and \(W_0,\dots,W_{n-1}\) is beneficial in practice;
extending the universal approximation theorems given here to explain this benefit 
would be an important but non-trivial development of the analysis we present here.

Let \(D \subset \R^d\) be a domain.
For any \(\sigma \in \mathsf{A}_0\), we define the set of affine kernel integral operators by 
\begin{align*}
\mathsf{IO}(\sigma;D,\R^{d_1},\R^{d_2}) = \{f \mapsto \int_D \kappa(\cdot, y) f(y) \dy + b : \: &\kappa \in \NN_{n_1}(\sigma;\R^{d} \times \R^{d},\R^{d_2 \times d_1}), \\
&b \in \NN_{n_2}(\sigma;\R^{d},\R^{d_2}), \: n_1, n_2 \in \N\},
\end{align*}
for any \(d_1,d_2 \in \N\). Clearly, since \(\sigma \in \mathsf{A}_0\), any \(S \in \mathsf{IO}(\sigma;D,\R^{d_1},\R^{d_2})\) acts 
as \(S : L^p(D;\R^{d_1}) \to L^p(D;\R^{d_2})\) for any \(1 \leq p \leq \infty\) since \(\kappa \in C(\bar{D} \times \bar{D};\R^{d_2 \times d_1})\)
and \(b \in C(\bar{D};\R^{d_2})\).
For any \(n \in \N_{\geq 2}\), \(d_a,d_u \in \N\), \(D \subset \R^d\), \(D' \subset \R^{d'}\) domains, and \(\sigma_1 \in \mathsf{A}^{\text{L}}_0\), \(\sigma_2, \sigma_3 \in \mathsf{A}_0\), 
we define the set of \(n\)-layer neural operators by
\begin{align*}
\mathsf{NO}_n(\sigma_1,\sigma_2,\sigma_3;D,D',\R^{d_a},\R^{d_u}) = \{&f \mapsto \int_{D} \kappa_n (\cdot, y) \bigl ( S_{n-1} \sigma_1 ( \dots S_2 \sigma_1(S_1(S_0 f)) \dots) \bigl )(y) \dy : \\
&S_0 \in \mathsf{IO}(\sigma_2,D;\R^{d_a},\R^{d_1}), \dots S_{n-1} \in \mathsf{IO}(\sigma_2,D;\R^{d_{n-1}},\R^{d_n}), \\
&\kappa_n \in \NN_l (\sigma_3; \R^{d'} \times \R^{d}, \R^{d_u \times d_n}), \: d_1,\dots,d_n, l \in \N\}.
\end{align*}
When \(d_a = d_u = 1\), we will simply write \(\mathsf{NO}_n(\sigma_1,\sigma_2,\sigma_3;D,D')\).
Since \(\sigma_1\) is linearly bounded, we can use a result about compositions of maps in \(L^p\) spaces such as \cite[Theorem 7.13]{dudley2010concrete} to conclude that any \(G \in \mathsf{NO}_n(\sigma_1,\sigma_2,\sigma_3,D,D';\R^{d_a},\R^{d_u})\) acts as \(G : L^p(D;\R^{d_a}) \to L^p (D';\R^{d_u})\).
Note that it is only in the last layer that we transition from functions defined
over domain $D$ to functions defined over domain $D'$.

When the input space of an operator of interest is \(C^m (\bar{D})\), for \(m \in \N\),
we will need to take in derivatives explicitly as they cannot be learned using kernel integration  as employed in the current 
construction given in Lemma~\ref{lemma:c_kernelapprox}; note that this is \textit{not} the case for \(W^{m,p}(D)\) as shown in Lemma~\ref{lemma:wmp_kernelapprox}.
We will therefore define the set of \(m\)-th order 
neural operators by
\begin{align*}
\mathsf{NO}_n^m(\sigma_1,\sigma_2,\sigma_3;D,D',\R^{d_a},\R^{d_u}) = \{&(\partial^{\alpha_1}f, \dots,\partial^{\alpha_{J_m}} f) \mapsto G(\partial^{\alpha_1}f, \dots,\partial^{\alpha_{J_m}} f) : \\
& G \in \mathsf{NO}_n (\sigma_1,\sigma_2,\sigma_3;D,D',\R^{J_m d_a}, \R^{d_u})\}
\end{align*}
where \(\alpha_1,\dots,\alpha_{J_m} \in \N^d\) is an enumeration of the set
\(\{\alpha \in \N^d : 0 \leq |\alpha|_1 \leq m\}\). Since we only use the \(m\)-th order operators when dealing with spaces of continuous functions, each element of \(\mathsf{NO}_n^m\) can be thought of as a mapping from a product space of spaces of the form \(C^{m - |\alpha_j|}(\bar{D};\R^{d_a})\) for all \(j \in \{1,\dots,J_m\}\) to an appropriate Banach space of interest.

\nk{
\subsection{Discretization Invariance}
\label{sec:discritizational_invariance}
Given the construction above and the definition of discretization invariance in Definition~\ref{def:discretization_invariance}, in the following we prove that neural operators are discretization invariant deep learning models.

\begin{theorem}
\label{thm:discretizational_invariance}
Let \(D \subset \R^{d}\) and \(D' \subset \R^{d'}\) be two domains for some \(d, d' \in \N\). Let
\(\A\) and \(\U\) be real-valued Banach function spaces on \(D\) and \(D'\) respectively. Suppose that
\(\A\) and \(\U\) can be continuously embedded in \(C(\bar{D})\) and \(C(\bar{D'})\) respectively and that
\(\sigma_1, \sigma_2, \sigma_3 \in C(\R)\). Then,
for any \(n \in \N\), the set of neural operators \(\mathsf{NO}_n(\sigma_1,\sigma_2,\sigma_3;D,D')\) whose elements are viewed 
as maps \(\A \to \U\) is discretization-invariant. 
\end{theorem}

The proof, provided in appendix~\ref{sec_proof:discretizational_invariance}, constructs a sequence of finite dimensional maps which approximate the neural operator by Riemann sums and shows uniform converges of the error over compact sets of \(\A\).   

}

\subsection{Approximation Theorems}
\label{sec:approximation_main}

\begin{figure}[t]
\centering
\begin{tikzpicture}
\begin{tikzcd}
\A \arrow[rr,"F"] \arrow[d,swap,"\G^\dagger"] && \mathbb{R}^{J} \arrow[rr,""] \arrow[d,swap,"\psi"] && \A \arrow[d,swap,"\G^\dagger"] \\
\U \arrow[rr,""] && \mathbb{R}^{J'} \arrow[rr,"G"] && \U
\end{tikzcd}
\end{tikzpicture}
\hspace*{6.2cm}
\caption{A schematic overview of the maps used to approximate \(\G^\dagger\).} \label{fig:approach}
\end{figure}

Let \(\A\) and \(\U\) be Banach function spaces on the domains \(D \subset \R^d\) and \(D' \subset \R^{d'}\) respectively.
We will work in the setting where functions in \(\A\) or \(\U\) are real-valued, but note that all results generalize 
in a straightforward fashion to the vector-valued setting. We are interested in the approximation of nonlinear operators \(\G^\dagger : \A \to \U\)
by neural operators. We will make the following assumptions on the spaces \(\A\) and \(\U\).

\begin{assumption}
\label{assump:input}
Let \(D \subset \R^d\) be a Lipschitz domain for some \(d \in \N\).
One of the following holds
\begin{enumerate}
	\item \(\A = L^{p_1} (D)\) for some \(1 \leq p_1 < \infty\).
	\item \(\A = W^{m_1,p_1}(D)\) for some \(1 \leq p_1 < \infty\) and \(m_1 \in \N\),
	\item \(\A = C (\bar{D})\).
\end{enumerate}
\end{assumption}

\begin{assumption}
\label{assump:output}
Let \(D' \subset \R^{d'}\) be a Lipschitz domain for some \(d' \in \N\).
One of the following holds
\begin{enumerate}
	\item \(\U = L^{p_2} (D')\) for some \(1 \leq p_2 < \infty\),
	and \(m_2 = 0\),
	\item \(\U = W^{m_2,p_2}(D')\) for some \(1 \leq p_2 < \infty\) and \(m_2 \in \N\),
	\item \(\U = C^{m_2}(\bar{D}')\) and \(m_2 \in \N_0\).  
\end{enumerate}
\end{assumption}

We first show that neural operators are dense in the continuous operators \(\G^\dagger : \A \to \U\) in the topology of uniform convergence on compacta. The proof proceeds by making three main approximations which are schematically shown in Figure~\ref{fig:approach}. First, inputs are mapped to a finite-dimensional representation through a set of appropriate linear functionals on \(\A\) denoted by \(F : \A \to \R^{J}\). We show in Lemmas 21 and 23 that, when \(\A\) satisfies Assumption \ref{assump:input}, elements of \(\A^*\) can be approximated by integration against smooth functions. This generalizes the idea from \citep{chen1995universal} where functionals on \(C(\bar{D})\) are approximated by a weighted sum of Dirac measures. We then show in Lemma 25 that, by lifting the dimension, this representation can be approximated by a single element of \(\mathsf{IO}\). Second, the representation is non-linearly mapped to a new representation by a continuous function \(\psi : \R^{J} \to \R^{J'}\) which finite-dimensionalizes the action of \(\G^\dagger\). We show, in Lemma 28, that this map can be approximated by a neural operator by reducing the architecture to that of a standard neural network. Third, the new representation is used as the coefficients of an expansion onto representers of \(\U\), the map denoted \(G : \R^{J'} \to \U\), which we show can be approximated by a single \(\mathsf{IO}\) layer in Lemma 27 using density results for continuous functions. The structure of the overall approximation is similar to \citep{Kovachki} but generalizes the ideas from working on Hilbert spaces to the spaces in Assumptions \ref{assump:input} and \ref{assump:output}. Statements and proofs of the lemmas used in the theorems are given in the appendices.

\begin{theorem}
\label{thm:main_compact}
Let Assumptions~\ref{assump:input} and \ref{assump:output} hold and suppose \(\G^\dagger : \A \to \U\) is continuous.
Let \(\sigma_1 \in \mathsf{A}^{\emph{\text{L}}}_0\), \(\sigma_2 \in \mathsf{A}_0\), and \(\sigma_3 \in \mathsf{A}_{m_2}\).
Then for any compact set \(K \subset \A\) and \(0 < \epsilon \leq 1\), there exists a number \(N \in \N\) and a neural operator 
\(\G \in \mathsf{NO}_{N}(\sigma_1,\sigma_2,\sigma_3;D,D')\) such that
\[\sup_{a \in K} \|\G^\dagger (a) - \G (a) \|_{\U} \leq \epsilon.\]
Furthermore, if \(\U\) is a Hilbert space and \(\sigma_1 \in \mathsf{BA}\) and, for some \(M > 0\), we have that \(\|\G^\dagger (a)\|_{\U} \leq M\) for all 
\(a \in \A\) then \(\G\) can be chosen so that
\[\|\G(a)\|_{\U} \leq 4M, \qquad \forall a \in \A.\]
\end{theorem}
The proof is provided in appendix~\ref{sec_proof:main_compact}
In the following theorem, we extend this result to the case \(\A = C^{m_1}(\bar{D})\), showing density of the \(m_1\)-th order neural operators.

\begin{theorem}
\label{thm:cm_compact}
Let \(D \subset \R^d\) be a Lipschitz domain, \(m_1 \in \N\), define \(\A := C^{m_1}(\bar{D})\), suppose Assumption~\ref{assump:output} holds and
assume that  \(\G^\dagger : \A \to \U\) is continuous.
Let \(\sigma_1 \in \mathsf{A}^{\emph{\text{L}}}_0\), \(\sigma_2 \in \mathsf{A}_0\), and \(\sigma_3 \in \mathsf{A}_{m_2}\).
Then for any compact set \(K \subset \A\) and \(0 < \epsilon \leq 1\), there exists a number \(N \in \N\) and a neural operator 
\(\G \in \mathsf{NO}_{N}^{m_1}(\sigma_1,\sigma_2,\sigma_3;D,D')\) such that
\[\sup_{a \in K} \|\G^\dagger (a) - \G (a) \|_{\U} \leq \epsilon.\]
Furthermore, if \(\U\) is a Hilbert space and \(\sigma_1 \in \mathsf{BA}\) and, for some \(M > 0\), we have that \(\|\G^\dagger (a)\|_{\U} \leq M\) for all 
\(a \in \A\) then \(\G\) can be chosen so that
\[\|\G(a)\|_{\U} \leq 4M, \qquad \forall a \in \A.\]
\end{theorem}
\begin{proof}
The proof follows as in Theorem~\ref{thm:main_compact}, replacing the use of Lemma~\ref{lemma:input_approx} with Lemma~\ref{lemma:cm_input_approx}.
\end{proof}

With these results in hand, we show density of neural operators in the space \(L^2_\mu (\A;\U)\)
where \(\mu\) is a probability measure and \(\U\) is a separable Hilbert space. The Hilbertian structure of \(\U\) allows us to uniformly control the norm of the approximation due to the isomorphism with \(\ell_2\) as shown in Theorem \ref{thm:main_compact}. It remains an interesting future direction to obtain similar results for Banach spaces. The proof follows the ideas in \citep{lanthaler2021error} where similar results are obtained for DeepONet(s) on \(L^2(D)\) by using Lusin's theorem to restrict the approximation to a large enough compact set and exploit the decay of \(\mu\) outside it. \cite{Kovachki} also employ a similar approach but explicitly constructs the necessary compact set after finite-dimensionalizing. 
 
\begin{theorem}
\label{thm:measurable_approx}
Let \(D' \subset \R^{d'}\) be a Lipschitz domain, \(m_2 \in \N_0\), and suppose Assumption~\ref{assump:input} holds.
Let \(\mu\) be a probability measure on \(\A\) and suppose \(\G^\dagger : \A \to H^{m_2}(D)\) is \(\mu\)-measurable and \(\G^\dagger \in L^2_\mu (\A;H^{m_2}(D))\).
Let \(\sigma_1 \in \mathsf{A}^{\emph{\text{L}}}_0 \cap \mathsf{BA}\), \(\sigma_2 \in \mathsf{A}_0\), and \(\sigma_3 \in \mathsf{A}_{m_2}\).
Then for any \(0 < \epsilon \leq 1\), there exists a number \(N \in \N\) and a neural operator 
\(\G \in \mathsf{NO}_{N}(\sigma_1,\sigma_2,\sigma_3;D,D')\) such that
\[\|\G^\dagger  - \G  \|_{L^2_\mu(\A;H^{m_2}(D))} \leq \epsilon.\]
\end{theorem}
The proof is provided in appendix~\ref{sec_proof:measurable_approx}. In the following we extend this result to the case \(\A = C^{m_1}(D)\) using the \(m_1\)-th order neural operators.

\begin{theorem}
\label{thm:cm_measurable_approx}
Let \(D \subset \R^d\) be a Lipschitz domain, \(m_1 \in \N\), define
\(\A := C^{m_1}(D)\) and suppose Assumption~\ref{assump:output} holds. Let \(\mu\)
be a probability measure on \(C^{m_1}(D)\) and let 
\(\G^\dagger : C^{m_1}(D) \to \U\) be \(\mu\)-measurable and suppose \(\G^\dagger \in L^2_\mu (C^{m_1}(D);\U)\).
Let \(\sigma_1 \in \mathsf{A}^{\emph{\text{L}}}_0 \cap \mathsf{BA}\), \(\sigma_2 \in \mathsf{A}_0\), and \(\sigma_3 \in \mathsf{A}_{m_2}\).
Then for any \(0 < \epsilon \leq 1\), there exists a number \(N \in \N\) and a neural operator 
\(\G \in \mathsf{NO}_{N}^{m_1}(\sigma_1,\sigma_2,\sigma_3;D,D')\) such that
\[\|\G^\dagger  - \G  \|_{L^2_\mu(C^{m_1}(D);\U)} \leq \epsilon.\]
\end{theorem}
\begin{proof}
The proof follows as in Theorem~\ref{thm:measurable_approx} by replacing the use of Theorem~\ref{thm:main_compact}
with Theorem~\ref{thm:cm_compact}.
\end{proof}

\section{Literature Review}
\label{ssec:related-work}

We outline the major neural network-based approaches for the solution of PDEs. 
\paragraph{Finite-dimensional Operators.}
An immediate approach to approximate \(\G^\dagger\) is to parameterize it as a deep convolutional neural network (CNN) between the finite-dimensional Euclidean spaces on which the data is discretized i.e. \(\F : \R^K \times \Theta \to \R^K\) \citep{guo2016convolutional, Zabaras, Adler2017, bhatnagar2019prediction,kutyniok2022atheoretical}.  \citet{khoo2017solving}
concerns a similar setting, but with output space $\R$. Such approaches are, by definition, not mesh independent and need modifications to the architecture for different resolution and discretization of \(D\) in order to achieve consistent error (if at all possible). We demonstrate this issue numerically in Section \ref{sec:numerics}. Furthermore, these approaches are limited to the discretization size and geometry of the training data and hence it is not possible to query solutions at new points in the domain. In contrast for our method, we show in Section \ref{sec:numerics}, both invariance of the error to grid resolution, and the ability to transfer the solution between meshes.
The work \citet{ummenhofer2020lagrangian} proposed a continuous convolution network for fluid problems, where off-grid points are sampled and linearly interpolated. However the continuous convolution method is still constrained by the underlying grid which prevents generalization to higher resolutions. Similarly, to get finer resolution solution, \citet{jiang2020meshfreeflownet} proposed learning super-resolution with a U-Net structure for fluid mechanics problems. However fine-resolution data is needed for training, while neural operators are capable of zero-shot super-resolution with no new data.

\paragraph{DeepONet}
A novel operator regression architecture, named DeepONet, was recently proposed by \cite{lu2019deeponet, lu2021learning}; it builds an iterated or deep structure on top
of the shallow architecture proposed in \cite{chen1995universal}. The architecture consists of two neural networks: a branch net applied on the input functions and a trunk net applied on the querying locations in the output space. The original work of \cite{chen1995universal}
provides a universal approximation theorem, and more recently \cite{lanthaler2021error} developed an error estimate for DeepONet itself.
The standard DeepONet structure is a linear approximation of the target operator, where the trunk net and branch net learn the coefficients and basis. On the other hand, the neural operator setting is heavily inspired by the advances in deep learning and is a non-linear approximation, which makes it constructively more expressive. A detailed discussion of DeepONet is provided in Section \ref{sec:deeponets} and as well as a numerical comparison to DeepONet in Section \ref{ssec:darcyburgers}.

\paragraph{Physics Informed Neural Networks (PINNs), Deep Ritz Method (DRM), and Deep Galerkin Method (DGM).}
A different approach is to directly parameterize the solution \(u\) as a neural network \(u : \bar{D} \times \Theta \to \R\) 
\citep{Weinan, raissi2019physics,sirignano2018dgm,bar2019unsupervised,smith2020eikonet,pan2020physics,beck2021solving}. This approach is designed to model one specific instance of the PDE, not the solution operator. It is mesh-independent, but for any given  new parameter coefficient function \(a \in \A\), one would need to train a new neural network \(u_a\) which is computationally costly and time consuming. Such an approach closely resembles classical methods such as finite elements, replacing the linear span of a finite set of local basis functions with the space of neural networks. 

\paragraph{ML-based Hybrid Solvers}
Similarly, another line of work proposes to enhance existing numerical solvers with neural networks by building hybrid models \citep{pathak2020using, um2020solver, greenfeld2019learning}
These approaches suffer from the same computational issue as classical methods: one needs to solve an optimization problem for every new parameter similarly to the PINNs setting. Furthermore, the approaches are limited to a setting in which the underlying PDE is known. Purely data-driven learning of a map between spaces of functions is not possible.

\paragraph{Reduced Basis Methods.}
Our methodology most closely
resembles the classical reduced basis method (RBM) \citep{DeVoreReducedBasis} or the method of \citet{cohendevore}. 
The method introduced here, along with the contemporaneous work introduced
in the papers \citep{ Kovachki,nelsen2020random,opschoor2020deep, schwab2019deep, o2020derivative, lu2019deeponet,fresca2022poddlrom}, are, to the best of our knowledge, 
amongst the first practical supervised learning methods designed to learn maps between infinite-dimensional spaces. Our methodology 
addresses the mesh-dependent nature of the approach in the papers~\citep{guo2016convolutional, Zabaras, Adler2017, bhatnagar2019prediction} by producing a single set of network
parameters that can be used with different discretizations. 
Furthermore, it has the ability to transfer solutions between meshes and indeed
between different discretization methods. Moreover, it needs only to be trained once on the equation set  \(\{a_j, u_j\}_{j=1}^N\). Then, obtaining a solution for a new \(a \sim \mu\) only requires a forward pass of the network, alleviating the major computational issues incurred in \citep{Weinan, raissi2019physics, herrmann2020deep, bar2019unsupervised}
where a different network would need to be trained for each input
parameter. Lastly, 
our method requires no knowledge of the underlying PDE: it is purely
data-driven and therefore non-intrusive. Indeed the true map can be treated as a black-box, perhaps to be learned from experimental data or from the output of a costly computer simulation, not necessarily from a PDE. \done{I think we should be citing DeepOnet -- the original
"practice" paper and the "theory" paper of Sid/Samuel.}

\paragraph{Continuous Neural Networks.}
Using continuity as a tool to design and interpret neural networks is gaining currency in the machine learning community, and the
formulation of ResNet as a continuous time process over the depth parameter is a powerful example of this
\citep{haber2017stable,weinan2017proposal}. The concept of defining neural networks in infinite-dimensional spaces is a central problem that has long been studied \citep{Williams, Neal, BengioLeRoux, GlobersonLivni, Guss}. The general idea is to take the infinite-width limit which yields a non-parametric method and has connections to Gaussian Process Regression \citep{Neal, MathewsGP, Garriga-AlonsoGP}, leading to the introduction of
deep Gaussian processes \citep{damianou2013deep,aretha}. Thus far, such methods have not yielded efficient numerical algorithms that can parallel the success of convolutional or recurrent neural networks for the problem of approximating mappings between finite dimensional spaces. Despite the superficial similarity with our proposed
work, this body of work differs substantially from what we
are proposing: in our work we are motivated by the continuous dependence
of the data, in the input or output spaces, in spatial or spatio-temporal variables;
in contrast the work outlined in this paragraph uses continuity in an artificial
algorithmic depth or width parameter to study the network architecture when the depth or width approaches
infinity, but the input and output spaces remain of fixed finite dimension.

\paragraph{Nystr\"om Approximation, GNNs, and Graph Neural Operators (GNOs).} 
The graph neural operators (Section \ref{sec:graphneuraloperator}) has an underlying Nystr\"om approximation formulation \citep{nystrom1930praktische} which links different grids to a single set of network parameters. This perspective relates our continuum approach to Graph Neural Networks (GNNs). GNNs are a recently developed class of neural networks that apply to graph-structured data; they
have been used in a variety of applications. Graph networks incorporate an array of techniques from neural network design such as graph convolution, edge convolution, attention, and graph pooling  \citep{kipf2016semi,hamilton2017inductive,gilmer2017neural,velivckovic2017graph,murphy2018janossy}. 
GNNs have also been applied to the modeling of physical phenomena such as molecules \citep{chen2019graph} and rigid body systems \citep{battaglia2018relational} since these problems exhibit a natural graph interpretation: the particles are the nodes and the interactions are the edges. The work \citep{pmlr-v97-alet19a} performs an initial study that employs graph networks on the problem of learning solutions to Poisson's equation, among other physical applications. They propose an encoder-decoder setting, constructing graphs in the latent space, and utilizing message passing between the encoder and decoder. However, their model uses a nearest neighbor structure that is unable to capture non-local dependencies as the mesh size is increased.
In contrast, we directly construct a graph in which the nodes are located on the spatial domain of the output function. Through message passing, we are then able to directly learn the kernel of the network
which approximates the PDE solution. When querying a new location, we simply add a new node to our spatial graph and connect it to the existing nodes, avoiding interpolation error by leveraging the power of the Nystr\"om extension for integral operators.

\paragraph{Low-rank Kernel Decomposition and Low-rank Neural Operators (LNOs).}
Low-rank decomposition is a popular method used in kernel methods and Gaussian process
\citep{kulis2006learning, bach2013sharp, lan2017low, gardner2018product}.
We present the low-rank neural operator in Section \ref{sec:lowrank} where we structure the kernel network as a product of two factor networks inspired by Fredholm theory. The low-rank method, while simple, is very efficient and easy to train especially when the target operator is close to linear.
\citet{khoo2019switchnet} proposed a related neural network with low-rank structure to approximate the inverse of differential operators.
The framework of two factor networks is also similar to the trunk and branch network used in DeepONet \citep{lu2019deeponet}. But in our work, the factor networks are defined on the physical domain and non-local information is accumulated through integration with respect to the Lebesgue measure. 
In contrast, DeepONet(s) integrate against delta measures at a set of pre-defined nodal points that are usually taken to be the grid on which the data is given. See section \ref{sec:deeponets} for further discussion.

\paragraph{Multipole, Multi-resolution Methods, and Multipole Graph Neural Operators (MGNOs).}
To efficiently capture long-range interaction,   multi-scale methods such as the classical fast multipole methods (FMM) have been developed \citep{greengard1997new}. Based on the assumption that long-range interactions decay quickly, FMM decomposes the kernel matrix into different ranges and hierarchically imposes low-rank structures on the long-range components (hierarchical matrices) \citep{borm2003hierarchical}. This decomposition can be viewed as a specific form of the multi-resolution matrix factorization of the kernel \citep{kondor2014multiresolution, borm2003hierarchical}.
For example, the works of \citet{fan2019multiscale,fan2019multiscale2, he2019mgnet} propose a similar multipole expansion for solving parametric PDEs on structured grids.  
However, the classical FMM requires nested grids as well as the explicit form of the PDEs. In Section \ref{sec:multipole}, 
we propose the multipole graph neural operator (MGNO) by generalizing this idea to arbitrary graphs in the data-driven setting, so that the corresponding graph neural networks can learn discretization-invariant solution operators which are fast and can work on complex geometries.

\paragraph{Fourier Transform, Spectral Methods, and Fourier Neural Operators (FNOs).}
The Fourier transform is frequently used in spectral methods for solving differential equations since differentiation is equivalent to multiplication in the Fourier domain.
Fourier transforms have also played an important role in the development of deep learning. 
They are used in theoretical work, such as the proof of the neural network
universal approximation theorem \citep{hornik1989multilayer} and related
results for random feature methods \citep{rahimi2008uniform}; empirically, they have been used to speed up convolutional neural networks \citep{mathieu2013fast}.
Neural network architectures involving the Fourier transform or the use of sinusoidal activation functions have also been proposed and studied \citep{bengio2007scaling,mingo2004Fourier, sitzmann2020implicit}.
Recently, some spectral methods for PDEs have been extended to neural networks \citep{fan2019bcr, fan2019multiscale, kashinath2020enforcing}. In Section \ref{sec:fourier}, we build on these works by proposing the Fourier neural operator architecture defined directly in Fourier space with quasi-linear time complexity and state-of-the-art approximation capabilities. 

\paragraph{Sources of Error}

In this paper we will study the error resulting from approximating
an operator (mapping between Banach spaces) from within a class of
finitely-parameterized operators. We show that the resulting error,
expressed in terms of universal approximation of operators over
a compact set or in terms of a resulting risk,
can be driven to zero by increasing the number of parameters, and 
refining the approximations inherent in the neural operator architecture.
In practice there will be two other sources of approximation error: firstly from the
discretization of the data; and secondly from the use of empirical
risk minimization over a finite data set to determine the parameters. 
Balancing all three sources of error is key to making algorithms efficient.
However we do not study these other two sources of error in this work. Furthermore we do not study how the number of parameters in our approximation grows as the error tolerance is refined. Generally, this growth may be super-exponential as shown in \citep{kovachki2021universal}. However, for certain classes of operators and related approximation methods, it is possible to beat the curse of dimensionality; 
we refer the reader to the works 
\citep{lanthaler2021error,kovachki2021universal} for detailed analyses
demonstrating this. Finally we also emphasize that there
is a potential source of error from the optimization procedure which attempts to minimize the empirical risk: it may not achieve the global minumum. Analysis of this error
in the context of operator approximation has not been undertaken.

%
%

\section{Conclusions}
\label{sec:conclusion}

We have introduced the concept of Neural Operator, the goal being
to construct a neural network architecture adapted to the problem
of mapping elements of one function space into elements of
another function space. The network is comprised of four steps
which, in turn,  (i) extract  features from the input
functions, (ii) iterate a recurrent neural network on feature space, defined through composition of a sigmoid function
and a nonlocal operator, and (iii) a final mapping from feature space into the output function.

We have studied four nonlocal operators in step (iii), one based
on graph kernel networks, one based on the low-rank decomposition, one based on the multi-level graph structure, and the last one based on convolution in Fourier space.  The designed network architectures are constructed to be mesh-free and our numerical experiments demonstrate that they have the desired property of being able to train and generalize on different meshes. This is because the networks learn the mapping between infinite-dimensional function spaces, which can then be shared with approximations at different levels of discretization. A further advantage  of the integral operator approach is that data may be incorporated
on unstructured grids, using the Nystr\"om approximation; these
methods, however, are quadratic in the number of discretization points; we describe variants on this methodology, using low rank
and multiscale ideas, to reduce this complexity. On the other hand the Fourier approach leads directly  to fast methods, linear-log linear in the number of discretization points, provided structured grids are used.
We demonstrate that our methods can achieve competitive performance with other mesh-free approaches developed in the numerical analysis community. 
Specifically, the Fourier neural operator achieves the best numerical performance among our experiments, potentially due to the smoothness of the solution function and the underlying uniform grids.
The methods developed in the numerical analysis community are less flexible than the approach we introduce here, relying heavily on the structure of an underlying PDE mapping input to output; our method is entirely data-driven. 

\subsection{Future Directions}   
We foresee three main directions in which this work will develop: firstly as a method to speed-up scientific computing tasks which involve repeated evaluation of a mapping between spaces of functions, following the example of the Bayesian inverse problem \ref{sec:bayesian}, or when the underlying model is unknown as in computer vision or robotics; 
and secondly the development of more advanced methodologies beyond the four approximation schemes presented in Section \ref{sec:four_schemes} that are more efficient or better in specific situations; thirdly, the development of 
an underpinning theory which captures the expressive power,  and approximation error properties, of the proposed neural network, following Section \ref{sec:approximation}, and
quantifies the computational complexity required to achieve given error.

\done{The preceding three points need to be linked to subsubsections
which follow. As it is things appear disjoint. Also an important
potential application is problems for which there is no model:
purely data driven. I think Robotics will provide such examples.}

\subsubsection{New Applications}
The proposed neural operator is a blackbox surrogate model for function-to-function mappings. It naturally fits into solving PDEs for physics and engineering problems. In the paper we mainly studied three partial differential equations: Darcy Flow, Burgers' equation, and Navier-Stokes equation, which cover a broad range of scenarios. Due to its blackbox structure, the neural operator is easily applied on other problems. We foresee applications on more challenging turbulent flows, such as those arising in subgrid models with in climate GCMs, 
high contrast media in geological models generalizing the Darcy model, and general physics simulation for games and visual effects. The operator setting leads to an efficient and accurate representation, and the resolution-invariant properties make it possible to training and a smaller resolution dataset, and be evaluated on arbitrarily large resolution.

The operator learning setting is not restricted to scientific computing. For example, in computer vision, images can naturally be viewed as real-valued functions on 2D domains and videos simply add a temporal structure. Our approach is therefore a natural choice for problems in computer vision where invariance to discretization is crucial. We leave this as an interesting future direction. 
\done{I think computer vision IS part of science or engineering so I suggest to reword the preceding.}

\subsubsection{New Methodologies}
Despite their excellent performance, there is still room for improvement upon the current methodologies. For example, the full $O(J^2)$ integration method still outperforms the FNO by about $40\%$,
albeit at greater cost. It is of potential interest to develop more advanced integration techniques or approximation schemes that follows the neural operator framework. For example, one can use adaptive graph or probability estimation in the Nystr\"om approximation. It is also possible to use other basis than the Fourier basis such as the PCA basis and the Chebyshev basis. 

Another direction for new methodologies is to combine the neural operator in other settings. The current problem is set as a supervised learning problem. Instead, one can combine the neural operator with solvers \citep{pathak2020using, um2020solverintheloop}, augmenting and correcting the solvers to get faster and more accurate approximation. Similarly, one can combine operator learning with physics constraints \citep{wang2021learning,li2021physics}.

\subsubsection{Theory}
In this work, we develop a universal approximation theory (Section \ref{sec:approximation})
for neural operators. As in the work of \cite{lu2019deeponet} studying universal
approximation for DeepONet, we use linear approximation techniques. The power of 
non-linear approximation \citep{devore1998nonlinear}, which is likely intrinsic 
to the success of neural operators in some settings, is still less studied, as discussed in Section \ref{sec:deeponets}; we note that DeepOnet is intrinsically limited by linear
approximation properties. 
For functions between Euclidean spaces, we clearly know, by combining two layers of linear functions with one layer of non-linear activation function, the neural network can approximate arbitrary continuous functions, and that deep neural networks can be exponentially more expressive compared to shallow networks \citep{poole2016exponential}. However issues are less clear when it comes to the choice of architecture and the scaling of the number of parameters within neural operators between Banach spaces.
The approximation theory of operators is much more complex and challenging compared to that of functions over Euclidean spaces. It is important to study the class of neural operators with respect to their architecture -- what spaces the true solution operators lie in, and which classes of PDEs the neural operator approximate efficiently. We leave these as exciting, but
open, research directions.

\newpage

\subsection*{Acknowledgements}
Z. Li gratefully acknowledges the financial support from the Kortschak Scholars, PIMCO Fellows, and Amazon AI4Science Fellows programs.
A. Anandkumar is supported in part by Bren endowed chair. 
K. Bhattacharya, N. B. Kovachki, B. Liu and A. M. Stuart gratefully acknowledge the financial support of the Army Research Laboratory through the Cooperative Agreement Number W911NF-12-0022. Research was sponsored by the Army Research Laboratory and was accomplished under Cooperative Agreement Number W911NF-12-2-0022. 
AMS is also supported by NSF (award DMS-1818977). 
Part of this research is developed when K. Azizzadenesheli was with the Purdue University. The authors are grateful to Siddhartha Mishra for his valuable feedback on this work.

The views and conclusions contained in this document are those of the authors and should not be interpreted as representing the official policies, either expressed or implied, of the Army Research Laboratory or the U.S. Government. The U.S. Government is authorized to reproduce and distribute reprints for Government purposes notwithstanding any copyright notation herein. 

The computations presented here were conducted on the Resnick High Performance Cluster at the California Institute of Technology.

\bibliography{ref}
\newpage

\newpage
\onecolumn
\appendix

\section{}
\label{sec:appendix_notation}

\begin{table}[ht]
\begin{center}
\begin{tabular}{|l|l|}
\multicolumn{1}{c}{\bf Notation} 
&\multicolumn{1}{c}{\bf Meaning}\\

\hline 
{\bf Operator Learning} &\\
$D \subset \R^d$  & The spatial domain for the PDE. \\
$x \in D$  & Points in the the spatial domain. \\
$a \in \A = (D;\R^{d_a})$  & The input functions (coefficients, boundaries, and/or initial conditions). \\
$u \in \U = (D;\R^{d_u})$  & The target solution functions. \\
$D_j$ & The discretization of $(a_j, u_j)$.\\
$\Ftrue: \A \to \U$  & The operator mapping the coefficients to the solutions.\\
$\mu$ & A probability measure where $a_j$ sampled from.\\
\hline
{\bf Neural Operator} &\\
$v(x) \in \R^{d_v}$  & The neural network representation of $u(x)$ \\
$d_a$ & Dimension of the input $a(x)$.\\
$d_u$ & Dimension of the output $u(x)$.\\
$d_v$ & The dimension of the representation $v(x)$.\\
$t = 0,\ldots,T$  & The layer (iteration) in the neural operator .\\
$\cP, \cQ$ & The pointwise linear transformation $\cP: a(x) \mapsto v_0(x)$ and $\cQ: v_T(x) \mapsto u(x)$. \\
$\cK$ & The integral operator in the iterative update $v_t \mapsto v_{t+1}$, \\
$\kappa : \R^{2(d+1)} \to \R^{d_v \times d_v}$  & The kernel maps $(x,y,a(x),a(y))$ to a $d_v \times d_v$ matrix\\
$K \in \R^{n \times n \times d_v \times d_v}$ & The kernel matrix with $K_{xy} = \kappa(x,y)$. \\
$W \in \R^{d_v \times d_v}$ & The pointwise linear transformation used as the bias term in the iterative update.\\
$\sigma $  & The activation function.  \\
\hline
\end{tabular}
\caption{Table of notations: operator learning and neural operators}
\label{table:notations1}
\small{In the paper, we will use lowercase letters such as $v, u$ to represent vectors and functions; uppercase letters such as $W, K$ to represent matrices or discretized transformations; and calligraphic letters such as $\cF, \cG$ to represent operators.
}
\end{center}
\end{table}

We write \(\N= \{1,2,3,\dots\}\) and \(\N_0 = \N \cup \{0\}\). Furthermore, we denote by \(|\cdot|_p\) the \(p\)-norm on any Euclidean space. We say \(\X\) is a Banach space if it is a Banach space over the real field \(\R\). 
We denote by \(\|\cdot\|_\X\) its norm and by \(\X^*\) its topological (continuous) dual.
In particular, \(\X^*\) is the Banach space consisting  of all continuous linear functionals 
\(f : \X \to \R\) with the operator norm
\[\|f\|_{\X^*} = \sup_{\substack{x \in \X \\ \|x\|_\X = 1}} |f(x)| < \infty.\]
For any Banach space \(\Y\), we denote by \(\mathcal{L}(\X;\Y)\) the Banach space of continuous 
linear maps \(T : \X \to \Y\) with the operator norm
\[\|T\|_{\X \to \Y} = \sup_{\substack{x \in \X \\ \|x\|_\X = 1}} \|Tx\|_\Y < \infty.\]
We will abuse notation and write \(\|\cdot\|\) for any operator norm when there is no ambiguity about the
spaces in question.

Let \(d \in \N\). We say that \(D \subset \R^d\) is a \textit{domain} if it is a bounded and connected open set that is topologically regular i.e. \(\text{int}(\bar{D}) = D\).
Note that, in the case \(d = 1\), a domain is any bounded, open interval. For \(d \geq 2\),
we say \(D\) is a \textit{Lipschitz domain} if \(\partial D\) can be locally represented as the graph of 
a Lipschitz continuous function defined on an open ball of \(\R^{d-1}\). If \(d=1\), we will call 
any domain a Lipschitz domain. For any multi-index \(\alpha \in \N_0^d\), we write \(\partial^\alpha f \) for the \(\alpha\)-th weak partial derivative of \(f\) when it exists. 

Let \(D \subset \R^d\) be a domain. For any \(m \in \N_0\), we define the following spaces
\begin{align*}
C (D) &=  \{f : D \to \R : f \text{ is continuous}\}, \\
C^m (D) & = \{f : D \to \R : \partial^\alpha f \in C^{m - |\alpha|_1}(D) \:\: \forall \: 0 \leq |\alpha|_1 \leq m\}, \\
C^m_{\text{b}}(D) &= \bigg \{f \in C^m (D) : \max_{0 \leq |\alpha|_1 \leq m} \sup_{x \in D} |\partial^\alpha f (x)| < \infty \bigg\}, \\
C^m (\bar{D}) &= \{f \in C^m_{\text{b}}(D) : \partial^\alpha f \text{ is uniformly continuous } \forall \: 0 \leq |\alpha|_1 \leq m \}
\end{align*}
and make the equivalent definitions when \(D\) is replaced with \(\R^d\). Note that any function in \(C^m (\bar{D})\)
has a unique, bounded, continuous extension from \(D\) to \(\bar{D}\) and is hence uniquely defined on \(\partial D\).
We will work with this extension without further notice. We remark that 
when \(D\) is a Lipschitz domain, the following definition for \(C^m (\bar{D})\) is equivalent
\[C^m (\bar{D}) = \{f : \bar{D} \to \R : \exists F \in C^m(\R^d) \text{ such that } f \equiv F|_{\bar{D}} \},\]
see \cite{whitney1934functions,brudnyi2012methods}. We define \(C^\infty (D) = \bigcap_{m=0}^\infty C^m (D)\) and, similarly, \(C^\infty_{\text{b}}(D)\) and \(C^\infty (\bar{D})\).
We further define
\[C^\infty_c (D) = \{f \in C^\infty (D) : \text{supp}(f) \subset D \text{ is compact}\}\]
and, again, note that all definitions hold analogously for \(\R^d\).
We denote by \(\|\cdot\|_{C^m} : C^m_{\text{b}}(D) \to \R_{\geq 0}\) the norm
\[\|f\|_{C^m} = \max_{0 \leq |\alpha|_1 \leq m} \sup_{x \in D} |\partial^\alpha f (x)|\]
which makes \(C^m_{\text{b}}(D)\) (also with \(D = \R^d\)) and \(C^m (\bar{D})\) Banach spaces.
For any \(n \in \N\), we write \(C(D;\R^n)\) for the \(n\)-fold Cartesian 
product of \(C(D)\) and similarly for all other spaces we have defined or will
define subsequently. We will
continue to write \(\|\cdot\|_{C^m}\) for the norm on \(C^m_{\text{b}}(D;\R^n)\)
and \(C^m(\bar{D};\R^n)\) defined as
\[\|f\|_{C^m} = \max_{j \in \{1,\dots,n\}} \|f_j\|_{C^m}.\]

For any \(m \in \N\) and \(1 \leq p \leq \infty\), we use the notation \(W^{m,p}(D)\) for the
standard \(L^p\)-type Sobolev space with \(m\) derivatives; we refer the reader to \cite{adams2003sobolev} for a formal definition. Furthermore, we, at times, use the notation \(W^{0,p}(D) = L^p(D)\) and \(W^{m,2}(D) = H^m (D)\). Since we use the standard definitions of Sobolev spaces that can be found in any reference on the subject, we do not give the specifics here.


\section{}
\label{sec:appendix_approximationproperty}

In this section we gather various results on the approximation property of Banach spaces.
The main results are Lemma~\ref{lemma:finitedim_approx} which  states that if two Banach spaces have
the approximation property then continuous maps between them can be approximated in a 
finite-dimensional manner, and Lemma~\ref{lemma:ap} which states the spaces in Assumptions~\ref{assump:input}
and \ref{assump:output} have the approximation property.

\begin{definition}
\label{def:schauder_bases}
A Banach space \(\X\) has a \textit{Schauder basis} if there exist some \(\{\varphi_j\}_{j=1}^\infty \subset \X\) and \(\{c_j\}_{j=1}^\infty \subset \X^*\)
such that 
\begin{enumerate}
	\item \(c_j(\varphi_k) = \delta_{jk}\) for any \(j,k \in \N\),
	\item \(\lim\limits_{n \to \infty} \|x - \sum_{j=1}^n c_j(x) \varphi_j \|_\X = 0\) for all \(x \in \X\).
\end{enumerate}
\end{definition}
We remark that definition \ref{def:schauder_bases} is equivalent to the following. The elements \(\{\varphi_j\}_{j=1}^\infty \subset \X\)
are called a \textit{Schauder basis} for \(\X\) if, for each \(x \in \X\), there exists a unique sequence \(\{\alpha_j\}_{j=1}^\infty \subset \R\)
such that
\[\lim_{n \to \infty} \|x - \sum_{j=1}^n \alpha_j \varphi_j\|_\X = 0.\]
For the equivalence, see, for example \cite[Theorem 1.1.3]{albiac2006topics}. Throughout
this paper we will simply write the
term \textit{basis} to mean Schauder basis. Furthermore, we note that if \(\{\varphi\}_{j=1}^\infty\) is a basis 
then so is \(\{\varphi_j / \|\varphi\|_{\X}\}_{j=1}^\infty\), so we will assume that any basis we use is normalized.

\begin{definition}
\label{def:finiterank}
Let \(\X\) be a Banach space and \(U \in \mathcal{L}(\X;\X)\). \(U\) is called a \textit{finite rank operator}
if \(U(\X) \subseteq \X\) is finite dimensional. 
\end{definition}
By noting that any finite dimensional subspace has a basis,
we may equivalently define a finite rank operator \(U \in \mathcal{L}(\X;\X)\) to be one such that
there exists a number \(n \in \N\) and some \(\{\varphi_j\}_{j=1}^n \subset \X\) and \(\{c_j\}_{j=1}^n \subset \X^*\)
such that
\[Ux = \sum_{j=1}^n c_j(x) \varphi_j, \qquad  \forall x \in \X.\]

\begin{definition}
\label{def:ap}
A Banach space \(\X\) is said to have the \textit{approximation property} (\emph{AP}) if, for any compact set \(K \subset \X\)
and \(\epsilon > 0\), there exists a finite rank operator \(U : \X \to \X\) such that
\[\|x - Ux\|_{\X} \leq \epsilon, \qquad  \forall x \in K.\]
\end{definition}

We now state and prove some well-known results about the relationship between basis and the AP. We were unable to find the statements of the following lemmas in the form given here in the literature and therefore we provide full proofs.

\begin{lemma}
\label{lemma:schauder_ap}
Let \(\X\) be a Banach space with a basis. Then \(\X\) has the \emph{AP}.
\end{lemma}
\begin{proof}
Let \(\{c_j\}_{j=1}^\infty \subset \X^*\) and \(\{\varphi_j\}_{j=1}^\infty \subset \X\) be a basis for \(\X\).
Note that there exists a constant \(C > 0\) such that, for any \(x \in \X\) and \(n \in \N\),
\[\|\sum_{j=1}^n c_j(x) \varphi_j\|_\X \leq \sup_{J \in \N} \|\sum_{j=1}^J c_j(x) \varphi_j\|_\X \leq C \|x\|_\X,\]
see, for example \cite[Remark 1.1.6]{albiac2006topics}. Assume, without loss of generality, that \(C \geq 1\).
Let \(K \subset \X\) be compact and \(\epsilon > 0\). Since \(K\) is compact, we can find 
a number \(n =  n(\epsilon, C) \in \N\) and elements \(y_1,\dots,y_n \in K\) such that
for any \(x \in K\) there exists a number \(l \in \{1,\dots,n\}\) with the property that
\[\|x - y_l\|_\X \leq \frac{\epsilon}{3C}.\] 
We can then find a number \(J = J(\epsilon,n) \in \N\) such that
\[\max_{j \in \{1,\dots,n\}}\|y_j - \sum_{k=1}^J c_k(y_j) \varphi_k \|_\X \leq \frac{\epsilon}{3}.\] 
Define the finite rank operator \(U : \X \to \X\) by
\[Ux = \sum_{j=1}^J c_j(x) \varphi_j, \qquad \forall x \in \X.\]
Triangle inequality implies that, for any \(x \in K\),
\begin{align*}
\|x - U(x)\|_\X &\leq \|x - y_l\|_\X + \|y_l - U(y_l)\|_\X + \|U(y_l) - U(x)\|_\X \\
&\leq \frac{2 \epsilon}{3} + \|\sum_{j=1}^J \bigl (c_j(y_l) - c_j(x) \bigl ) \varphi_j \|_\X \\
&\leq \frac{2 \epsilon}{3} + C\|y_l - x\|_\X \\
&\leq \epsilon
\end{align*}
as desired.
\end{proof}

\begin{lemma}
\label{lemma:schauder_schauder}
Let \(\X\) be a Banach space with a basis and \(\Y\) be any Banach space. Suppose there exists a continuous linear bijection \(T : \X \to \Y\).
Then \(\Y\) has a basis. 
\end{lemma}
\begin{proof}
Let \(y \in \Y\) and \(\epsilon > 0\). Since \(T\) is a bijection, there exists an element \(x \in \X\) so that \(Tx = y\) and \(T^{-1}y = x\).
Since \(\X\) has a basis, we can find \(\{\varphi_j\}_{j=1}^\infty \subset \X\) and \(\{c_j\}_{j=1}^\infty \subset \X^*\) 
and a number \(n = n(\epsilon, \|T\|) \in \N\) such that
\[\|x - \sum_{j=1}^n c_j(x) \varphi_j \|_\X \leq \frac{\epsilon}{\|T\|}.\]
Note that
\[
\|y - \sum_{j=1}^n c_j(T^{-1}y) T\varphi_j \|_\Y = \|Tx - T \sum_{j=1}^n c_j(x)\varphi_j\| \leq \|T\| \|x - \sum_{j=1}^n c_j(x) \varphi_j \|_\X \leq \epsilon
\]
hence \(\{T \varphi_j\}_{j=1}^\infty \subset \Y\) and \(\{c_j(T^{-1} \cdot)\}_{j=1}^\infty \subset \Y^*\) form a basis for \(\Y\)
by linearity and continuity of \(T\) and \(T^{-1}\).
\end{proof}

\begin{lemma}
\label{lemma:ap_ap}
Let \(\X\) be a Banach space with the \emph{AP} and \(\Y\) be any Banach space. Suppose there exists a continuous linear bijection \(T : \X \to \Y\).
Then \(\Y\) has the \emph{AP}. 
\end{lemma}
\begin{proof}
Let \(K \subset \Y\) be a compact set and \(\epsilon > 0\). The set \(R = T^{-1}(K) \subset \X\) is compact since \(T^{-1}\)
is continuous. Since \(\X\) has the AP, there exists a finite rank operator \(U : \X \to \X\) such that
\[\|x - Ux\|_\X \leq \frac{\epsilon}{\|T\|}, \qquad \forall x \in R.\]
Define the operator \(W : \Y \to \Y\) by \(W = T U T^{-1}\). Clearly \(W\) is a finite rank operator since \(U\) is a finite rank operator. Let \(y \in K\) then,
since \(K = T(R)\), there exists \(x \in R\) such that \(Tx = y\) and \(x = T^{-1}y\).
Then
\[
\|y - Wy\|_\Y = \|Tx - TUx\|_\Y \leq \|T\| \|x - Ux\|_\X \leq \epsilon.
\]
hence \(\Y\) has the AP.
\end{proof}

The following lemma shows than the infinite union of compact sets is compact if each set is the image of a fixed compact set under a convergent sequence of continuous maps. The result is instrumental in proving Lemma~\ref{lemma:finitedim_approx}.

\begin{lemma}
\label{lemma:compact_union}
Let \(\X,\Y\) be Banach spaces and \(F : \X \to \Y\) be a continuous map. Let \(K \subset \X\)
be a compact set in $\X$ and \(\{F_n : \X \to \Y\}_{n=1}^\infty\) be a sequence of continuous maps such that
\[\lim_{n \to \infty} \sup_{x \in K} \|F(x) - F_n(x)\|_\Y = 0.\]
Then the set \[W \coloneqq \bigcup_{n=1}^\infty F_n (K) \cup F(K)\] 
is compact in $\Y$.
\end{lemma}
\begin{proof}
Let \(\epsilon > 0\) then there exists a number \(N = N(\epsilon) \in \N\) such that
\[\sup_{x \in K} \|F(x) - F_n(x)\|_\Y \leq \frac{\epsilon}{2}, \qquad \qquad \forall n \geq N.\]
Define the set
\[W_N = \bigcup_{n=1}^N F_n(K) \cup F(K)\]
which is compact since \(F\) and each \(F_n\) are continuous. We can therefore find a number \(J = J(\epsilon, N) \in \N\) 
and elements \(y_1,\dots,y_J \in W_N\) such that, for any \(z \in W_N\), there exists a number \(l = l(z) \in \{1,\dots,J\}\)
such that 
\[\|z - y_{l}\|_\Y \leq \frac{\epsilon}{2}.\]
Let \(y \in W \setminus W_N\) then there exists a number \(m > N\) and an element \(x \in K\) such that \(y = F_m (x)\).
Since \(F (x) \in W_N\), we can find a number \(l \in \{1,\dots,J\}\) such that
\[\|F (x) - y_l\|_\Y \leq \frac{\epsilon}{2}.\]
Therefore,
\[\|y - y_l\|_\Y \leq \|F_m (x) - F (x) \|_\Y + \|F (x) - y_l\|_\Y \leq \epsilon\]
hence \(\{y_j\}_{j=1}^J\) forms a finite \(\epsilon\)-net for \(W\), showing that \(W\) is totally bounded.

We will now show that \(W\) is closed. To that end, let \(\{p_n\}_{n=1}^\infty\) be a convergent sequence in \(W\), in particular,
\(p_n \in W\) for every \(n \in \N\) and \(p_n \to p \in \Y\) as \(n \to \infty\). We can thus find convergent sequences \(\{x_n\}_{n=1}^\infty\)
and \(\{\alpha_n\}_{n=1}^\infty\) such that \(x_n \in K\), \(\alpha_n \in \N_0\), and 
\(p_n = F_{\alpha_n}(x_n)\) where we define \(F_0 \coloneqq F\). Since \(K\) is closed, \(\lim\limits_{n \to \infty} x_n = x \in K\) thus, for each fixed \(n \in \N\),
\[\lim_{j \to \infty} F_{\alpha_n}(x_j) = F_{\alpha_n}(x) \in W\]
by continuity of \(F_{\alpha_n}\). Since uniform convergence implies point-wise convergence 
\[p = \lim\limits_{n \to \infty} F_{\alpha_n}(x) = F_{\alpha}(x) \in W\]
for some \(\alpha \in \N_0\) thus \(p \in W\), showing that \(W\) is closed.
\end{proof}

The following lemma shows that any continuous operator acting between two Banach spaces with the AP can be approximated in a finite-dimensional manner. The approximation proceeds in three steps which are shown schematically in Figure~\ref{fig:approach}. First an input is mapped to a finite-dimensional representation via the action of a set of functionals on \(\X\). This representation is then mapped by a continuous function to a new finite-dimensional representation which serves as the set of coefficients onto representers of \(\Y\). The resulting expansion is an element of \(\Y\) that is \(\epsilon\)-close to the action of \(\G\) on the input element. A similar finite-dimensionalization was used in \citep{Kovachki} by using PCA on \(\X\) to define the functionals acting on the input and PCA on \(\Y\) to define the output representers. However the result in that work is restricted to separable Hilbert spaces; here, we generalize it to Banach spaces with the AP.

\begin{lemma}
\label{lemma:finitedim_approx}
Let \(\X, \Y\) be two Banach spaces with the AP and let \(\G : \mathcal{X} \to \mathcal{Y}\) be a
continuous map. For every compact set \(K \subset \mathcal{X}\) and \(\epsilon > 0\), there exist numbers \(J,J' \in \N\)
and continuous linear maps \(F_J : \X \to \R^{J}\), \(G_{J'} : \R^{J'} \to \Y\) as well as  \(\varphi \in C(\R^{J};\R^{J'})\)
such that  
\[\sup_{x \in K} \|\G(x) - (G_{J'} \circ \varphi \circ F_J) (x) \|_{\Y} \leq \epsilon.\]
Furthermore there exist \(w_1,\dots,w_J \in \X^*\) such that \(F_J\) has the form
\[F_J(x) = \bigl ( w_1(x), \dots, w_{J}(x) \bigl ), \qquad \forall x \in \X\]
and there exist \(\beta_1,\dots,\beta_{J'} \in \Y\) such that \(G_{J'}\) has the form
\[G_{J'}(v) = \sum_{j=1}^{J'} v_j \beta_j, \qquad \forall v \in \R^{J'}.\]
If \(\Y\) admits a basis then \(\{\beta_j\}_{j=1}^{J'}\)
can be picked so that there is an extension \(\{\beta_j\}_{j=1}^\infty \subset \Y\) which is a basis for \(\Y\).
\end{lemma}
\begin{proof}
Since \(\X\) has the AP, there exists a sequence of finite rank operators \(\{U^{\X}_n : \X \to \X\}_{n=1}^\infty\) such that
\[\lim_{n \to \infty} \sup_{x \in K} \|x - U^{\X}_n x\|_\X = 0.\]
Define the set 
\[Z = \bigcup_{n=1}^\infty U^{\X}_n(K) \cup K\]
which is compact by Lemma~\ref{lemma:compact_union}. Therefore, \(\mathcal{G}\) is uniformly continuous on \(Z\) hence 
there exists a modulus of continuity \(\omega : \R_{\geq 0} \to \R_{\geq 0}\) which is non-decreasing 
and satisfies \(\omega(t) \to \omega(0) = 0\) as \(t \to 0\) as well as 
\[\|\mathcal{G}(z_1) - \mathcal{G}(z_2)\|_{\mathcal{Y}} \leq \omega \bigl( \|z_1 - z_2\|_\mathcal{X} \bigl ) \qquad \forall z_1, z_2 \in Z.\]
We can thus find, a number \(N = N(\epsilon) \in \N\) such that
\[\sup_{x \in K} \omega \bigl ( \|x - U^{\X}_N x \|_{\X} \bigl ) \leq \frac{\epsilon}{2}. \]
Let \(J = \text{dim } U^{\X}_N (\X) < \infty\). There exist elements \(\{\alpha_j\}_{j=1}^J \subset \X\) and \(\{w_j\}_{j=1}^J \subset \X^*\)
such that
\[U^{\X}_N x = \sum_{j=1}^J w_j(x) \alpha_j, \qquad \forall x \in X.\]
Define the maps \(F^{\X}_J : \X \to \R^{J}\) and \(G^{\X}_J : \R^{J} \to \X\) by
\begin{align*}
F^{\X}_J(x) &= (w_1(x), \hdots, w_J(x)), \qquad \forall  x \in \X, \\
G^{\X}_J (v) &= \sum_{j=1}^J v_j \alpha_j, \qquad \qquad \qquad \:\:\:  \forall v \in \R^J,
\end{align*}
noting that \(U^{\X}_N = G^{\X}_J \circ F^{\X}_J\). Define the set \(W = (\G \circ U^{\X}_N)(K) \subseteq \Y\) which is clearly compact.
Since \(\Y\) has the AP, we can similarly find a finite rank operator \(U^{\Y}_{J'} : \Y \to \Y\) with \(J' = \text{dim } U^{\Y}_{J'} (\Y) < \infty\)
such that
\[\sup_{y \in W} \|y - U^{\Y}_{J'} y \|_{\Y} \leq \frac{\epsilon}{2}.\]
Analogously, define the maps \(F^{\Y}_{J'} : \Y \to \R^{J'}\) and \(G^{\Y}_{J'} : \R^{J'} \to \Y\) by 
\begin{align*}
F^{\Y}_{J'}(y) &= (q_1(y), \hdots, q_{J'}(y)), \qquad \forall y \in \Y, \\
G^{\Y}_{J'} (v) &= \sum_{j=1}^{J'} v_j \beta_j, \qquad \qquad \qquad \:\:  \forall v \in \R^{J'}
\end{align*}
for some \(\{\beta_j\}_{j=1}^{J'} \subset \Y\) and \(\{q_j\}_{j=1}^{J'} \subset \Y^*\) such that \(U^{\Y}_{J'} = G^{\Y}_{J'} \circ F^{\Y}_{J'}\).
Clearly if \(\Y\) admits a basis then we could have defined \(F^{\Y}_{J'}\) and \(G^{\Y}_{J'}\) through it instead 
of through \(U^{\Y}_{J'}\).
Define \(\varphi : \R^{J} \to \R^{J'}\) by
\[\varphi(v) = (F^{\Y}_{J'} \circ \G \circ G^{\X}_{J} )(v), \qquad \forall v \in \R^{J}\]
which is clearly continuous and note that \(G^{\Y}_{J'} \circ \varphi \circ F^{\X}_J = U^{\Y}_{J'} \circ \G \circ U^{\X}_N\).
Set \(F_J = F^{\X}_J\) and \(G_{J'} = G^{\Y}_{J'}\) then, for any \(x \in K\),
\begin{align*}
\|\G(x) - (G_{J'} \circ \varphi \circ F_J) (x) \|_{\Y} &\leq \|\G(x) - \G(U^{\X}_N x)\|_\Y + \|\G(U^{\X}_N x) - (U^{\Y}_{J'} \circ \G \circ U^{\X}_N) (x) \|_\Y \\
&\leq \omega \bigl (\|x - U^{\X}_N x\|_\X \bigl ) + \sup_{y \in W} \|y - U^{\Y}_{J'} y \|_{\Y} \\
&\leq \epsilon
\end{align*}
as desired.
\end{proof}

We now state and prove some results about isomorphisms of function spaces defined on different domains. These results are instrumental in proving Lemma~\ref{lemma:ap}.

\begin{lemma}
\label{lemma:cm_isomorphism}
Let \(D, D' \subset \R^d\) be domains. Suppose that, for 
some \(m \in \N_0\), there exists a \(C^m\)-diffeomorphism \(\tau : \bar{D}' \to \bar{D}\).
Then the mapping \(T: C^m (\bar{D}) \to C^m (\bar{D}')\) defined as
\[T(f)(x) = f(\tau(x)), \qquad \forall f \in C^m (\bar{D}), \:\: x \in \bar{D}'\]
is a continuous linear bijection.
\end{lemma}

\begin{proof}
Clearly \(T\) is linear since the evaluation functional is linear. To see that it is continuous, 
note that by the chain rule we can find a constant \(Q = Q(m) > 0\) such that
\[\|T(f)\|_{C^m} \leq Q \|\tau\|_{C^m} \|f\|_{C^m}, \qquad \forall f \in C^m(\bar{D}).\]
We will now show that it is bijective. Let \(f,g \in C^m (\bar{D})\) so that \(f \neq g\).
Then there exists a point \(x \in \bar{D}\) such that \(f(x) \neq g(x)\). Then 
\(T(f)(\tau^{-1}(x)) = f(x)\) and \(T(g)(\tau^{-1}(x)) = g(x)\) hence \(T(f) \neq T(g)\) thus
\(T\) is injective.
Now let \(g \in C^m (\bar{D}')\) and define \(f : \bar{D} \to \R\) by \(f = g \circ \tau^{-1}\). 
Since \(\tau^{-1} \in C^m (\bar{D};\bar{D}')\), we have that \(f \in C^m (\bar{D})\). Clearly, 
\(T(f) = g\) hence \(T\) is surjective.
\end{proof}

\begin{corollary}
\label{corr:cm_isomorphism}
Let \(M > 0\) and \(m \in \N_0\). There exists a continuous linear bijection \(T : C^m([0,1]^d) \to C^m([-M,M]^d)\).
\end{corollary}
\begin{proof}
Let $\one \in \R^d$ denote the vector in which all entries are $1$. 
Define the map \(\tau : \R^d \to \R^d\) by
\begin{equation}
\label{eq:tau}
\tau(x) = \frac{1}{2M} x + \frac{1}{2}\one, \qquad \forall x \in \R^d.
\end{equation}
Clearly \(\tau\) is a \(C^\infty\)-diffeomorphism between \([-M,M]^d\) and \([0,1]^d\)
hence Lemma~\ref{lemma:cm_isomorphism} implies the result.
\end{proof}

\begin{lemma}
\label{lemma:w1_isomorphism}
Let \(M > 0\) and \(m \in \N\). There exists a continuous linear bijection \(T : W^{m,1}((0,1)^d) \to W^{m,1}((-M,M)^d)\).
\end{lemma}
\begin{proof}
Define the map \(\tau : \R^d \to \R^d\) by \eqref{eq:tau}.
We have that \(\tau((-M,M)^d) = (0,1)^d\). Define the operator \(T\) by
\[Tf = f \circ \tau, \qquad \forall f \in W^{m,1}((0,1)^d).\]
which is clearly linear since composition is linear. We compute that, for any \(0 \leq |\alpha|_1 \leq m\),
\[\partial^\alpha (f \circ \tau) = (2M)^{-|\alpha|_1} (\partial^\alpha f) \circ \tau\]
hence, by the change of variables formula, 
\[\|Tf\|_{W^{m,1}((-M,M)^d)} = \sum_{0 \leq |\alpha|_1 \leq m} (2M)^{d - |\alpha|_1} \|\partial^\alpha f\|_{L^1((0,1)^d)}. \]
We can therefore find numbers \(C_1,C_2 > 0\), depending on \(M\) and \(m\), such that
\[C_1 \|f\|_{W^{m,1}((0,1)^d)} \leq \|Tf\|_{W^{m,1}((-M,M)^d)} \leq C_2 \|f\|_{W^{m,1}((0,1)^d)}.\]
This shows that \(T : W^{m,1}((0,1)^d) \to W^{m,1}((-M,M)^d)\) is continuous and injective.
Now let \(g \in W^{m,1}((-M,M)^d)\) and define \(f = g \circ \tau^{-1}\). A similar argument shows that
\(f \in W^{m,1}((0,1)^d)\) and, clearly, \(Tf = g\) hence \(T\) is surjective.
\end{proof}

We now show that the spaces in Assumptions \ref{assump:input} and \ref{assump:output} have the AP. While the result is well-known when the domain is \((0,1)^d\) or \(\R^d\), we were unable to find any results in the literature for Lipschitz domains and we therefore give a full proof here. The essence of the proof is to either exhibit an isomorphism to a space that is already known to have AP or to directly show AP by embedding the Lipschitz domain into an hypercube for which there are known basis constructions. Our proof shows the stronger result that \(W^{m,p}(D)\) for \(m \in \N_0\) and \(1 \leq p < \infty\) has a basis, but, for \(C^m (\bar{D})\), we only establish the AP and not necessarily a basis. The discrepancy comes from the fact that there is an isomorphism between \(W^{m,p}(D)\) and \(W^{m,p}(\R^d)\)\ while there is not one between \(C^m(\bar{D})\) and \(C^m(\R^d)\).

\begin{lemma}
\label{lemma:ap}
Let Assumptions \ref{assump:input} and \ref{assump:output} hold. Then \(\A\) and \(\U\) have the \emph{AP}.
\end{lemma}

\begin{proof}
It is enough to show that the spaces \(W^{m,p}(D)\), and \(C^m (\bar{D})\) for 
any \(1 \leq p < \infty\) and \(m \in \N_0\) with \(D \subset \R^d\) a Lipschitz domain have the AP.
Consider first the spaces \(W^{0,p} (D) = L^p(D)\). Since the Lebesgue measure on \(D\) is 
\(\sigma\)-finite and has no atoms, \(L^p(D)\) is isometrically isomorphic 
to \(L^p ((0,1))\) (see, for example, \cite[Chapter 6]{albiac2006topics}). Hence by
Lemma \ref{lemma:ap_ap}, it is enough to show that \(L^p((0,1))\) has the AP. Similarly, consider the spaces \(W^{m,p}(D)\) for \(m > 0\) and \(p > 1\). 
Since \(D\) is Lipschitz, 
there exists a continuous linear operator \(W^{m,p} (D) \to W^{m,p} (\R^d)\) \cite[Chapter 6, Theorem 5]{stein1970singular}
(this also holds for \(p=1\)). We can therefore apply \cite[Corollary 4]{pelczynski2001contribution} (when \(p > 1\))
to conclude that \(W^{m,p}(D)\) is isomorphic to \(L^p((0,1))\). By \cite[Proposition 6.1.3]{albiac2006topics},
\(L^p((0,1))\) has a basis hence Lemma~ \ref{lemma:schauder_ap} implies the result.

Now, consider the spaces \(C^m(\bar{D})\). Since \(D\) is bounded, there exists a number \(M > 0\)
such that \(\bar{D} \subseteq [-M,M]^d\).
Hence, by Corollary~\ref{corr:cm_isomorphism}, \(C^m ([0,1]^d)\) is isomorphic to \(C^m ([-M,M]^d)\).
Since \(C^m ([0,1]^d)\) has a basis \cite[Theorem 5]{ciesielski1972construction}, Lemma~\ref{lemma:schauder_schauder} then implies that
\(C^m ([-M,M]^d)\) has a basis. By \cite[Theorem 1]{fefferman2007cmextension}, 
there exists a continuous linear operator \(E : C^m (\bar{D}) \to C^m_{\text{b}}(\R^d)\)
such that \(E(f)|_{\bar{D}} = f\) for all \(f \in C(\bar{D})\).
Define the restriction operators \(R_M : C^m_{\text{b}}(\R^d) \to C^m([-M,M]^d)\) and \(R_D : C^m([-M,M]^d) \to C^m(\bar{D})\)
which are both clearly linear and continuous and \(\|R_M\| = \|R_D\| = 1\). Let \(\{c_j\}_{j=1}^\infty \subset \bigl ( C^m([-M,M]^d) \bigl )^*\)
and \(\{\varphi_j\}_{j=1}^\infty  \subset C^m([-M,M]^d)\) be a basis for \(C^m([-M,M]^d)\). 
As in the proof of Lemma~\ref{lemma:schauder_ap}, there exists a constant \(C_1 > 0\)
such that, for any \(n \in \N\) and \(f \in C^m([-M,M]^d) \), 
\[\|\sum_{j=1}^n c_j(f) \varphi_j\|_{C^m([-M,M]^d)} \leq C_1 \|f\|_{C^m([-M,M]^d)}.\]
Suppose, without loss of generality, that \(C_1 \|E\| \geq 1\). Let \(K \subset C^m(\bar{D})\)
be a compact set and \(\epsilon > 0\). Since \(K\) is compact, we can find a number \(n = n(\epsilon) \in \N\)
and elements \(y_1,\dots,y_n \in K\) such that, for any \(f \in K\) there exists a number \(l \in \{1,\dots,n\}\)
such that
\[\|f - y_l\|_{C^m(\bar{D})} \leq \frac{\epsilon}{3 C_1 \|E\|}.\]
For every \(l \in \{1,\dots,n\}\), define \(g_l = R_M (E(y_l))\) and note that \(g_l \in C^m([-M,M]^d)\)
hence there exists a number \(J = J(\epsilon,n) \in \N\) such that
\[\max_{l \in \{1,\dots,n\}} \|g_l - \sum_{j=1}^J c_j(g_l) \varphi_j \|_{C^m([-M,M]^d)} \leq \frac{\epsilon}{3}.\]
Notice that, since \(y_l = R_D(g_l)\), we have
\[ \max_{l \in \{1,\dots,n\}} \|y_l - \sum_{j=1}^J c_j \bigl ( R_M (E(y_l)) \bigl ) R_D(\varphi_j)\|_{C^m(\bar{D})} \leq \|R_D\| \max_{l \in \{1,\dots,n\}}  \|g_l - \sum_{j=1}^J c_j(g_l) \varphi_j \|_{C^m([-M,M]^d)} \leq \frac{\epsilon}{3}.\]
Define the finite rank operator \(U : C^m(\bar{D}) \to C^m(\bar{D})\) by
\[Uf = \sum_{j=1}^J c_j \bigl ( R_M (E(f)) \bigl ) R_D(\varphi_j), \qquad \forall f \in C^m(\bar{D}).\]
We then have that, for any \(f \in K\),
\begin{align*}
\|f - Uf\|_{C^m(\bar{D})} &\leq \|f - y_l\|_{C^m(\bar{D})} + \|y_l - Uy_l\|_{C^m(\bar{D})} + \|U y_l - Uf\|_{C^m(\bar{D})} \\
&\leq \frac{2\epsilon}{3} + \|\sum_{j=1}^J c_j \bigl( R_M(E(y_l - f)) \bigl ) \varphi_j \|_{C^m([-M,M]^d)} \\
&\leq \frac{2\epsilon}{3} + C_1 \|R_M(E(y_l - f))\|_{C^m([-M,M]^d)} \\
&\leq \frac{2\epsilon}{3} + C_1 \|E\| \|y_l - f\|_{C^m(\bar{D})} \\
&\leq \epsilon
\end{align*}
hence \(C^m(\bar{D})\) has the AP.

We are left with the case \(W^{m,1}(D)\). A similar argument as for the \(C^m(\bar{D})\) case holds. In
particular the basis from \cite[Theorem 5]{ciesielski1972construction} is also a basis for \(W^{m,1}((0,1)^d)\).
Lemma~\ref{lemma:w1_isomorphism} gives an isomorphism between \(W^{m,1}((0,1)^d)\) and \(W^{m,1}((-M,M)^d)\) hence 
we may use the extension operator \(W^{m,1} (D) \to W^{m,1} (\R^d)\) from \cite[Chapter 6, Theorem 5]{stein1970singular}
to complete the argument. In fact, the same construction yields a basis for \(W^{m,1} (D)\) due 
to the isomorphism with \(W^{m,1} (\R^d)\), see, for example \cite[Theorem 1]{pelczynski2001contribution}.
\end{proof}

\section{}
\label{sec:appendix_functionalapprox}

In this section, we prove various results about the approximation of linear functionals by kernel integral operators. Lemma~\ref{lemma:reisz} establishes a Riesz-representation theorem for \(C^m\). The proof proceeds exactly as in the well-known result for \(W^{m,p}\) but, since we did not find it in the literature, we give full details here. Lemma~\ref{lemma:wmp_kernelapprox} shows that linear functionals on \(W^{m,p}\) can be approximated uniformly over compact set by integral kernel operators with a \(C^\infty\) kernel. Lemmas~\ref{lemma:c_kernelapprox} and \ref{lemma:cm_kernelapprox} establish similar results for \(C\) and \(C^m\) respectively by employing Lemma~\ref{lemma:reisz}. These lemmas are crucial in showing that NO(s) are universal since they imply that the functionals from Lemma~\ref{lemma:finitedim_approx} can be approximated by elements of \(\mathsf{IO}\).

\begin{lemma}
\label{lemma:reisz}
Let \(D \subset \R^d\) be a domain and \(m \in \N_0\). For every \(L \in \bigl (C^m (\bar{D}) \bigl )^*\) 
there exist finite, signed, Radon measures \(\{\lambda_\alpha\}_{0 \leq |\alpha|_1 \leq m}\) such that
\[L(f) = \sum_{0 \leq |\alpha|_1 \leq m} \int_{\bar{D}} \partial^\alpha f \: \mathsf{d}\lambda_\alpha, \qquad \forall f \in C^m(\bar{D}).\] 
\end{lemma}
\begin{proof}
The case \(m = 0\) follow directly from \cite[Theorem B.111]{leoni2009first}, so we assume that \(m > 0\).
Let \(\alpha_1,\dots,\alpha_J\) be an enumeration of the set \(\{\alpha \in \N^d : |\alpha|_1 \leq m\}\).
Define the mapping \(T : C^m(\bar{D}) \to C(\bar{D};\R^J)\) by
\[Tf = \bigl ( \partial^{\alpha_0} f, \dots, \partial^{\alpha_J}f ), \qquad \forall f \in C^m(\bar{D}).\]
Clearly \(\|Tf\|_{C(\bar{D};\R^J)} = \|f\|_{C^m(\bar{D})}\) hence \(T\) is an injective, continuous linear operator.
Define \(W \coloneqq T(C^m(\bar{D})) \subset C(\bar{D};\R^J)\) then \(T^{-1} : W \to C^m (\bar{D})\) is a continuous linear 
operator since \(T\) preserves norm. Thus \(W = \bigl (T^{-1} \bigl)^{-1}(C^m(\bar{D}))\) is closed as the pre-image of a 
closed set under a continuous map. In particular, \(W\) is a Banach space since \(C(\bar{D};\R^J)\) is a Banach space
and \(T\) is an isometric isomorphism between \(C^m(\bar{D})\) and \(W\). Therefore, there exists a 
continuous linear functional \(\tilde{L} \in W^*\) such that 
\[L(f) = \tilde{L}(Tf), \qquad \forall f \in C^m (\bar{D}).\]
By the Hahn-Banach theorem, \(\tilde{L}\) can be extended to a continuous linear functional \(\bar{L} \in \bigl ( C(\bar{D};\R^J) \bigl)^*\)
such that \(\|L\|_{(C^m(\bar{D}))^*} = \|\tilde{L}\|_{W^*} = \|\bar{L}\|_{(C(\bar{D};\R^J))^*}\). We have that
\[L(f) = \tilde{L}(Tf) = \bar{L}(Tf), \qquad \forall f \in C^m (\bar{D}).\]
Since 
\[\bigl ( C(\bar{D};\R^J) \bigl )^* \cong \bigtimes_{j=1}^J \bigl ( C(\bar{D}) \bigl )^* \cong \bigoplus_{j=1}^J \bigl ( C(\bar{D}) \bigl )^*,\]
we have, by applying \cite[Theorem B.111]{leoni2009first} \(J\) times, that there exist finite, signed, Radon measures \(\{\lambda_\alpha\}_{0 \leq |\alpha|_1 \leq m}\) 
such that
\[\bar{L}(Tf) = \sum_{0 \leq |\alpha|_1 \leq m} \int_{\bar{D}} \partial^\alpha f \: \mathsf{d}\lambda_\alpha, \qquad \forall f \in C^m(\bar{D})\]
as desired.
\end{proof}

\begin{lemma}
\label{lemma:wmp_kernelapprox}
Let \(D \subset \R^d\) be a bounded, open set and \(L \in (W^{m,p}(D))^*\) for some \(m \geq 0\)
and \(1 \leq p < \infty\). For any closed and bounded set \(K \subset W^{m,p}(D)\) (compact if \(p=1\)) and \(\epsilon > 0\),
there exists a function \(\kappa \in C^\infty_c (D)\) such that
\[\sup_{u \in K} |L(u) - \int_D \kappa  u \dx| < \epsilon.\]
\end{lemma}
\begin{proof}
First consider the case \(m = 0\) and \(1 \leq p < \infty\). By the Riesz Representation Theorem \cite[Appendix B]{conway2007acourse},
there exists a function \(v \in L^q(D)\) such that
\[L(u) = \int_D v u \dx.\]
Since \(K\) is bounded, there is a constant \(M > 0\) such that
\[\sup_{u \in K} \|u\|_{L^p} \leq M.\]

Suppose \(p > 1\), so that \(1 < q < \infty\). Density of \(C^\infty_c(D)\) in \(L^q (D)\) \cite[Corollary 2.30]{adams2003sobolev}  implies there exists 
a function \(\kappa \in C^\infty_c(D)\) such that
\[\|v - \kappa\|_{L^q} < \frac{\epsilon}{M}.\]
By the H{\"o}lder inequality,
\[|L(u) - \int_D \kappa u \dx| \leq \|u\|_{L^p} \|v - \kappa\|_{L^q} < \epsilon.\]

Suppose that \(p=1\) then \(q=\infty\). Since \(K\) is totally bounded, there exists a number
\(n \in \N\) and functions \(g_1,\dots,g_n \in K\) such that, for any \(u \in K\),
\[\|u - g_l\|_{L^1} < \frac{\epsilon}{3 \|v\|_{L^\infty}}\]
for some \(l \in \{1,\dots,n\}\). Let \(\psi_\eta \in C^\infty_c (D)\) denote a standard mollifier
for any \(\eta > 0\). We can find \(\eta > 0\) small enough such that
\[\max_{l \in \{1,\dots,n\}} \|\psi_\eta * g_l - g_l\|_{L^1} < \frac{\epsilon}{9 \|v\|_{L^\infty}}\]
Define \(f = \psi_\eta * v \in C(D)\) and note that \(\|f\|_{L^\infty} \leq \|v\|_{L^\infty}\). By Fubini's theorem, we find
\[|\int_D (f - v)g_l \dx| = \int_D v ( \psi_\eta * g_l - g_l) \dx \leq \|v\|_{L^\infty} \|\psi_\eta * g_l - g_l\|_{L^1} < \frac{\epsilon}{9}.\]
Since \(g_l \in L^1 (D)\), by Lusin's theorem, we can find a compact set \(A \subset D\) such that
\[\max_{l \in \{1,\dots,n\}}  \int_{D \setminus A} |g_l| \dx < \frac{\epsilon}{18 \|v\|_{L^\infty}}\]
Since \(C^\infty_c (D)\) is dense in \(C(D)\) over compact sets \cite[Theorem C.16]{leoni2009first}, we can find a function \(\kappa \in C^\infty_c (D)\) such that
\[\sup_{x \in A} |\kappa(x) - f(x)| \leq \frac{\epsilon}{9M}\]
and \(\|\kappa\|_{L^\infty} \leq \|f\|_{L^\infty} \leq \|v\|_{L^\infty}\).
We have,
\begin{align*}
|\int_D (\kappa - v) g_l \dx| &\leq \int_A |(\kappa - v)g_l| \dx + \int_{D \setminus A} |(\kappa - v)  g_l | \dx\\
&\leq \int_A |(\kappa - f)g_l| \dx + \int_D |(f-v)g_l| \dx + 2 \|v\|_{L^\infty} \int_{D \setminus A} |g_l| \dx \\
&\leq \sup_{x \in A} |\kappa(x) - f(x)| \|g_l\|_{L^1} + \frac{2\epsilon}{9}   \\ 
&< \frac{\epsilon}{3}.
\end{align*}
Finally,
\begin{align*}
|L(u) - \int_{D} \kappa u \dx| &\leq |\int_D vu \dx - \int_D vg_l \dx| + |\int_D v g_l \dx - \int_D \kappa u \dx| \\
&\leq \|v\|_{L^\infty} \|u- g_l\|_{L^1} + |\int_D \kappa u \dx - \int_D \kappa g_l \dx | + | \int_D \kappa g_l \dx - \int_D v g_l \dx| \\
&\leq \frac{\epsilon}{3} + \|\kappa\|_{L^\infty} \|u - g_l\|_{L^1} + |\int_D (\kappa - v) g_l  \dx| \\
&\leq \frac{2\epsilon}{3} + \|v\|_{L^\infty} \|u-g_l\|_{L^1} \\
&< \epsilon.
\end{align*}

Suppose \(m \geq 1\). By the Riesz Representation Theorem \cite[Theorem 3.9]{adams2003sobolev}, there exist elements \((v_\alpha)_{0\leq |\alpha|_1 \leq m}\) of 
\(L^q (D)\) where \(\alpha \in \N^d\) is a multi-index such that
\[L(u) = \sum_{0 \leq |\alpha|_1 \leq m} \int_D v_\alpha \partial_\alpha u  \dx.\]
Since \(K\) is bounded, there is a constant \(M > 0\) such that
\[\sup_{u \in K} \|u\|_{W^{m,p}} \leq M.\]

Suppose \(p > 1\), so that \(1 < q < \infty\). Density of \(C^\infty_0(D)\) in \(L^q (D)\) implies there exist
functions \((f_\alpha)_{0 \leq |\alpha|_1 \leq m}\) in \(C^\infty_c(D)\) such that 
\[\|f_\alpha - v_\alpha\|_{L^q} < \frac{\epsilon}{MJ}\]
where \(J = |\{\alpha \in \N^d : |\alpha|_1 \leq m \}|\). Let
\[\kappa = \sum_{0 \leq |\alpha|_1 \leq m} (-1)^{|\alpha|_1} \partial_\alpha f_\alpha\]
then, by definition of a weak derivative, 
\[\int_D \kappa u \dx = \sum_{0 \leq |\alpha|_1 \leq m} (-1)^{|\alpha|_1} \int_D \partial_\alpha f_\alpha  u \dx = \sum_{0 \leq |\alpha|_1 \leq m} \int_D f_\alpha \partial_\alpha u \dx.  \]
By the H{\"o}lder inequality,
\[|L(u) - \int_D \kappa u \dx | \leq \sum_{0 \leq |\alpha|_1 \leq m} \|\partial_\alpha u\|_{L^p} \|f_\alpha - v_\alpha\|_{L^q} < M \sum_{0 \leq |\alpha|_1 \leq m} \frac{\epsilon}{MJ} = \epsilon.\]

Suppose that \(p=1\) then \(q=\infty\). Define the constant \(C_v > 0\) by
\[C_v = \sum_{0 \leq |\alpha|_1 \leq m} \|v_\alpha\|_{L^\infty}.\]
Since \(K\) is totally bounded, there exists a number
\(n \in \N\) and functions \(g_1,\dots,g_n \in K\) such that, for any \(u \in K\),
\[\|u - g_l\|_{W^{m,1}} < \frac{\epsilon}{3C_v}\]
for some \(l \in \{1,\dots,n\}\). 
Let \(\psi_\eta \in C^\infty_c (D)\) denote a standard mollifier
for any \(\eta > 0\). We can find \(\eta > 0\) small enough such that
\[\max_{\alpha} \max_{l \in \{1,\dots,n\}} \|\psi_\eta * \partial_\alpha g_l - \partial_\alpha g_l\|_{L^1} < \frac{\epsilon}{9 C_v}.\]
Define \(f_\alpha = \psi_\eta * v_\alpha \in C(D)\) and note that \(\|f_\alpha\|_{L^\infty} \leq \|v_\alpha\|_{L^\infty}\). By Fubini's theorem, we find
\begin{align*}
\sum_{0 \leq |\alpha|_1 \leq m} |\int_D (f_\alpha - v_\alpha) \partial_\alpha g_l \dx| &= \sum_{0 \leq |\alpha|_1 \leq m} |\int_D v_\alpha ( \psi_\eta * \partial_\alpha g_l - \partial_\alpha g_l) \dx| \\
&\leq \sum_{0 \leq |\alpha|_1 \leq m} \|v_\alpha\|_{L^\infty} \|\psi_\eta * \partial_\alpha g_l - \partial_\alpha g_l\|_{L^1} \\
&< \frac{\epsilon}{9}.
\end{align*}
Since \(\partial_\alpha g_l \in L^1 (D)\), by Lusin's theorem, we can find a compact set \(A \subset D\) such that
\[\max_\alpha \max_{l \in \{1,\dots,n\}}  \int_{D \setminus A} |\partial_\alpha g_l| \dx < \frac{\epsilon}{18 C_v}.\]
Since \(C^\infty_c (D)\) is dense in \(C(D)\) over compact sets, we can find functions \(w_\alpha \in C^\infty_c (D)\) such that
\[\sup_{x \in A} |w_\alpha (x) - f_\alpha (x)| \leq \frac{\epsilon}{9M J}\]
where \(J = |\{\alpha \in \N^d : |\alpha|_1 \leq m \}|\) and \(\|w_\alpha\|_{L^\infty} \leq \|f_\alpha\|_{L^\infty} \leq \|v_\alpha\|_{L^\infty}\).
We have,
\begin{align*}
\sum_{0 \leq |\alpha|_1 \leq m} \int_D |(w_\alpha - v_\alpha) \partial_\alpha g_l| &= \sum_{0 \leq |\alpha|_1 \leq m} \left (  \int_A |(w_\alpha - v_\alpha) \partial_\alpha g_l| dx + 
\int_{D \setminus A} |(w_\alpha - v_\alpha) \partial_\alpha g_l| dx \right )\\
&\leq \sum_{0 \leq |\alpha|_1 \leq m} \bigg ( \int_A |(w_\alpha - f_\alpha) \partial_\alpha g_l| \dx + \int_D |(f_\alpha -v_\alpha) \partial_\alpha g_l| \dx \\
&\quad \: + 2 \|v_\alpha\|_{L^\infty} \int_{D \setminus A} |\partial_\alpha g_l| \dx \bigg ) \\
&\leq \sum_{0 \leq |\alpha|_1 \leq m} \sup_{x \in A} |w_\alpha (x) - f_\alpha (x)| \|\partial_\alpha g_l\|_{L^1} + \frac{2\epsilon}{9}   \\ 
&< \frac{\epsilon}{3}.
\end{align*}
Let
\[\kappa = \sum_{0 \leq |\alpha|_1 \leq m} (-1)^{|\alpha|_1} \partial_\alpha w_\alpha.\]
then, by definition of a weak derivative, 
\[\int_D \kappa u \dx = \sum_{0 \leq |\alpha|_1 \leq m} (-1)^{|\alpha|_1} \int_D \partial_\alpha w_\alpha  u \dx = \sum_{0 \leq |\alpha|_1 \leq m} \int_D w_\alpha \partial_\alpha u \dx.  \]
Finally,
\begin{align*}
|L(u) - \int_{D} \kappa u \dx| &\leq \sum_{0 \leq |\alpha|_1 \leq m} \int_D | v_\alpha \partial_\alpha u - w_\alpha \partial_\alpha u| \dx \\
&\leq \sum_{0 \leq |\alpha|_1 \leq m} \left ( \int_D |v_\alpha (\partial_\alpha u - \partial_\alpha g_l)| \dx + \int_D |v_\alpha \partial_\alpha g_l - w_\alpha \partial_\alpha u| \dx \right ) \\
&\leq \sum_{0 \leq |\alpha|_1 \leq m} \left ( \|v_\alpha\|_{L^\infty} \|u - g_l\|_{W^{m,1}} +  \int_D |(v_\alpha - w_\alpha) \partial_\alpha g_l| \dx + \int_D |(\partial_\alpha g_l - \partial_\alpha u) w_\alpha| \dx \right ) \\
&< \frac{2\epsilon}{3} + \sum_{0 \leq |\alpha|_1 \leq m} \|w_\alpha\|_{L^\infty} \|u - g_l\|_{W^{m,1}} \\
&< \epsilon.
\end{align*}

\end{proof}

\begin{lemma}
\label{lemma:cm_delta_approx}
Let \(D \subset \R^d\) be a domain and \(L \in \bigl ( C^m(\bar{D}) \bigl)^*\) for some \(m \in \N_0\). 
For any compact set \(K \subset C^m(\bar{D})\) and \(\epsilon > 0\),
there exists distinct points \(y_{11},\dots,y_{1 n_1},\dots,y_{J n_J} \in D\) and numbers \(c_{11},\dots,c_{1 n_1},\dots,c_{J n_J} \in \R\)
such that
\[\sup_{u \in K} |L(u) - \sum_{j=1}^J \sum_{k=1}^{n_j} c_{j k} \partial^{\alpha_j} u(y_{jk})| \leq \epsilon\]
where \(\alpha_1,\dots,\alpha_J\) is an enumeration of the set \(\{\alpha \in \N^d_0 : 0 \leq |\alpha|_1 \leq m\}\).
\end{lemma}
\begin{proof}
By Lemma~\ref{lemma:reisz}, there exist finite, signed, Radon
measures \(\{\lambda_{\alpha}\}_{0 \leq |\alpha|_1 \leq m}\) such that
\[L(u) = \sum_{0 \leq |\alpha|_1 \leq m} \int_{\bar{D}} \partial^\alpha u \: \mathsf{d} \lambda_\alpha, \qquad \forall u \in C^m(\bar{D}).\]
Let \(\alpha_1,\dots,\alpha_J\) be an enumeration of the set \(\{\alpha \in \N^d_0 : 0 \leq |\alpha|_1 \leq m\}\).
By weak density of the Dirac measures \cite[Example 8.1.6]{bogachev2007measure}, we can find points 
\(y_{1 1},\dots,y_{1 n_{1}},\dots,y_{J 1},\dots,y_{J n_{J}} \in \bar{D}\) as well
as numbers \(c_{1 1},\dots,c_{J n_{J}} \in \R\) such that
\[|\int_{\bar{D}} \partial^{\alpha_j} u \: \mathsf{d}\lambda_{\alpha_j} -\sum_{k=1}^{n_j} c_{j k} \partial^{\alpha_j} u(y_{jk})| \leq \frac{\epsilon}{4J}, \qquad \forall u \in C^m(\bar{D})\]
for any \(j \in \{1,\dots,J\}\). Therefore,
\[|\sum_{j=1}^J \int_{\bar{D}} \partial^{\alpha_j} u \: \mathsf{d}\lambda_{\alpha_j} - \sum_{j=1}^J \sum_{k=1}^{n_j} c_{j k} \partial^{\alpha_j} u(y_{jk})| \leq \frac{\epsilon}{4}, \qquad \forall u \in C^m(\bar{D}).\]
Define the constant 
\[Q \coloneqq \sum_{j=1}^J \sum_{k=1}^{n_j} |c_{j k}|.\]
Since \(K\) is compact, we can find functions \(g_1,\dots,g_N \in K\) such that, for any \(u \in K\), 
there exists \(l \in \{1,\dots,N\}\) such that
\[\|u - g_l\|_{C^k} \leq \frac{\epsilon}{4Q}.\]
Suppose that some \(y_{jk} \in \partial D\). By uniform continuity, we can find a point \(\tilde{y}_{jk} \in D\) such that
\[\max_{l \in \{1,\dots,N\}} |\partial^{\alpha_j} g_l (y_{jk}) - \partial^{\alpha_j} g_l (\tilde{y}_{jk})| \leq \frac{\epsilon}{4Q}.\]
Denote 
\[S(u) = \sum_{j=1}^J \sum_{k=1}^{n_j} c_{j k} \partial^{\alpha_j} u(y_{jk})\]
and by \(\tilde{S}(u)\) the sum \(S(u)\) with \(y_{jk}\) replaced by \(\tilde{y}_{jk}\). Then, for any \(u \in K\), we have
\begin{align*}
|L(u) - \tilde{S}(u)| &\leq |L(u) - S(u)| + |S(u) - \tilde{S}(u)| \\
&\leq \frac{\epsilon}{4} + |c_{jk} \partial^{\alpha_j} u(\tilde{y}_{jk}) - c_{jk} \partial^{\alpha_j} u(y_{jk})| \\
&\leq \frac{\epsilon}{4} + |c_{jk} \partial^{\alpha_j} u(\tilde{y}_{jk}) - c_{jk} \partial^{\alpha_j} g_l(\tilde{y}_{jk})| + |c_{jk} \partial^{\alpha_j} g_l(\tilde{y}_{jk}) - c_{jk} \partial^{\alpha_j} u(y_{jk})| \\
&\leq \frac{\epsilon}{4} + |c_{jk}|\|u - g_l\|_{C^m} + |c_{jk} \partial^{\alpha_j} g_l(\tilde{y}_{jk}) - c_{jk} \partial^{\alpha_j} g_l(y_{jk})| + |c_{jk} \partial^{\alpha_j} g_l(y_{jk}) - c_{jk} \partial^{\alpha_j} u(y_{jk})| \\
&\leq \frac{\epsilon}{4}  + 2|c_{jk}|\|u - g_l\|_{C^m} + |c_{jk}| |\partial^{\alpha_j} g_l(\tilde{y}_{jk}) - \partial^{\alpha_j} g_l(y_{jk})| \\
&\leq \epsilon.
\end{align*} 
Since there are a finite number of points, this implies that all points \(y_{jk}\) can be chosen in \(D\). Suppose now 
that \(y_{jk} = y_{qp}\) for some \((j,k) \neq (q,p)\). As before, we can always find a point \(\tilde{y}_{jk}\)
distinct from all others such that
\[\max_{l \in \{1,\dots,N\}} |\partial^{\alpha_j} g_l (y_{jk}) - \partial^{\alpha_j} g_l (\tilde{y}_{jk})| \leq \frac{\epsilon}{4Q}.\]
Repeating the previous argument then shows that all points \(y_{jk}\) can be chosen distinctly as desired.
\end{proof}

\begin{lemma}
\label{lemma:c_kernelapprox}
Let \(D \subset \R^d\) be a domain and \(L \in \bigl ( C(\bar{D}) \bigl)^*\). 
For any compact set \(K \subset C(\bar{D})\) and \(\epsilon > 0\),
there exists a function \(\kappa \in C^\infty_c (D)\) such that
\[\sup_{u \in K} |L(u) - \int_D \kappa  u \dx| < \epsilon.\]
\end{lemma}

\begin{proof}
By Lemma~\ref{lemma:cm_delta_approx}, we can find points distinct points 
\(y_1, \dots, y_n \in D\) as well
as numbers \(c_1,\dots,c_n \in \R\) such that
\[\sup_{u \in K} |L(u) - \sum_{j=1}^{n} c_{j} u(y_{j})| \leq \frac{\epsilon}{3}.\]
Define the constants
\[Q \coloneqq \sum_{j=1}^n |c_{j}|.\]
Since \(K\) is compact, there exist functions \(g_1,\dots,g_J \in K\) such that, for any \(u \in K\), there exists 
some \(l \in \{1,\dots,J\}\) such that
\[\|u - g_l\|_{C} \leq \frac{\epsilon}{6nQ}.\]
Let \(r > 0\) be such that the open balls \(B_r (y_{j}) \subset D\) and are pairwise disjoint.
Let \(\psi_\eta \in C^\infty_c (\R^d)\) denote the standard mollifier with parameter \(\eta > 0\), noting that \(\text{supp } \psi_r = B_r(0)\). We can find a number \(0 < \gamma \leq r\) such that
\[\max_{\substack{l \in \{1,\dots,J\} \\ j \in \{1,\dots,n\}}} |\int_D \psi_{\gamma}(x - y_{j}) g_l(x) \dx - g_l(y_{j})| \leq \frac{\epsilon}{3nQ}.\]
Define \(\kappa : \R^d \to \R\) by
\[\kappa(x) = \sum_{j=1}^n c_{j} \psi_{\gamma}(x - y_{j}), \qquad \forall x \in \R^d.\]
Since \(\text{supp } \psi_\gamma (\cdot - y_j) \subseteq B_r(y_j)\), we have that \(\kappa \in C^\infty_c(D)\).
Then, for any \(u \in K\),
\begin{align*}
|L(u) - \int_D \kappa u \dx| &\leq |L({}u) - \sum_{j=1}^n c_j u(y_j) | + |\sum_{j=1}^n c_j u(y_j) - \int_D \kappa u \dx| \\
&\leq \frac{\epsilon}{3} + \sum_{j=1}^n |c_j| |u(y_j) - \int_D \psi_\eta(x - y_j)u(x) \dx| \\
&\leq \frac{\epsilon}{3} + Q \sum_{j=1}^n |u(y_j) - g_l(y_j)| + |g_l(y_j) - \int_D \psi_\eta(x - y_j)u(x) \dx| \\
&\leq \frac{\epsilon}{3} + nQ \|u - g_l\|_{C} + Q \sum_{j=1}^n |g_l(y_j) - \int_D \psi_\eta (x - y_j)g_l(x) \dx|\\
&\quad\quad\quad\quad\quad\quad+ | \int_D \psi_\eta (x - y_j) \bigl (g_l(x) - u(x) \bigl) \dx| \\
&\leq \frac{\epsilon}{3} + nQ \|u - g_l\|_{C} + n Q \frac{\epsilon}{3n Q} + Q \|g_l - u\|_C \sum_{j=1}^n \int_D \psi_{\gamma}(x - y_{j}) \dx \\
&= \frac{2\epsilon}{3} + 2nQ \|u - g_l\|_{C} \\
&=\epsilon
\end{align*}
where we use the fact that mollifiers are non-negative and integrate to one.
\end{proof}

\begin{lemma}
\label{lemma:cm_kernelapprox}
Let \(D \subset \R^d\) be a domain and \(L \in \bigl ( C^m (\bar{D}) \bigl)^*\). 
For any compact set \(K \subset C^m (\bar{D})\) and \(\epsilon > 0\),
there exist functions \(\kappa_1,\dots,\kappa_J \in C^\infty_c (D)\) such that
\[\sup_{u \in K} |L(u) - \sum_{j=1}^J \int_D \kappa_j  \partial^{\alpha_j} u \dx| < \epsilon\]
where \(\alpha_1,\dots,\alpha_J\) is an enumeration of the set \(\{\alpha \in \N^d_0 : 0 \leq |\alpha|_1 \leq m\}\).
\end{lemma}

\begin{proof}
By Lemma~\ref{lemma:cm_delta_approx}, we find distinct points \(y_{11},\dots,y_{1 n_1},\dots,y_{J n_J} \in D\) and numbers \(c_{11},\dots,c_{J n_J} \in \R\) such that
\[\sup_{u \in K} |L(u) - \sum_{j=1}^J \sum_{k=1}^{n_j} c_{j k} \partial^{\alpha_j} u(y_{jk})| \leq \frac{\epsilon}{2}.\]
Applying the proof of Lemma~\ref{lemma:cm_kernelapprox} \(J\) times to each of the inner sums, we find functions 
\(\kappa_1,\dots,\kappa_J \in C^\infty_c (D)\) such that
\[\max_{j \in \{1,\dots,J\}}|\int_D \kappa_j \partial^{\alpha_j} u \dx - \sum_{k=1}^{n_j} c_{jk} \partial^{\alpha_j} u(y_{jk}) | \leq \frac{\epsilon}{2J}.\]
Then, for any \(u \in K\),
\begin{align*}
|L(u) - \sum_{j=1}^J &\int_D \kappa_j  \partial^{\alpha_j} u \dx|\\& \leq |L(u) - \sum_{j=1}^J \sum_{k=1}^{n_j} c_{j k} \partial^{\alpha_j} u(y_{jk})| + \sum_{j=1}^J  |\int_D \kappa_j \partial^{\alpha_j} u \dx - \sum_{k=1}^{n_j} c_{jk} \partial^{\alpha_j} u(y_{jk})| \leq \epsilon
\end{align*}
as desired.
\end{proof}

\section{}
\label{sec:appendix_nos}

The following lemmas show that the three pieces used in constructing the approximation from Lemma~\ref{lemma:finitedim_approx}, which are schematically depicted in Figure~\ref{fig:approach}, can all be approximated by NO(s). Lemma~\ref{lemma:input_approx} shows that \(F_J : \A \to \R^{J}\) can be approximated by an element of \(\mathsf{IO}\) by mapping to a vector-valued constant function. Similarly, Lemma~\ref{lemma:output_approx} shows that \(G_{J'} : \R^{J'} \to \U\) can be approximated by an element of \(\mathsf{IO}\) by mapping a vector-valued constant function to the coefficients of a basis expansion. Finally, Lemma~\ref{lemma:nn_emulation} shows that NO(s) can exactly represent any standard neural network by viewing the inputs and outputs as vector-valued constant functions.

\begin{lemma}
\label{lemma:input_approx}
Let Assumption \ref{assump:input} hold. Let \(\{c_j\}_{j=1}^n \subset \A^*\) for some \(n \in \N\).
Define the map \(F : \A \to \R^n\) by 
\[F(a) = \bigl (c_1(a), \dots, c_n(a) \bigl ), \qquad \forall a \in \A.\]
Then, for any compact set \(K \subset \A\), \(\sigma \in \mathsf{A}_{0}\), and \(\epsilon > 0\), there exists 
a number \(L \in \N\) and neural network \(\kappa \in \NN_L(\sigma;\R^{d} \times \R^{d},\R^{n \times 1})\) such that
\[\sup_{a \in K} \sup_{y \in \bar{D}} |F(a) - \int_{D} \kappa(y, x) a(x) \dx |_1 \leq \epsilon. \]
\end{lemma}
\begin{proof}
Since \(K\) is bounded, there exists a number \(M > 0\) such that
\[\sup_{a \in K} \|a\|_\A \leq M.\]
Define the constant
\[Q \coloneqq \begin{cases}
M, & \A = W^{m,p}(D) \\
M |D|, & A = C(\bar{D})
\end{cases}\]
and let \(p = 1\) if \(\A = C(\bar{D})\).
By Lemma~\ref{lemma:wmp_kernelapprox} and Lemma~\ref{lemma:c_kernelapprox}, there exist functions
\(f_1,\dots,f_n \in C^\infty_c (D)\) such that
\[\max_{j\in \{1,\dots,n\}} \sup_{a \in K} |c_j(a) - \int_D f_j a \dx | \leq \frac{\epsilon}{2 n^{\frac{1}{p}}}.\]
Since \(\sigma \in \mathsf{A}_0\), there exits some \(L \in \N\) and neural networks \(\psi_1,\dots,\psi_n \in \NN_L(\sigma;\R^d)\)
such that
\[\max_{j \in \{1,\dots,n\}} \|\psi_j - f_j\|_C \leq \frac{\epsilon}{2Q n^{\frac{1}{p}}}.\]
By setting all weights associated to the first argument to zero, 
we can modify each neural network \(\psi_j\)
to a neural network \(\psi_j \in \NN_L(\sigma;\R^{d} \times \R^{d})\) so that
\[\psi_j(y,x) = \psi_j(x) \mathds{1}(y), \qquad \forall y,x \in \R^{d}.\]
Define \(\kappa \in \NN_L (\sigma; \R^{d} \times \R^{d}, \R^{n \times 1})\) by
\[\kappa(y,x) = [\psi_1(y,x), \dots, \psi_n (y,x) ]^T.\]
Then for any \(a \in K\) and \(y \in \bar{D}\), we have
\begin{align*}
|F(a) - \int_D \kappa(y,x) a \dx|_p^p &= \sum_{j=1}^n |c_j(a) - \int_D \mathds{1}(y) \psi_j(x) a(x) \dx|^p \\
&\leq 2^{p-1} \sum_{j=1}^n |c_j(a) - \int_D f_j a \dx|^p + |\int_D (f_j - \psi_j) a \dx |^p \\
&\leq \frac{\epsilon^p}{2} + 2^{p-1} n Q^p \|f_j - \psi_j\|_C^p \\
&\leq \epsilon^p
\end{align*}
and the result follows by finite dimensional norm equivalence.
\end{proof}

\begin{lemma}
\label{lemma:cm_input_approx}
Suppose \(D \subset \R^d\) is a domain and let \(\{c_j\}_{j=1}^n \subset \bigl ( C^{m}(\bar{D}) \bigl )^*\) for some \(m, n \in \N\).
Define the map \(F : \A \to \R^n\) by 
\[F(a) = \bigl (c_1(a), \dots, c_n(a) \bigl ), \qquad \forall a \in C^m(\bar{D}).\]
Then, for any compact set \(K \subset C^m(\bar{D})\), \(\sigma \in \mathsf{A}_{0}\), and \(\epsilon > 0\), there exists 
a number \(L \in \N\) and neural network \(\kappa \in \NN_L(\sigma;\R^{d} \times \R^{d},\R^{n \times J})\) such that
\[\sup_{a \in K} \sup_{y \in \bar{D}} |F(a) - \int_{D} \kappa(y, x) \bigl (\partial^{\alpha_1} a(x), \dots, \partial^{\alpha_J} a(x) \bigl ) \dx |_1 \leq \epsilon \]
where \(\alpha_1,\dots,\alpha_J\) is an enumeration of the set \(\{\alpha \in \N^d : 0 \leq |\alpha|_1 \leq m\}\).
\end{lemma}
\begin{proof}
The proof follows as in Lemma~\ref{lemma:input_approx} by replacing the use of Lemmas~\ref{lemma:wmp_kernelapprox} and \ref{lemma:c_kernelapprox}
by Lemma~\ref{lemma:cm_kernelapprox}.
\end{proof}

\begin{lemma}
\label{lemma:output_approx}
Let Assumption \ref{assump:output} hold. Let \(\{\varphi_j\}_{j=1}^n \subset \U\) for some \(n \in \N\).
Define the map \(G : \R^n \to \U\) by 
\[G(w) = \sum_{j=1}^n w_j \varphi_j, \qquad \forall w \in \R^n.\]
Then, for any compact set \(K \subset \R^n\), \(\sigma \in \mathsf{A}_{m_2}\), and \(\epsilon > 0\), there exists 
a number \(L \in \N\) and a neural network \(\kappa \in \NN_L(\sigma;\R^{d'} \times \R^{d'},\R^{1 \times n})\) such that
\[\sup_{w \in K} \|G(w) - \int_{D'} \kappa(\cdot, x) w \mathds{1}(x) \dx \|_\U \leq \epsilon. \]
\end{lemma}
\begin{proof}
Since \(K \subset \R^n\) is compact, there is a number \(M > 1\) such that
\[\sup_{w \in K} |w|_1 \leq M.\]
If \(\U = L^{p_2}(D')\), then density of \(C^\infty_c (D')\) implies there are 
functions \(\tilde{\psi}_1, \dots, \tilde{\psi}_n \in C^\infty (\bar{D}')\) such that
\[\max_{j \in \{1,\dots,n\}} \|\varphi_j - \tilde{\psi}_j \|_\U \leq \frac{\epsilon}{2nM}.\]
Similarly if \(U = W^{m_2, p_2}(D')\), then density of the restriction of functions in \(C^{\infty}_c (\R^{d'})\)
to \(D'\) \cite[Theorem 11.35]{leoni2009first} implies the same result.
If \(\U = C^{m_2}(\bar{D}')\) then we set \(\tilde{\psi}_j = \varphi_j\) for any \(j \in \{1,\dots,n\}\).
Define \(\tilde{\kappa} : \R^{d'} \times \R^{d'} \to \R^{1 \times n}\) by 
\[\tilde{\kappa}(y,x) = \frac{1}{|D'|} [\tilde{\psi}_1(y), \dots, \tilde{\psi}_n (y)].\]
Then, for any \(w \in K\), 
\begin{align*}
\|G(w) - \int_{D'} \tilde{\kappa}(\cdot,x) w \mathds{1}(x) \dx \|_\U &= \| \sum_{j=1}^n w_j \varphi_j - \sum_{j=1}^n w_j \tilde{\psi}_j \|_\U \\
&\leq \sum_{j=1}^n  |w_j| \|\varphi_j - \tilde{\psi}_j\|_\U \\
&\leq \frac{\epsilon}{2}.
\end{align*}
Since \(\sigma \in \mathsf{A}_{m_2}\), there exists neural networks \(\psi_1,\dots,\psi_n \in \NN_1(\sigma;\R^{d'})\) such that
\[\max_{j \in \{1,\dots,n\}}\| \tilde{\psi}_j - \psi_j \|_{C^{m_2}} \leq \frac{\epsilon}{2nM(J|D'|)^{\frac{1}{p_2}}}\]
where, if \(\U = C^{m_2}(\bar{D}')\), we set \(J = 1/|D'|\) and \(p_2=1\), and otherwise \(J = |\{\alpha \in \N^d : |\alpha|_1 \leq m_2 \}|\). By setting all weights associated to the second argument to zero, 
we can modify each neural network \(\psi_j\)
to a neural network \(\psi_j \in \NN_1(\sigma;\R^{d'} \times \R^{d'})\) so that
\[\psi_j(y,x) = \psi_j(y) \mathds{1}(x), \qquad \qquad \forall y,x \in \R^{d'}.\]
Define \(\kappa \in \NN_1(\sigma;\R^{d'} \times \R^{d'},\R^{1 \times n})\) as
\[\kappa(y,x) = \frac{1}{|D'|}[\psi_1(y,x), \dots, \psi_n (y,x)].\]
Then, for any \(w \in \R^n\),
\[\int_{D'} \kappa(y,x) w \mathds{1}(x) \dx = \sum_{j=1}^nw_j \psi_j(y).\]
We compute that, for any \(j \in \{1,\dots,n\}\),
\[\|\psi_j - \tilde{\psi}_j\|_\U \leq \begin{cases}
|D'|^{\frac{1}{p_2}} \|\psi_j - \tilde{\psi}_j\|_{C^{m_2}}, & \U = L^{p_2}(D') \\
(J|D'|)^{\frac{1}{p_2}} \|\psi_j - \tilde{\psi}_j\|_{C^{m_2}}, & \U = W^{m_2,p_2}(D') \\
\|\psi_j - \tilde{\psi}_j\|_{C^{m_2}}, & \U = C^{m_2}(\bar{D}')
\end{cases}\]
hence, for any \(w \in K\),
\[\|\int_{D'} \kappa(y,x) w \mathds{1}(x) \dx - \sum_{j=1}^n w_j \tilde{\psi}_j \|_{\U} \leq \sum_{j=1}^n |w_j| \|\psi_j - \tilde{\psi}_j \|_\U \leq \frac{\epsilon}{2}.\]
By triangle inequality, for any \(w \in K\), we have
\begin{align*}
\|G(w) - \int_D \kappa(\cdot,x) w \mathds{1}(x) \dx\|_\U &\leq \|G(w) - \int_D \tilde{\kappa}(\cdot,x) w \mathds{1}(x) \dx\|_\U \\
&\:\:\:\: + \|\int_D \tilde{\kappa}(\cdot,x) w \mathds{1}(x) \dx - \int_D \kappa(\cdot,x) w \mathds{1}(x) \dx\|_\U \\
&\leq \frac{\epsilon}{2} + \|\int_D \kappa(\cdot,x) w \mathds{1}(x) -\sum_{j=1}^n w_j \tilde{\psi}_n \|_\U \\
&\leq \epsilon
\end{align*}
as desired.
\end{proof}

\begin{lemma}
\label{lemma:nn_emulation}
Let \(N, d, d', p, q \in \N\), \(m, n \in \N_0\), \(D \subset \R^p\) and \(D' \subset \R^q\) be domains and \(\sigma_1 \in \mathsf{A}^{\emph{\text{L}}}_m\). 
For any \(\varphi \in \NN_N (\sigma_1;\R^d, \R^{d'})\) and \(\sigma_2, \sigma_3 \in \mathsf{A}_n\), there exists a 
\(G \in \mathsf{NO}_N (\sigma_1,\sigma_2,\sigma_3;D,D',\R^d,\R^{d'})\) such that 
\[\varphi(w) = G(w \mathds{1})(x), \qquad \forall w \in \R^d, \:\: \forall x \in D'.\].
\end{lemma}
\begin{proof}
We have that
\[\varphi(x) = W_N \sigma_1 (\dots W_1 \sigma_1 (W_0x + b_0) + b_1 \dots ) + b_N, \qquad \forall x \in \R^d\]
where \(W_0 \in \R^{d_0 \times d}, W_1 \in \R^{d_1 \times d_0}, \dots, W_N \in \R^{d' \times d_{N-1}}\) and \(b_0 \in \R^{d_0}, b_1 \in \R^{d_1}, \dots,b_N \in \R^{d'}\)
for some \(d_0,\dots,d_{N-1} \in \N\). By setting all parameters to zero except for the last bias term, we can find 
\(\kappa^{(0)} \in \NN_1(\sigma_2;\R^p \times \R^p, \R^{d_0 \times d})\) such that
\[\kappa_0 (x,y) = \frac{1}{|D|} W_0, \qquad \forall x,y \in \R^p.\]
Similarly, we can find \(\tilde{b}_0 \in \NN_1(\sigma_2;\R^p, \R^{d_0})\)
such that
\[\tilde{b}_0 (x) = b_0, \qquad \forall x \in \R^p.\]
Then 
\[\int_D \kappa_0 (y,x) w \mathds{1}(x) \dx + \tilde{b}(y) = (W_0 w + b_0) \mathds{1}(y), \qquad \forall w \in \R^d, \:\: \forall y \in D.\]
Continuing a similar construction for all layers clearly yields the result.
\end{proof}

\section{}
\label{sec_proof:discretizational_invariance}

\nk{

\begin{proof}[of Theorem~\ref{thm:discretizational_invariance}]
Without loss of generality, we will assume that \(D = D'\) and, by continuous embedding, that
\(\A = \U = C(\bar{D})\). Furthermore, note that, by continuity, it suffices to show the 
result for the single layer 
\[\mathsf{NO} = \left \{ f \mapsto \sigma_1 \left ( \int_D \kappa(\cdot, y) f(y) \dy + b \right ) : \kappa \in \mathsf{N}_{n_1}(\sigma_2;\R^d \times \R^d), \: b \in \mathsf{N}_{n_2}(\sigma_2;\R^d), \: n_1, n_2 \in \N \right \}.\]
Let \(K \subset \A\) be a compact set and  \((D_j)_{j=1}^\infty\) be a discrete refinement of \(D\).
To each discretization \(D_j\) associate partitions \({P_j^{(1)},\dots,P_j^{(j)} \subseteq D}\) which are pairwise disjoint,
each contains a single, unique point of \(D_j\), each has positive Lebesgue measure, and
\[\coprod_{k=1}^j P_j^{(k)} = D.\]
We can do this since the points in each discretization \(D_j\) are pairwise distinct.
For any \(\G \in \mathsf{NO}\) with parameters \(\kappa, \: b\) define the sequence of maps \(\hat{\G}_j : \R^{jd} \times \R^j \to \Y\) by 
\[\hat{\G}_j(y_1,\dots,y_j,w_1,\dots,w_j) = \sigma_1 \left ( \sum_{k=1}^j \kappa(\cdot, y_k) w_k |P_j^{(k)}| + b(\cdot) \right )\]
for any \(y_k \in \R^d\) and \(w_k \in \R\). Since \(K\) is compact, there is a constant \(M > 0\) such that 
\[\sup_{a \in K} \|a\|_{\U} \leq M.\]
Therefore,
\begin{align*}
\sup_{x \in \bar{D}} \sup_{j \in \N} \left | \int_D \kappa (x,y) a(y) \dy + \sum_{k=1}^j \kappa(x,y_k) a(y_k) |P_j^{(k)}| + 2b(x) \right | &\leq 2(M |D| \|\kappa\|_{C(\bar{D} \times \bar{D})} + \|b\|_{C(\bar{D})} ) \\
&\coloneqq R.
\end{align*}
Hence we need only consider \(\sigma_1\) as a map \([-R,R] \to \R\). Thus, by uniform continuity, there exists a modulus of continuity \(\omega : \R_{\geq 0} \to \R_{\geq 0} \) which is continuous, non-negative, and non-decreasing on \(\R_{\geq 0}\), satisfies \(\omega(z) \to \omega(0) =0\) 
as \(z \to 0\) and 
\begin{equation}
\label{eq:disc_inv_modulusofcont}
|\sigma_1(z_1) - \sigma_1(z_2) | \leq \omega(|z_1 - z_2|) \qquad \forall z_1,z_2 \in [-R,R].
\end{equation}
Let \(\epsilon > 0\). Equation \eqref{eq:disc_inv_modulusofcont} and the non-decreasing property of \(\omega\) imply that in order to show there exists \(Q = Q(\epsilon) \in \N\) such that for any \(m \geq Q\) implies 
\[\sup_{a \in K}\| \hat{\G}_m (D_m, a|_{D_m}) - \G(a)\|_\Y < \epsilon,\]
it is enough to show that
\begin{equation}
\label{eq:disc_inv_needtoshow}
\sup_{a \in K} \sup_{x \in \bar{D}} \left | \int_D \kappa(x,y) a(y) \dy - \sum_{k=1}^m \kappa(x,y_k) a(y_k) |P_m^{(k)}| \right | < \epsilon
\end{equation}
for any \(m \geq Q\).
Since \(K\) is compact, we can find functions \(a_1,\dots,a_N \in K\) such that, for any \(a \in K\), there is some \(n \in \{1,\dots,N\}\) such that
\[\|a-a_n\|_{C(\bar{D})} \leq \frac{\epsilon}{4 |D| \|\kappa\|_{C(\bar{D} \times \bar{D})}}.\]
Since \((D_j)\) is a discrete refinement, by convergence of Riemann sums, we can find some \(q \in \N\) such that for any \(t \geq q\), we have
\[\sup_{x \in \bar{D}} \left | \sum_{k=1}^t \kappa(x,y_k) |P_t^{(k)}| - \int_{D} \kappa(x,y) \dy  \right | < |D| \|\kappa\|_{C(\bar{D} \times \bar{D})}\]
where \(D_t = \{y_1,\dots,y_t\}\). Similarly, we can find \(p_1,\dots,p_N \in \N\) such that, for any \(t_n \geq p_n\), we have
\[\sup_{x \in \bar{D}} \left | \sum_{k=1}^{t_n} \kappa(x,y_k^{(n)}) a_n(y_k^{(n)}) |P_{t_n}^{(k)}| - \int_D \kappa(x,y) a_n(y) \dy \right | < \frac{\epsilon}{4}\]
where \(D_{t_n} = \{y_1^{(n)},\dots,y_{t_n}^{(n)}\}\).
Let \(m \geq \max \{q,p_1,\dots,p_N\}\) and denote \(D_m = \{y_1,\dots,y_m\}\).
Note that,
\[\sup_{x \in \bar{D}} \left | \int_D \kappa(x,y) \left ( a(y) - a_n(y) \right ) \dy \right | \leq |D| \|\kappa\|_{C(\bar{D} \times \bar{D})} \|a - a_n\|_{C(\bar{D})}.\]
Furthermore,
\begin{align*}
    \sup_{x \in \bar{D}} \left | \sum_{k=1}^m \kappa(x,y_k) \left ( a_n(y_k) - a(y_k) \right ) |P_m^{(k)}| \right | &\leq \|a_n - a \|_{C(\bar{D})} \sup_{x \in \bar{D}} \left | \sum_{k=1}^m \kappa(x,y_k) |P_m^{(k)}| \right | \\
    &\leq \|a_n - a \|_{C(\bar{D})} \bigg ( \sup_{x \in \bar{D}} \left | \sum_{k=1}^m \kappa(x,y_k) |P_m^{(k)}| - \int_{D} \kappa(x,y) \dy  \right | \\
    &\qquad\qquad\qquad\qquad+ \sup_{x \in \bar{D}} \left | \int_D \kappa(x,y) \dy \right | \bigg  ) \\
    &\leq 2 |D|\|\kappa\|_{C(\bar{D} \times \bar{D})} \|a_n - a \|_{C(\bar{D})}.
\end{align*}
Therefore, for any \(a \in K\), by repeated application of the triangle inequality, we find that 
\begin{align*}
    \sup_{x \in \bar{D}} \left | \int_D \kappa(x,y) a(y) \dy - \sum_{k=1}^m \kappa(x,y_k) a(y_k) |P_m^{(k)}| \right | &\leq \sup_{x \in \bar{D}} \left | \sum_{k=1}^{m} \kappa(x,y_k) a_n(y_k) |P_{m}^{(k)}| - \int_D \kappa(x,y) a_n(y) \dy \right | \\
    &\quad+ 3 |D| \|\kappa\|_{C(\bar{D} \times \bar{D})} \|a - a_n\|_{C(\bar{D})} \\
    &< \frac{\epsilon}{4} + \frac{3 \epsilon}{4} = \epsilon
\end{align*}
which completes the proof.
\end{proof}
}

\section{}
\label{sec_proof:main_compact}

\begin{proof}[of Theorem~\ref{thm:main_compact}]
The statement in Lemma~\ref{lemma:ap} allows us to apply Lemma~\ref{lemma:finitedim_approx} to find a mapping \(\G_1 : \A \to \U\) 
such that
\[\sup_{a \in K} \|\G^\dagger (a) - \G_1(a)\|_{\U} \leq \frac{\epsilon}{2}\]
where \(\G_1 = G \circ \psi \circ F\) with \(F : \A \to \R^J\), \(G : \R^{J'} \to \U\) continuous linear maps  and \(\psi \in C(\R^{J};\R^{J'})\) for some \(J, J' \in \N\).
By Lemma~\ref{lemma:input_approx}, we can find a sequence of maps \(F_t \in \mathsf{IO}(\sigma_2;D,\R,\R^J)\) for \(t=1,2,\dots\) such that
\[\sup_{a \in K} \sup_{x \in \bar{D}} | \bigl (F_t (a) \bigl )(x) - F (a)|_1 \leq \frac{1}{t}.\]
In particular, \(F_t(a)(x) = w_t(a) \mathds{1}(x)\) for some \(w_t : \A \to \R^J\)
which is constant in space.
We can therefore identify the range of $F_t(a)$ with \(\R^J\).
Define the set 
\[Z \coloneqq \bigcup_{t = 1}^\infty F_t (K) \cup F (K) \subset \R^{J}\]
which is compact by Lemma~\ref{lemma:compact_union}.
Since \(\psi\) is continuous, it is uniformly continuous on \(Z\) hence there exists a modulus of continuity \(\omega : \R_{\geq 0} \to \R_{\geq 0}\) which is continuous, non-negative, and non-decreasing on \(\R_{\geq 0}\), satisfies \(\omega(s) \to \omega(0) = 0\) as \(s \to 0\) and
\[|\psi(z_1) - \psi(z_2)|_1 \leq \omega ( |z_1 - z_2|_1 ) \qquad \forall z_1, z_2 \in Z.\]
We can thus find \(T \in \N\) large enough such that
\[\sup_{a \in K} \omega (|F (a) - F_T (a)|_1) \leq \frac{\epsilon}{6 \|G\|}.\]
Since \(F_T\) is continuous, \(F_T (K)\) is compact. \nk{Since \(\psi\) is a continuous function on the compact set \(F_T(K) \subset \R^J\) mapping into \(\R^{J'}\), we can use any classical neural network approximation theorem such as \cite[Theorem 4.1]{pinkus1999approximation} to find an \(\epsilon\)-close (uniformly) neural network. Since Lemma \ref{lemma:nn_emulation} shows that neural operators can exactly mimic standard neural networks, it follows that} we can find 
\(S_1 \in \mathsf{IO}(\sigma_1;D,\R^J,\R^{d_1}), \dots,\) \(S_{N-1} \in \mathsf{IO}(\sigma_1;D,\R^{d_{N-1}},\R^{J'})\) for some \(N \in \N_{\geq 2}\) 
and \(d_1,\dots,d_{N-1} \in \N\) such that
\[\tilde{\psi}(f) \coloneqq \bigl ( S_{N-1} \circ \sigma_1 \circ \dots \circ S_2 \circ \sigma_1 \circ S_1 \bigl)(f), \qquad \forall f \in L^1(D;\R^J) \]
satisfies
\[\sup_{q \in F_T (K)} \sup_{x \in \bar{D}}  | \psi(q) - \tilde{\psi}(q \mathds{1})(x) |_1 \leq \frac{\epsilon}{6 \|G\|}. \]
By construction, \(\tilde{\psi}\) maps constant functions into constant functions and is continuous in the 
appropriate subspace topology of constant functions hence we can identity it as an element of \(C(\R^J;\R^{J'})\)
for any input constant function taking values in \(\R^J\).
Then \((\tilde{\psi} \circ F_T)(K) \subset \R^{J'}\) is compact. 
Therefore, by Lemma~\ref{lemma:output_approx}, we can find a neural network
\(\kappa \in \mathcal{N}_L (\sigma_3; \R^{d'} \times \R^{d'}, \R^{1 \times J'})\) for some \(L \in \N\) such that
\[\tilde{G} (f) \coloneqq \int_{D'} \kappa (\cdot, y) f(y) \: \text{d}y, \qquad \forall f \in L^1(D;\R^{J'})\]
satisfies
\[\sup_{y \in (\tilde{\psi}  \circ F_T)(K)} \|G (y) - \tilde{G} (y \mathds{1}) \|_\U \leq \frac{\epsilon}{6}.\]
Define
\[\G(a) \coloneqq \bigl (\tilde{G} \circ \tilde{\psi} \circ F_T \bigl)(a) = \int_{D'} \kappa(\cdot,y) \bigl ( (S_{N-1} \circ \sigma_1 \circ \dots \sigma_1 \circ S_1 \circ F_T)(a)\bigl) (y)\: \mathsf{d}y, \qquad \forall a \in \A,\]
noting that \(\G \in \mathsf{NO}_{N}(\sigma_1,\sigma_2,\sigma_3;D,D')\).
For any \(a \in K\), define \(a_1 \coloneqq (\psi \circ F) (a)\) and \(\tilde{a}_1 \coloneqq (\tilde{\psi} \circ F_T)(a)\) so that $\G_1(a)=G(a_1)$ and $\G(a)=\tilde{G} (\tilde{a}_1)$ then
\begin{align*}
    \|\G_1(a) - \G(a)\|_\U &\leq \|G (a_1) - G ( \tilde{a}_1) \|_\U + \|G (\tilde{a}_1) - \tilde{G} (\tilde{a}_1) \|_\U \\
    &\leq \|G\| |a_1 - \tilde{a_1}|_1 + \sup_{y \in (\tilde{\psi} \circ F_T)(K)} \|G (y) - \tilde{G} (y \mathds{1}) \|_\U \\
    &\leq \frac{\epsilon}{6} + \|G\| |(\psi \circ F)(a) - (\psi \circ F_T)(a) |_1 + \|G\| |(\psi \circ F_T)(a) - (\tilde{\psi} \circ F_T)(a) |_1\\
    &\leq \frac{\epsilon}{6} + \|G\| \omega \bigl ( |F(a) - F_T (a)|_1 \bigl ) + \|G\| \sup_{q \in F_T (K)} |\psi(q) - \tilde{\psi}(q)|_1  \\
    &\leq \frac{\epsilon}{2}.
\end{align*}
Finally we have
\begin{align*}
    \|\G^\dagger(a) - \G(a)\|_\U \leq \|\G^\dagger (a) - \G_1(a) \|_\U + \|\G_1(a) - \G(a)\|_\U \leq \frac{\epsilon}{2} + \frac{\epsilon}{2} = \epsilon 
\end{align*}
as desired.

To show boundedness, we will exhibit a neural operator \(\tilde{\G}\) that is \(\epsilon\)-close to \(\G\)
in \(K\) and is uniformly bounded by \(4M\). Note first that
\[\|\G(a)\|_\U \leq \|\G(a) - \G^\dagger(a)\|_\U + \|\G^\dagger(a)\|_\U \leq \epsilon + M \leq 2M, \qquad \forall a \in K\]
where, without loss of generality, we assume that \(M \geq 1\). By construction, we have that
\[\G(a) = \sum_{j=1}^{J'} \tilde{\psi}_j(F_T(a)) \varphi_j, \qquad \forall a \in \A\]
for some  neural network \(\varphi : \R^{d'} \to \R^{J'}\). Since \(\U\) is a Hilbert space and 
by linearity, we may assume that the components \(\varphi_j\) are orthonormal since orthonormalizing them 
only requires multiplying the last layers of \(\tilde{\psi}\) by an invertible linear map.
Therefore
\[|\tilde{\psi}(F_T(a))|_2 = \|\G(a)\|_\U \leq 2M, \qquad \forall a \in K.\]
Define the set \(W \coloneqq (\tilde{\psi} \circ F_T)(K) \subset \R^{J'}\) which is compact as before.
We have 
\[\text{diam}_2 (W) = \sup_{x,y \in W} |x-y|_2 \leq \sup_{x,y \in W} |x|_2 + |y|_2 \leq 4M.\]
Since \(\sigma_1 \in \mathsf{BA}\), there exists 
a number \(R \in \N\) and a neural network \(\beta \in \NN_R (\sigma_1;\R^{J'},\R^{J'})\) such that
\begin{align*}
|\beta(x) - x|_2 &\leq \epsilon ,\qquad \forall x \in W \\
|\beta(x)|_2 &\leq 4M, \quad \forall x \in \R^{J'}.
\end{align*}
Define
\[\tilde{\G}(a) \coloneqq \sum_{j=1}^{J'} \beta_j (\tilde{\psi}(F_T(a))) \varphi_j, \qquad \forall a \in \A.\]
Lemmas \ref{lemma:output_approx} and \ref{lemma:nn_emulation} then shows that \(\tilde{\G} \in \mathsf{NO}_{N+R}(\sigma_1,\sigma_2,\sigma_3;D,D')\).
Notice that
\[\sup_{a \in K} \|\G(a) - \tilde{\G}(a)\|_{\U} \leq \sup_{w \in W} |w - \beta(w)|_2 \leq \epsilon.\]
Furthermore,
\[\|\tilde{\G}(a)\|_{\U} \leq \|\tilde{\G}(a) - \G(a)\|_{\U} + \|\G(a)\|_\U \leq \epsilon + 2M \leq 3M, \qquad \forall a \in K.\]
Let \(a \in \A \setminus K\) then there exists \(q \in \R^{J'} \setminus W\) such that \(\tilde{\psi}(F_T(a)) = q\) and 
\[\|\tilde{\G}(a)\|_{\U} = |\beta(q)|_2 \leq 4M\]
as desired.
\end{proof}

\section{}
\label{sec_proof:measurable_approx}

\begin{proof}[of Theorem~\ref{thm:measurable_approx}]
Let \(\U = H^{m_2}(D)\). For any \(R > 0\), define 
\[\G^\dagger_R (a) \coloneqq \begin{cases}
\G^\dagger(a), & \|\G^\dagger(a)\|_{\U} \leq R \\
\frac{R}{\|\G^\dagger(a)\|_{\U}} \G^\dagger(a), & \text{otherwise}
\end{cases}\]
for any \(a \in \A\). Since \(\G^\dagger_R \to \G^\dagger\) as \(R \to \infty\) \(\mu\)-almost everywhere, \(\G^\dagger \in L^2_\mu(\A;\U)\), and clearly \(\|\G^\dagger_R(a)\|_{\U} \leq \|\G^\dagger (a)\|_{\U}\) for any \(a \in \A\), we can apply the dominated convergence theorem for Bochner integrals to find \(R > 0\) large enough such that
\[\|\G^\dagger_R - \G^\dagger\|_{L^2_\mu (\A;\U)} \leq \frac{\epsilon}{3}.\]
Since \(\A\) and \(\U\) are Polish spaces, by Lusin's theorem \cite[Theorem 1.0.0]{aaronson1997introduction} we can find a compact set \(K \subset \A\) such that
\[\mu(\A \setminus K) \leq \frac{\epsilon^2}{153R^2}\]
and \(\G^\dagger_R |_K\) is continuous. Since \(K\) is closed, by a generalization of the Tietze extension theorem \cite[Theorem 4.1]{dugundji1961anextension},
there exist a continuous mapping \(\tilde{\G}^\dagger_R : \A \to \U \) such that \(\tilde{\G}^\dagger_R(a) = \G^\dagger_R(a)\) for 
all \(a \in K\) and 
\[\sup_{a \in \A} \|\tilde{\G}^\dagger_R(a)\| \leq \sup_{a \in \A} \|\G^\dagger_R(a)\| \leq R.\]
Applying Theorem~\ref{thm:main_compact} to \(\tilde{\G}^\dagger_R\), we find that there exists a number \(N \in \N\) and a neural operator \(\G \in \mathsf{NO}_{N}(\sigma_1,\sigma_2,\sigma_3;D,D')\)
such that
\[\sup_{a \in K} \|\G(a) - \G^\dagger_R(a) \|_\U \leq \frac{\sqrt{2} \epsilon}{3}\]
and
\[\sup_{a \in \A} \|\G(a)\|_{\U} \leq 4R.\]
We then have
\begin{align*}
\|\G^\dagger - \G\|_{L^2_\mu(\A;\U)} &\leq \|\G^\dagger - \G^\dagger_R\|_{L^2_\mu (\A;\U)} + \|\G^\dagger_R - \G\|_{L^2_\mu (\A;\U)} \\
&\leq \frac{\epsilon}{3} + \left ( \int_K \|\G^\dagger_R(a) - \G(a)\|^2_\U \: \mathsf{d}\mu (a) + \int_{\A \setminus K} \|\G^\dagger_R(a) - \G(a)\|^2_\U \: \mathsf{d}\mu (a) \right )^{\frac{1}{2}} \\
&\leq \frac{\epsilon}{3} + \left (\frac{2 \epsilon^2}{9} + 2 \left ( \sup_{a \in \A} \|\G^\dagger_R (a)\|_\U^2 +  \|\G (a)\|_\U^2 \right ) \mu(\A \setminus K) \right )^{\frac{1}{2}} \\
&\leq \frac{\epsilon}{3} + \left (\frac{2 \epsilon^2}{9} + 34 R^2 \mu(\A \setminus K) \right )^{\frac{1}{2}} \\
&\leq\frac{\epsilon}{3} + \left (\frac{4 \epsilon^2}{9}\right )^{\frac{1}{2}} \\
&= \epsilon
\end{align*}
as desired.
\end{proof}

\end{document}